\documentclass{article}

\usepackage[final]{neurips_2025}

\usepackage[utf8]{inputenc} % allow utf-8 input
\usepackage[T1]{fontenc}    % use 8-bit T1 fonts
\usepackage[
    bookmarks=false,
    colorlinks=true,
    citecolor=brown,
    linkcolor=brown,
    urlcolor=brown
]{hyperref}                 % hyperlinks
\usepackage{url}            % simple URL typesetting
\usepackage{booktabs}       % professional-quality tables
\usepackage{amsfonts}       % blackboard math symbols
\usepackage{amssymb}
\usepackage{cancel}
\usepackage{dsfont}         % for \mathds
\usepackage{nicefrac}       % compact symbols for 1/2, etc.
\usepackage{microtype}      % microtypography
\usepackage{xcolor}         % colors
\usepackage{graphicx}
\usepackage{glossaries}
\usepackage[noend]{algpseudocode}
\usepackage{algorithm, algorithmicx}
\usepackage{multicol}
\usepackage{adjustbox}
\usepackage{array}
\usepackage{pifont}
\usepackage{subcaption}
\usepackage{enumitem}
\usepackage{cleveref}
\usepackage{listings}
\usepackage{amsthm}
\usepackage{mathtools}
\usepackage{thmtools}
\glsdisablehyper
\newcommand{\round}{\ensuremath{t}}
\newcommand{\thresh}{\ensuremath{\tau}}
\newcommand{\acid}{\ensuremath{x}}
\newcommand{\obs}{\ensuremath{\mathbf{x}}}

\newcommand{\tar}{\ensuremath{y}}
\newcommand{\tars}{\mathbf{\tar}}
\newcommand{\bbf}{\ensuremath{{f\!\centerdot}}}
\newcommand{\bbfs}{\ensuremath{{\mathbf{f}\centerdot}}}
\newcommand{\labl}{\ensuremath{z}}

\newcommand{\mask}{\mathbf{o}}
\newcommand{\maskp}{o}
\newcommand{\refp}{\ensuremath{\mathbf{r}}}
\newcommand{\qparam}{\ensuremath{\phi}}

\newcommand{\mparam}{\ensuremath{\theta}}
\newcommand{\uparam}{\ensuremath{\psi}}
\newcommand{\pparam}{\ensuremath{\gamma}}

\newcommand{\errs}{\ensuremath{\boldsymbol{\epsilon}}}

\newcommand{\solnspace}{\ensuremath{\mathcal{S}}}
\newcommand{\pareto}{\ensuremath{\mathcal{S}_{\textrm{Pareto}}}}
\newcommand{\fpareto}{\ensuremath{\mathcal{F}_{\textrm{Pareto}}}}
\newcommand{\paretofn}[1]{\operatorname{Pareto}(#1)}
\newcommand{\prefrspace}{\ensuremath{\mathcal{U}}}
\newcommand{\prefr}{\ensuremath{\mathbf{u}}}
\newcommand{\pweights}{\ensuremath{\boldsymbol{\lambda}}}
\newcommand{\lablu}{a}
\newcommand{\lablf}{\labl}

\newcommand{\real}{\ensuremath{\mathbb{R}}}
\newcommand{\natno}{\ensuremath{\mathbb{N}}}
\newcommand{\obsspace}{\ensuremath{\mathcal{X}}}
\newcommand{\acidspace}{\ensuremath{\mathcal{V}}}

\newcommand{\data}{\ensuremath{\mathcal{D}}}

\newcommand{\threshparam}{\ensuremath{\delta}}

\newcommand{\prob}[1]{\ensuremath{p({#1})}}
\newcommand{\probc}[2]{\ensuremath{p({#1}|{#2}})}

\newcommand{\qrob}[2]{\ensuremath{q_{#1}({#2})}}
\newcommand{\qrobc}[3]{\ensuremath{q_{#1}({#2}|{#3}})}

\newcommand{\normalc}[2]{\ensuremath{\mathcal{N}\!\left({#1}\middle|{#2}
                         \right)}}

\newcommand{\bernc}[2]{\ensuremath{\mathrm{Bern}\!\left({#1}\middle|{#2}
                                    \right)}}

\newcommand{\categc}[2]{\ensuremath{\mathrm{Categ}({#1}|{#2})}}

\newcommand{\multinc}[2]{\ensuremath{\mathrm{Multinomial}({#1}|{#2})}}

\newcommand{\uniform}[1]{\ensuremath{\mathcal{U}({#1})}}
\newcommand{\expece}{\ensuremath{\mathbb{E}}}
\newcommand{\expec}[2]{\ensuremath{\expece_{#1}\!\left[{#2}\right]}}
\newcommand{\dkl}[2]{\ensuremath{\mathbb{D}_\mathrm{KL}\!\left[{#1}\|{#2}\right]}}

\newcommand{\softmax}[1]{\ensuremath{\mathrm{softmax}\!\left({#1}\right)}}
\newcommand{\msoftmax}[1]{\ensuremath{\mathbf{softmax}\!\left({#1}\right)}}
\newcommand{\cpe}[3]{\ensuremath{\pi^{#1}_{#2}({#3})}}
\newcommand{\indic}[1]{\ensuremath{\mathds{1}[{#1}]}}

\newcommand{\elbo}[1]{\ensuremath{\mathcal{L_{\textrm{ELBO}}}\!\left({#1}\right)}}
\newcommand{\aelbo}[1]{\ensuremath{\mathcal{L_{\textrm{A-ELBO}}}\!\left({#1}\right)}}
\newcommand{\lcpe}[1]{\ensuremath{\mathcal{L_{\textrm{CPE}}}\!\left({#1}\right)}}
\newcommand{\lpref}[1]{\ensuremath{\mathcal{L_{\textrm{Pref}}}\!\left({#1}\right)}}

\newcommand{\threshfn}[1]{\ensuremath{f_\thresh}\!\left({#1}\right)}

\newcommand{\lcpef}[1]{\ensuremath{\mathcal{L}^\lablf_{\textrm{CPE}}\!\left({#1}\right)}}
\newcommand{\lcpeu}[1]{\ensuremath{\mathcal{L}^\lablu_{\textrm{CPE}}\!\left({#1}\right)}}
\newcommand{\hvbox}[1]{\ensuremath{\mathrm{B}({#1})}}
\newcommand{\hv}[1]{\ensuremath{\mathrm{HV}({#1})}}
\newcommand{\hvi}[1]{\ensuremath{\mathrm{HVI}({#1})}}

\DeclareMathOperator*{\argmax}{arg\,\!max}
\DeclareMathOperator*{\argmin}{arg\,\!min}

\newcommand*\Let[2]{\State #1 $\gets$ #2}

\newcommand{\card}[1]{\ensuremath{|#1|}}

\newtheorem{assumption}{Assumption}

\newacronym{gp}{GP}{Gaussian process}
\newacronym{vo}{VO}{variational optimization}
\newacronym{vi}{VI}{variational inference}
\newacronym{gvi}{GVI}{generalized variational inference}
\newacronym{bo}{BO}{Bayesian optimization}
\newacronym{ei}{EI}{expected improvement}
\newacronym{pi}{PI}{probability of improvement}
\newacronym{ucb}{UCB}{upper confidence bound}
\newacronym{bore}{BORE}{Bayesian optimization by density-ratio estimation}
\newacronym{bopr}{BOPR}{Bayesian optimization with probabilistic reparametrization}
\newacronym{elbo}{ELBO}{evidence lower bound}
\newacronym{cpe}{CPE}{class probability estimator}
\newacronym{eda}{EDA}{estimation of distribution algorithms}
\newacronym{nes}{NES}{natural evolution strategies}
\newacronym{es}{ES}{evolution strategies}
\newacronym{sgd}{SGD}{stochastic gradient descent}
\newacronym{vsd}{VSD}{variational search distributions}
\newacronym{dbas}{DbAS}{design by adaptive sampling}
\newacronym{cbas}{CbAS}{conditioning by adaptive sampling}
\newacronym{ell}{ELL}{expected log-likelihood}
\newacronym{fdr}{FDR}{false discovery rate}
\newacronym{kl}{KL}{Kullback-Leibler}
\newacronym{pex}{PEX}{proximal exploration}
\newacronym{bbo}{BBO}{black-box optimization}
\newacronym{lstm}{LSTM}{long short-term memory}
\newacronym{rnn}{RNN}{recurrent neural network}
\newacronym{ml}{ML}{maximum likelihood}
\newacronym{lse}{LSE}{level set estimation}
\newacronym{lso}{LSO}{latent space optimization}
\newacronym{ntk}{NTK}{neural tangent kernel}
\newacronym{nos}{NOS}{diffusioN Optimized Sampling}
\newacronym{ga}{GA}{genetic algorithm}
\newacronym{nn}{NN}{neural network}
\newacronym{ag}{AG}{active generation}
\newacronym{moo}{MOO}{multi-objective optimization}
\newacronym{mobo}{MOBO}{multi-objective Bayesian optimization}
\newacronym{mog}{MOG}{multi-objective generation}
\newacronym{peft}{PEFT}{parameter efficient fine-tuning}
\newacronym{ehvi}{EHVI}{expected hypervolume improvement}
\newacronym{nehvi}{nEHVI}{noisy expected hypervolume improvement}
\newacronym{mgd}{MGD}{multiple-gradient descent}
\newacronym{agps}{A-GPS}{active generation of Pareto sets}
\newacronym{hvi}{HVI}{hypervolume improvement}
\newacronym{rasp}{RaSP}{rapid stability predictions}
\newacronym{phvi}{PHVI}{probability of hypervolume improvement}
\newacronym{pca}{PCA}{principal component analysis}
\newacronym{moea}{MOEA}{many-objective evolutionary algorithm}

\newcolumntype{R}[2]{%
    >{\adjustbox{angle=#1,lap=\width-(#2)}\bgroup}%
    l%
    <{\egroup}%
}
\newcommand*\rot{\multicolumn{1}{R{25}{1em}}}%
\newcommand{\cmark}{\ding{51}}
\newcommand{\xmark}{\ding{55}}

\makeatletter
\renewcommand{\paragraph}{%
  \@startsection{paragraph}{4}%
  {\z@}{0ex \@plus 1ex \@minus .2ex}{-1em}%
  {\normalfont\normalsize\bfseries}%
}
\makeatother

\title{Amortized Active Generation of Pareto Sets}

\author{%
  Daniel M. Steinberg$^1$\thanks{corresponding author: \texttt{dan.steinberg@data61.csiro.au}} \quad
  Asiri Wijesinghe$^1$ \quad
  Rafael Oliveira$^1$ \\
  \textbf{Piotr Koniusz}$^{1,2,3}$ \quad
  \textbf{Cheng Soon Ong}$^{1,3}$ \quad
  \textbf{Edwin V. Bonilla}$^1$ \\
  $^1$CSIRO's Data61 \quad $^2$University of New South Wales \quad $^3$Australian National University \\
}

\begin{document}

\maketitle

\begin{abstract}
    % We propose a new framework called \gls{agps} for online discrete black-box
    % \gls{moo} tasks. \Gls{agps} learns a generative model of the Pareto set and
    % supports a-posteriori preference conditioning.
    % Our method uses a \gls{cpe} for predicting Pareto-optimality and
    % conditioning the generative model to generate designs in this
    % high-performance region. We show that this Pareto set membership \gls{cpe}
    % is in fact estimating \gls{phvi}. Furthermore, motivated by discrete/mixed
    % design problems where we must balance multiple competing objectives, it
    % introduces preference direction vectors to capture subjective trade-offs.
    % Thus, at each iteration, we update a generative model conditioned on Pareto
    % set membership \textit{and} alignment with preference directions. Our
    % method yields high-quality Pareto set approximations using only simple
    % \gls{cpe} guidance, avoids hypervolume computation, and supports sampling
    % at arbitrary trade-off points without retraining. Empirical results on
    % synthetic functions and protein design benchmarks demonstrate strong sample
    % efficiency and effective incorporation of users' preferences.
    We introduce \gls{agps}, a new framework for online discrete black-box
    \gls{moo}. \Gls{agps} learns a generative model of the Pareto set that
    supports a-posteriori conditioning on user preferences. The method employs
    a \gls{cpe} to predict non-dominance relations and to condition the generative
    model toward high-performing regions of the search space. We also show that this
    non-dominance \gls{cpe} implicitly estimates the \gls{phvi}. To incorporate subjective trade-offs,
    \gls{agps} introduces \emph{preference direction vectors} that encode
    user-specified preferences in objective space. At each iteration, the model
    is updated using both Pareto membership and alignment with these preference
    directions, producing an amortized generative model capable of sampling
    across the Pareto front without retraining. The result is a simple yet
    powerful approach that achieves high-quality Pareto set approximations,
    avoids explicit hypervolume computation, and flexibly captures user
    preferences. Empirical results on synthetic benchmarks and protein design
    tasks demonstrate strong sample efficiency and effective preference
    incorporation.
\end{abstract}

\section{Introduction}
\label{sec:intro}

In many scientific and engineering domains, practitioners face the challenge of
optimizing complex, high-dimensional, discrete objects under expensive
black-box evaluation processes. Examples include designing protein sequences
for enhanced stability and activity, synthesizing small molecules with tailored
pharmacokinetics, and engineering DNA constructs for precise gene regulation.
In these settings, each candidate design must be evaluated via computationally
intensive simulations or laboratory assays, making efficient search strategies
essential.
Furthermore, these applications frequently involve multiple, often conflicting
objectives. For instance, in protein engineering one may wish to maximize
thermal stability, catalytic turnover rate, and expression yield, yet
improvements in one property can degrade another. The set of non-dominated
trade-off designs (where no objective can be improved without sacrificing
performance in at least one other objective) is known as the Pareto set.
Accurately approximating this Pareto set is critical for enabling informed
decision-making in downstream experimental workflows.

Traditional \gls{bo} methods for black-box \acrfull{moo} problems, often
referred to as \gls{mobo}, rely on acquisition functions such as \gls{ehvi}
\citep{yang2019efficient, daulton2020differentiable,ament2023unexpected} or
their quasi-MonteCarlo extensions \citep{daulton2021parallel,
ament2023unexpected}, which involve complex numerical integration and/or scale
poorly with the number of objectives. Alternatively they may rely on random
scalarizations \citep{knowles2006parego, paria2020flexible, de2022mbore}
offering simplicity and scalability but rely on sufficient sampling density to
capture complex Pareto front geometries.
%
% In contrast, the recently proposed
% \gls{vsd} \citep{steinberg2025variational} framework formulates
% optimization as active learning of a generative model conditioned on
% high-performance regions (active generation). \gls{vsd} alternates between
% fitting a \acrfull{cpe} to discriminate favorable designs and updating a
% conditional generative model to propose new candidates directly.
%
% In this work, we extend the \gls{vsd} paradigm to directly estimate a generative
% model of the Pareto set in an online sequential, black-box optimization
% setting. Specifically, we replace the fitness-threshold \gls{pi}-\gls{cpe} used
% in standard \gls{vsd} with a Pareto-set \gls{cpe}, enabling the generative
% model to focus exclusively on non-dominated designs. This approach bypasses
% explicit hypervolume computations and scalarizations and leads to a scalable
% online algorithm for Pareto set approximation that is different from other
% recent \gls{mog} methods \citep{yuan2024paretoflow, yao2024proud}.

In this work we propose a fundamentally different approach, in the spirit of \gls{mog} \cite{yuan2024paretoflow,yao2024proud}, that directly estimates a generative model of the Pareto set in an online sequential, black-box optimization setting.
% \dan{This is actually already called \gls{mog} -- but we should make a distinction that we are taking a new approach to \gls{mog},  and \emph{not} doing \gls{mobo}}.
To this end, we build upon the recently proposed \gls{vsd} framework for single-objective optimization \citep{steinberg2025variational}. \gls{vsd} formulates
optimization as active learning of a generative model of
high-performing designs (active generation). Thus, \gls{vsd} alternates between
fitting a \acrfull{cpe} to discriminate favorable designs and updating a
conditional generative model to propose new candidates directly. Hence, instead of using a fitness-threshold \gls{pi}-\gls{cpe}, we propose a Pareto-set \gls{cpe}, which enables the generative model to focus exclusively on non-dominated designs. This approach bypasses explicit hypervolume computations and scalarizations and leads to a scalable online algorithm for Pareto set approximation that is different from other
recent \gls{mog} methods \citep{yuan2024paretoflow, yao2024proud}. Furthermore, we show that such a non-dominance \gls{cpe} is implicitly estimating \acrfull{phvi}.
Additionally, practical decision-making often requires incorporating user or
stakeholder preferences over the trade-offs among objectives. We introduce a
novel mechanism for embedding subjective preference conditioning into the
generative model, allowing practitioners to sample candidates that align with
specified trade-off directions. By using amortized \gls{vi} with \emph{preference direction vectors}, our method supports a-posteriori preference specification without retraining, offering flexibility in downstream design exploration.
% A similar concept has been used in other applications, e.g.\ style-transfer for image synthesis \cite{dosovitskiy2020You}, and we now apply these concepts to generative online black-box \gls{moo} solvers.

We show that our method, \acrfull{agps}, performs well against competing
approaches on a suite of challenging synthetic and real multi-objective
optimization benchmarks.

\begin{figure}[t]
    \begin{minipage}[c]{0.49\textwidth}
    \includegraphics[width=\linewidth]{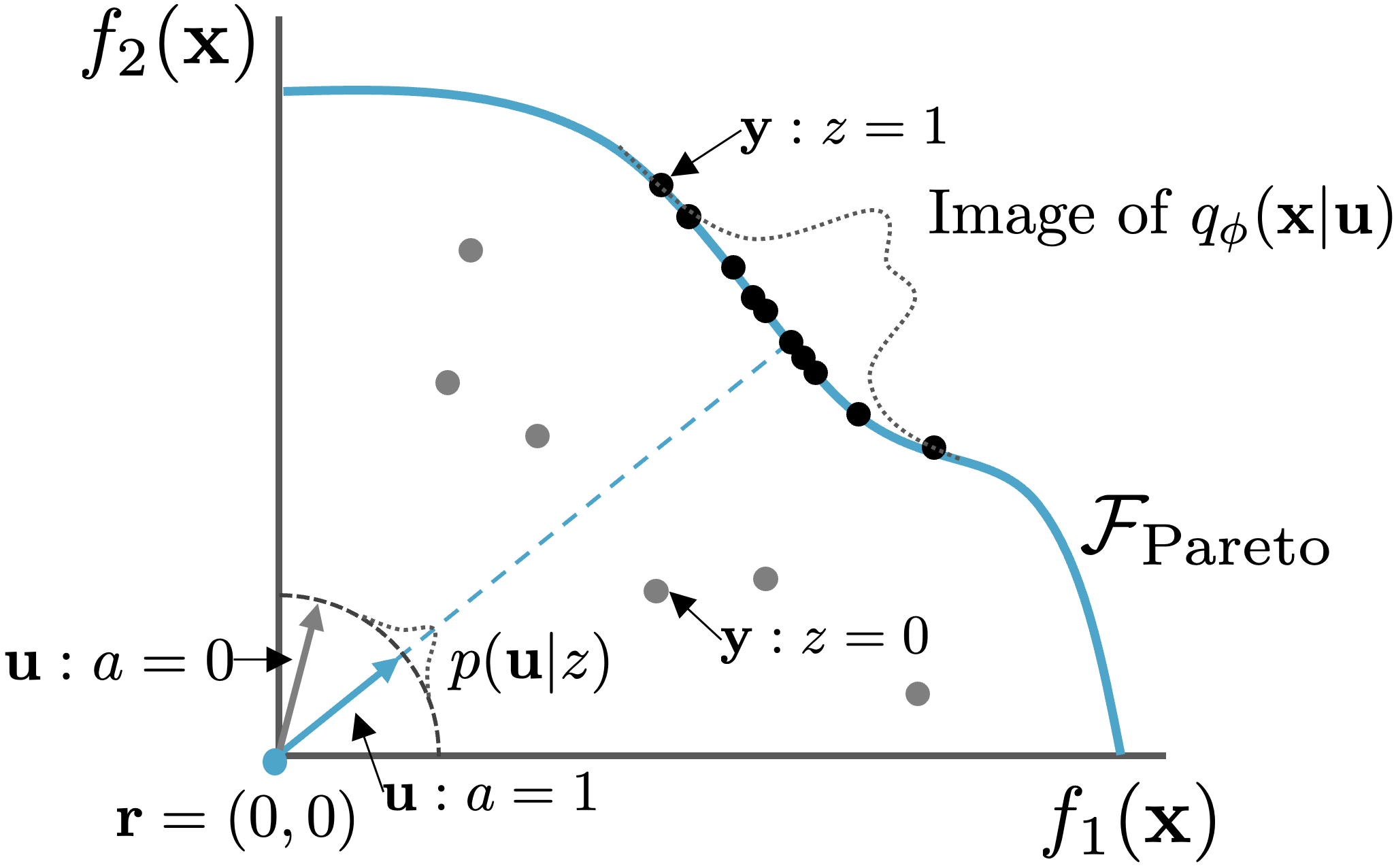}
    \end{minipage}\hfill
    \begin{minipage}[c]{0.49\textwidth}
    \caption{
        A visualization of a Pareto front, $\fpareto$, and the
        random variables used with A-GPS. $\tars$ are noisy realizations of
        the objectives, $\bbfs$. When $\lablf=1$ these observations lie on
        the Pareto front. Preference direction vectors, $\prefr$, are
        unit vectors pointing to a region of the Pareto front from a reference
        point, $\refp$, \autoref{eq:pref_train}. We derive aligned
        ($\lablu=1$) training preference direction vectors, $\prefr_n$, from
        observation pairs $(\tars_n, \obs_n)$, and mis-aligned preference
        direction vectors from permuting these pairs, $(\tars_{\rho(n)},
        \obs_n)$, \autoref{eq:cpe_pref}. The aim is to learn the
        distribution of the Pareto set $\qrobc{\qparam}{\obs}{\prefr} \approx
        \probc{\obs}{\prefr, \lablf=1, \lablu=1}$.
    }
    \label{fig:pareto_projection}
    \end{minipage}
    \vspace{-0.5cm}
\end{figure}

\section{Preliminaries}
\label{sec:method}

% In this section we briefly recall the concepts of active generation and
% black-box \gls{moo}.

\subsection{Active generation}
\label{sec:activegen}

Active generation as implemented by \cite{steinberg2025variational} reframes
online black-box optimization as sequential learning of a conditional
generative model, guided by a \gls{cpe}. At each round, $t \in \{1, \dots, T\}$
we: (1) fit a \gls{cpe} (using some proper loss, $\mathcal{L}_\text{CPE}$),
\begin{equation}
    \cpe{\lablf}{\mparam}{\obs} \approx \probc{\lablf = 1}{\obs},
\end{equation}
parameterized by $\mparam$ and where $\labl = \indic{\obs \in \solnspace}$ indicates
membership in some desired set, $\solnspace$. For example, designs fitter than some
incumbent, $\obs^\star$, under some black box
function, $\bbf: \obsspace \to \real$; $\solnspace := \{\obs \in \obsspace : \bbf(\obs) > \bbf(\obs^\star) \}$.
Then (2) update the generative model
$\qrob{\qparam}{\obs}$, e.g.\ by minimizing the reverse \gls{kl} divergence to
the ideal conditional, $\probc{\obs}{\lablf=1}$, or equivalently maximizing the
\gls{elbo},
\begin{equation}
    \elbo{\qparam, \mparam} =
        \expec{\qrob{\qparam}{\obs}}{\log \cpe{\lablf}{\mparam}{\obs}}
        - \dkl{\qrob{\qparam}{\obs}}{\probc{\obs}{\data_0}}, \label{eq:vsdcpe}
\end{equation}
where $\probc{\obs}{\data_0}$ is a prior over the design space. Formally, using
data $\data^\lablf_N = \{(\obs_n, \lablf_n)\}^N_{n=1}$, active generation
optimizes,
\begin{equation}
    \mparam^*_t \leftarrow \argmin_\mparam \lcpe{\mparam, \data^\lablf_N} \qquad
    \qparam^*_t \leftarrow \argmax_\qparam \elbo{\qparam, \mparam^*_t},
    \label{eq:vsd_opt}%
\end{equation}%
then samples from $\qrob{\qparam^*_t}{\obs}$ are used to propose new candidates
for evaluation. New labels are acquired for these candidates, the dataset
is augmented, and the process is repeated until convergence. This solution to active generation is referred to as \gls{vsd} \cite{steinberg2025variational}. Under certain assumptions
on the form of the models, this procedure has proven convergence rates to
the ideal $\probc{\obs}{\lablf=1}$.

\subsection{Optimizing over multiple objectives}
\label{sec:moo}

In this work we are concerned with generating discrete or mixed discrete-continuous designs,
for example sequences $\obs \in \obsspace = \acidspace^M$, where $\acidspace$ is
the sequence vocabulary and $M$ is the sequence length, that have particular
measurable properties $\tars \in \real^L$.
%Note that in multi-objective optimization, $\tars$ is a multi-dimensional vector of objective values,
%in contrast with the classical scalar optimization settings.
We assume the `black-box'
relationship $\tars = \bbfs(\obs) + \errs$ where $\bbfs(\obs) = [\bbf^1(\obs),
\ldots, \bbf^l(\obs), \ldots, \bbf^L(\obs)]$ and $\expec{\prob{\errs}}{\errs} =
\mathbf{0}$.
The black-box function $\bbfs(\cdot)$ could be a noisy empirical observation, or an expensive
physics/chemistry simulation (where $\errs=\mathbf{0}$), etc.
In \gls{moo} we would like to find the global optimum,
\begin{align}
    \max_{\obs \in \obsspace} \bbfs(\obs).
    \label{eq:moo}
\end{align}
There are a number of issues that present themselves here though.
Firstly, we cannot use gradient based optimization methods directly since we cannot access $\nabla_\obs \bbf^l(\obs)$ as $\bbf^l$ are black-boxes and $\obs$ is (partially) discrete. But more importantly, $\max$ is not uniquely defined for the
vector valued $\bbfs$ as the individual objectives can be in conflict with one (or more formally, there is no total ordering). Instead, we are interested in finding the set of designs for which we cannot increase one objective without compromising
others. This is known as the Pareto set $\pareto^* \subset \obsspace$, % where,
\begin{align}
    \pareto^* := \{ \obs : \obs' \not\succ \obs, ~ \forall \obs' \in \obsspace \},
    \label{eq:pareto}
\end{align}
where $\obs' \succ \obs$ refers to $\obs'$ dominating $\obs$, i.e., all of the objective function values for $\obs'$ are greater than or equal to those of $\obs$, and at least one is greater,
\begin{align}
    \obs' \succ \obs ~ \text{iff} ~
    \bbf^l(\obs') \geq \bbf^l(\obs) ~~ \forall l \in \{1, \ldots, L\}~\text{and}~
    \exists l \in \{1, \ldots, L\}~\text{such that}~\bbf^l(\obs') > \bbf^l(\obs).
    \label{eq:dominate}
\end{align}
The Pareto set also induces the Pareto front, $\fpareto^* := \{\bbfs(\obs) :
\forall\obs \in \pareto^* \}$, which is the image of the Pareto set outcomes in
$\real^L$.
For data $\data_N = \{(\tars_n, \obs_n)\}^N_{n=1}$ at round $t$ we define
the current \emph{observable} Pareto set,
\begin{align}
    \pareto^t &:= \{ \obs_i : \obs_j \not\succ \obs_i,
        ~ \forall i \in \{1, \ldots, N \},
        ~ \forall j \in \{1, \ldots, N \} \backslash i
    \}, \label{eq:apareto} \\
    &\text{where} \quad \obs_j \succ \obs_i ~ \textrm{iff} ~
    \tar^l_j \geq \tar^l_i ~~ \forall l \in \{1, \ldots, L\}
    ~\text{and}~
    \exists l \in \{1, \ldots, L\}~\text{such that}~\tar_j^l> \tar_i^l.
    \nonumber
\end{align}
We define a Pareto set membership (or non-dominance) label $\lablf_n :=
\indic{\obs_n \in \pareto^t}$, where $\indic{\cdot}: \{\text{True},
\text{False}\} \to \{1, 0\}$, which we will use for training the \gls{cpe} in
active generation. This definition can be extended to the whole domain $\obsspace$ as a labeling function $\lablf(\obs) := \indic{\obs \in \paretofn{\pareto^t \cup \{\obs\}}}$, where $\paretofn{\solnspace}$ denotes the Pareto subset of an arbitrary set $\solnspace\subset\obsspace$. Note that the definition remains unchanged for observed points already in the dataset, i.e., $\lablf(\obs_n) = \lablf_n$. As it turns out, this non-dominance \gls{cpe} is also
estimating the probability of hypervolume improvement (\gls{phvi}) for new query points as the following theorem and corollary  (assuming $\errs = \mathbf{0})$. We fully define \gls{hvi} in \autoref{app:proof}, as its definition is a little involved.
\begin{restatable}[Equivalence of Indicators]{theorem}{mainthm}\label{thm:indicator}
    For every $\obs \notin \pareto^t$, the \gls{hvi} indicator is equivalent to a
    non-dominance indicator,
    \begin{equation}
        \indic{\hvi{\obs} > 0} = \lablf(\obs).
    \end{equation}%
\end{restatable}
\begin{restatable}[Non-Dominance \gls{cpe} estimates \gls{phvi}]{corollary}{phvicor}
Following straightforwardly from \autoref{thm:indicator},
\begin{equation}
    \mathbb{P}(\lablf(\obs) = 1 | \obs)
    = \mathbb{P}(\hvi{\obs} > 0 | \obs) \eqqcolon \mathrm{PHVI}(\obs), \quad \forall \obs \notin \pareto^t,
\end{equation}%
as the events are equivalent. Thus, a \gls{cpe} trained on $\lablf$,
using a proper loss, is predicting \gls{phvi}.\qed%
\end{restatable}
See \autoref{app:proof} for the proof and assumptions under which this is
true. We note that, for existing $\obs \in \pareto^t$, we have a discrepancy, as $\lablf(\obs) = 1$, whereas $\hvi{\obs} = 0$. 
However, when the objectives are continuous, constructing the indicators from an existing dataset for training a \gls{cpe} is almost surely equivalent to constructing them in a held-out sense as there will be no ties. Furthermore, if there is observation noise it is beneficial to allow sampling at $\obs$ again for noise reduction (e.g.\ averaging), as the true $\bbfs(\obs)$ may be dominated even if a single observation $\tars$ at $\obs$ suggests otherwise.
%Yet, due to observation noise, it is beneficial to allow sampling at $\obs$ again for noise reduction (e.g.\ averaging), as the true $\bbfs(\obs)$ may be dominated even if a single observation $\tars$ at $\obs$ suggests otherwise. In addition, if required, the \gls{cpe} can be forced to zero at observed points by external mechanisms.

\section{Incorporating User Preferences}
% \subsection{Preference Vectors and Alignment Indicators}
\label{sec:prefdirs}

In multi-objective optimization (MOO),  practitioners invoke \emph{subjective
preferences} to single out a subset of designs to meet application-specific requirements. Ideally, we would like not only to incorporate these subjective preferences but also to avoid retraining our active generation framework every time a new preference is given.

A standard approach to incorporating subjective preferences is
\emph{scalarization}: e.g.\ for convex scalarization we specify a weight vector
$\pweights\in\mathbb{R}^L$ with $\|\pweights\|_1=1$ and maximize
\begin{equation}
    \argmax_\obs s_{\pweights, \refp}(\obs), \quad \text{where} \quad s_{\pweights, \refp}(\obs) = \pweights^\top(\bbfs(\obs) - \refp)
\end{equation}
where $\refp \in \mathbb{R}^L$ is a reference point \citep{knowles2006parego,
zhang2007moea, zhang2007moea, paria2020flexible}. While scalarization blends
objectives according to explicit trade-off weights, it is not optimal when
learning a conditional generative model of the Pareto set via a \gls{cpe} on labels $\lablf_n = \indic{s_{\pweights, \refp}(\obs_n) > \thresh}$ for some threshold $\thresh \in \real$, as in \cite{de2022mbore}, since each new $\boldsymbol\lambda$ would require retraining. Furthermore, thresholding weighted objectives can blur the
non-dominance boundary compared to directly labeling it.

\subsection{Preference direction vectors and alignment indicators}
As we will see in \autoref{sec:vsd_moo}, our solution to incorporating user preferences for active generation is based on amortization. In other words, instead of estimating a model $\qrob{\qparam}{\obs}$ as in \gls{vsd}, we will learn a conditional model of the form $\qrobc{\qparam}{\obs}{\prefr}$. Consequently, instead of scalarization, we introduce
\emph{preference direction vectors} $\prefr \in \prefrspace$ where $\prefrspace
= \{\prefr \in \real^L : \|\prefr\|_2 = 1\}$, defined from observed or desired (subjective) user specified outcomes. In our experiments, we train our method using
\begin{equation}
    \prefr_n = g( \tars_n) := \frac{\tars_n- \refp}{\| \tars_n - \refp \|_2}.
    \label{eq:pref_train}
\end{equation}
These unit vectors capture the relative emphasis among objectives in a single
geometric object. Given a trained model, a user can specify their own
preferences via $\prefr_{\star} = g(\tars_{\star})$ and our approach will
generate solutions from $\qrobc{\qparam}{\obs}{\prefr_{\star}}$. Importantly,
our generative model needs to enforce that generated samples respect a user's
desired trade-off. Therefore, we define an \emph{alignment indicator}, $\lablu
\in \{0, 1\}$, that labels each $(\obs, \prefr)$ pair as `aligned' if it
achieves correct projection onto the preference direction. We will make clear the need
for this indicator variable in the next section.

Preference directions generalize (convex) scalarization weights to a
(non-convex) generative setting: any $\boldsymbol\lambda$ can be mapped to a
unit-norm vector $\mathbf{u} = \boldsymbol\lambda / \|\boldsymbol\lambda\|_2$,
and conversely each $\mathbf{u}$ induces a unique normalized weight. 
In fact \citep{cheng2016reference, chugh2018surrogate} use ``reference vectors'', which are similar to our $\prefr$, for scalarization of objectives and to promote diversity for \glspl{moea}. 
By conditioning on $(\prefr, \lablu)$ rather than scalarizing by $\pweights$, our
generative Pareto-set model becomes both more flexible (no retraining for new
trade-offs) and more faithful to non-dominance structure. We visualize these
preference direction vectors in \autoref{fig:pareto_projection}.

\section{Amortized Active Generation of Pareto Sets}
\label{sec:vsd_moo}

We now have all the components to describe our amortized active generation framework that learns to generate (approximate) solutions in the Pareto set, conditioned on user preferences. We call our method \acrfull{agps}, and it begins by generalizing the active generation objective in \cite{steinberg2025variational}. That is, for each round, $t$, we minimize the reverse \gls{kl} divergence
between the generative model $\qrobc{\qparam}{\obs}{\prefr}$
and an underlying (unobserved) true model $\probc{\obs}{\prefr, \lablf, \lablu}$,
\begin{align}
    \qparam^*_t &= \argmin_\qparam \dkl{
        \qrobc{\qparam}{\obs}{\prefr}\probc{\prefr}{\lablf}
    }{
        \probc{\obs}{\prefr, \lablf, \lablu}\probc{\prefr}{\lablf}
    }, \nonumber \\
    &= \argmin_\qparam \expec{\probc{\prefr}{\lablf}}{\dkl{
        \qrobc{\qparam}{\obs}{\prefr}
    }{
        \probc{\obs}{\prefr, \lablf, \lablu}
    }}.
    \label{eq:ekl_moo}
\end{align}
Here we are actually conditioning on $z=1$ and $a=1$, however we leave this implicit henceforth to avoid notational clutter. The inclusion of $\probc{\prefr}{\lablf}$ rewards learning an \emph{amortized}
generative model, $\qrobc{\qparam}{\obs}{\prefr}$, over the distribution of the
relevant preference directions. Naturally we cannot evaluate
$\probc{\obs}{\prefr, \lablf, \lablu}$ directly, and so we appeal to Bayes'
rule,
\begin{equation}
    \probc{\obs}{\prefr, \lablf, \lablu} = \frac{1}{Z}
        \probc{\lablf}{\obs, \prefr}
        \probc{\lablu}{\obs, \prefr}
        \probc{\obs}{\prefr}.
\end{equation}
Here we have assumed conditional independence between $\lablf$ and $\lablu$ given $\obs$ and $\prefr$,
% $Z = \probc{\lablf}{\prefr}\probc{\lablu}{\prefr}$
and since
$Z = \probc{\lablf,\lablu}{\prefr}$
is a constant w.r.t.\ $\obs$,  we
will omit it from our objective. We make a further simplifying assumption that
a-priori $\probc{\obs}{\prefr} = \probc{\obs}{\data_0}$, and then we rely on
the likelihood guidance terms, $\probc{\lablf}{\obs, \prefr}
\probc{\lablu}{\obs, \prefr}$, to capture the joint relationship between
$(\obs, \prefr)$ in the variational posterior. We justify this decision by
noting that the \emph{alignment} relationship, $\obs | \prefr$, may be difficult to reason about
a-priori, and requiring such a prior would then preclude the use of pre-trained
models for $\probc{\obs}{\data_0}$. Putting this all together results  in the
following equivalent amortized \gls{elbo} objective,
\begin{align}
    \qparam^*_t &= \argmax_\qparam \aelbo{\qparam} \quad \text{where}, \\
    \aelbo{\qparam} &= \expec{\probc{\prefr}{\lablf}}{\elbo{\qparam}}, \nonumber \\
    &= \expece_{\!\!\!\!\!\!
            \underbrace{\scriptstyle\probc{\prefr}{\lablf}}_{\scriptstyle\text{Direction dist.}}
        \!\!\!\!\!\!}\!\big[
            \expece_{\qrobc{\qparam}{\obs}{\prefr}}\!
        \big[\!%
                \log \underbrace{\probc{\lablf}{\obs, \prefr}}_{\text{Pareto CPE}} +
                \log \underbrace{\probc{\lablu}{\obs, \prefr}}_{\text{Align. CPE}}
        - \beta \dkl{\qrobc{\qparam}{\obs}{\prefr}}{\probc{\obs}{\data_0}}\!
        \big]\!
    \big].
    \label{eq:aelbo}
\end{align}
The $\beta$ coefficient appears here to control the objective's
exploration-exploitation tradeoff, and where $\beta=1$ results in the exact
minimization of \autoref{eq:ekl_moo}. We will now discuss how we estimate each
of these components in turn, leading to the \gls{agps} algorithm presented in
\autoref{alg:optloop}.
\subsection{Estimating \gls{agps}'s component distributions}
\paragraph{Preference direction distribution, $\probc{\prefr}{\lablf}$.}
Since we observe $\prefr_n$, we can approximate empirically
$
    \probc{\prefr}{\lablf} \approx
    % \card{\pareto^t}^{-1}
    (\sum_{n=1}^N \lablf_n)^{-1}
    \sum_{n=1}^N \lablf_n \, \indic{\prefr = \prefr_n}
$.
Alternatively, we can use maximum likelihood to
learn a parameterized estimator $\qrob{\pparam}{\prefr} \approx
\probc{\prefr}{\lablf}$, with data $\data_N^\lablf = \{(\obs_n, \prefr_n,
\lablf_n)\}^N_{n=1}$,
\begin{equation}
    \pparam^*_t =\argmin_\pparam \lpref{\pparam, \data_N^\labl},
    \quad \text{where} \quad
    \lpref{\pparam, \data^\lablf_N} =
        % - \frac{1}{\card{\pareto^t}}
        - \frac{1}{\sum_{n=1}^N \lablf_n}
        \sum\nolimits^N_{n=1}
        \lablf_n \log \qrob{\pparam}{\prefr_n}.
    \label{eq:ml_pref}
\end{equation}
We find this occasionally aids exploration.
Examples of appropriate parametric forms are von Mises-Fisher distributions,
power spherical distributions \citep{de2020power} or normalizing
flows~\citep{rezende2020normalizing}. We find Normal distributions, or
mixtures, normalized to the unit sphere are more numerically stable than some
of the specialized spherical distributions, see \autoref{app:pref_dist} for
more detail. We also find that fitting an unconditional $\qrob{\pparam}{\prefr}
\approx \prob{\prefr}$ for just the initial round can aid exploration of
the Pareto front.

\paragraph{Pareto \gls{cpe}, $\probc{\lablf}{\obs, \prefr}$.}
As per the original \gls{vsd}, we define a \gls{cpe} to directly discriminate
over the solution set, $\solnspace$. Sometimes we find setting $\solnspace =
\pareto^t$ leads to overly exploitative behavior. Instead, we anneal the set
using the Pareto ranking method described in \cite{deb2002fast}, and define the
labels based on a thresholded rank, $k$: $\lablf_n = \indic{\obs_n \in
\{\bigcup_k \pareto^t{_{,k}} : \forall k \leq \thresh_\round\}}$. Here $k=1$
indicates the Pareto set, $k=2$ the next non-dominated set once $\pareto^t{_{,1}}$
is removed, etc. We use $\thresh_\round = \threshfn{\{\tars: \tars \in
\data_N\}, \threshparam_\round}$ as presented in
\cite{steinberg2025variational}, ensuring $\thresh_T$ labels just $\pareto^t$.
The annealed method (Equation 20) applied to quantiles of $-k$ is particularly
effective. So, with this label, we use the log-loss to train a \gls{cpe},
\begin{equation}
    \lcpef{\mparam, \data_N^\labl} = - \frac{1}{N} \sum\nolimits_{n=1}^N
    \lablf_n \log \cpe{\lablf}{\mparam}{\obs_n, \prefr_n}
    + (1-\lablf_n) \log (1 - \cpe{\lablf}{\mparam}{\obs_n, \prefr_n}),
    \label{eq:cpe_pareto}
\end{equation}
where $\cpe{\lablf}{\mparam}{\obs, \prefr}$ is a discriminative model parameterized by
$\mparam$, e.g.~a neural network.
%To acquire labels, $\lablf_n$, we could make
% use of fast dominance checking \cite{kung1975finding} to implement
% \autoref{eq:apareto}, in practice we find the dominance checker in
% \cite{balandat2020botorch} sufficient for our purposes.

\paragraph{Preference alignment \gls{cpe}, $\probc{\lablu}{\obs, \prefr}$.}
Since we do not wish to rely on a strong prior, $\probc{\obs}{\prefr}$, for
our sole-source of preference alignment information, we instead explicitly
reward alignment in our conditional generative model by using a \gls{cpe}
guide. We create contrastive data for training this guide, $\data^\lablu_N =
\{(\lablu_n\! =\! 1, \obs_n, \prefr_n)\}^{N}_{n=1} \cup \bigcup^P_{i=1}
\{(\lablu_n\! = \! 0, \obs_n, \prefr_{\rho_i(n)}\}^{N}_{n=1}$ where the
second set are purposefully
misaligned by permutations, $\rho_i : \natno \to \natno$. This results in
the log-loss,
\begin{equation}
    % \lcpeu{\uparam, \data_N^\lablu} :=& - \frac{1}{2N} \sum\nolimits_{n=1}^{2N}
    % \lablu_n \log \cpe{\uparam}{\obs_n, \prefr_n}
    % + (1-\lablu_n) \log (1 - \cpe{\uparam}{\obs_n, \prefr_n}).
    \lcpeu{\uparam, \data_N^\lablu} =
    - \frac{1}{N+PN} \left[
        \sum_{n=1}^{N} \log \cpe{\lablu}{\uparam}{\obs_n, \prefr_n}
    + \sum_{i=1}^P \sum_{n=1}^{N}
        \log (1 - \cpe{\lablu}{\uparam}{\obs_n, \prefr_{\rho_i(n)}})%
    \right],
    \label{eq:cpe_pref}
\end{equation}
where $\cpe{\lablu}{\uparam}{\obs, \prefr}$ is our \gls{cpe} parameterized by
$\uparam$. We make use of two permutation methods for creating the contrastive data. The
first is to just use random permutation without allowing any random alignments.
The second is to use the top-$k$ nearest neighbors, based on $\prefr_n$ cosine
distance, for $k$ replicates. The random permutation contrastive data covers the
space of misalignment, while the top-$k$ permutations improves the angular
precision of the alignment scoring \gls{cpe}. For all experiments we use 7
random permutation replicates, and 2 top-2 replicates, for a total of $P=9$.

\subsection{Learning \gls{agps}'s variational distribution}
To learn $\qrobc{\qparam}{\obs}{\prefr}$,  we can now re-write our amortized \gls{elbo}, \autoref{eq:aelbo}, in terms of
these estimated quantities,
\begin{equation}
    \aelbo{\qparam, \mparam, \uparam, \pparam} =
        \expec{\qrob{\pparam}{\prefr}}{
            \expec{\qrobc{\qparam}{\obs}{\prefr}}{
                \log \cpe{\lablf}{\mparam}{\obs, \prefr} +
                \log \cpe{\lablu}{\uparam}{\obs, \prefr}
            }
        - \beta \dkl{\qrobc{\qparam}{\obs}{\prefr}}{\probc{\obs}{\data_0}}
    }.
    \label{eq:agps}
\end{equation}
% \cheng{We should note as future work, that it would be interesting how to combine different CPEs, as opposed to just an unweighted sum. Boosting? Bagging? A neural network perhaps?}
We find that using `on-policy' gradient estimation methods such as REINFORCE
\cite{williams1992simple, mohamed2020monte} are very slow when we have complex
variational distribution forms, $\qrobc{\qparam}{\obs}{\prefr}$, e.g.\ causal
transformers. This is because we have to set a low learning rate to avoid the
variance of this estimator inducing exploding gradients for long sequences. Also,
new samples have to be drawn from the variational distribution every
iteration of \gls{sgd}, which can be computationally expensive.
So instead we use an `off-policy' gradient estimator
with importance weights to emulate the on-policy estimator,
%
% To compute the gradients for optimization we use an `on-policy' score function estimator (REINFORCE) \citep{williams1992simple,mohamed2020monte},
% \begin{align}
%     \nabla_\qparam & \elbo{\qparam, \mparam, \uparam, \pparam} \nonumber \\
%         &\quad = \expec{\qrobc{\qparam}{\obs}{\prefr} \qrob{\gamma}{\prefr}}
%         {\left(\log \cpe{\mparam}{\obs, \prefr} + \log \cpe{\uparam}{\obs, \prefr}
%        - \log \frac{\qrobc{\qparam}{\obs}{\prefr}}{\probc{\obs}{\data_0}}
%        \right)
%         \nabla_\qparam \log \qrobc{\qparam}{\obs}{\prefr}}\!.
%     \label{eq:on-pol-grads}
% \end{align}
% Where we use $S$ samples, $\prefr^{(s)} \sim \qrob{\gamma^*_t}{\prefr}$, and
% $\obs^{(s)} \sim \qrobc{\qparam}{\obs}{\prefr^{(s)}}$ for Monte Carlo
% approximation of the expectation, and then optimize \autoref{eq:vsd_moo} using
% \gls{sgd} (e.g.~Adam \citep{kingma2014adam}) with suitable control-variates as
% in \citet{daulton2022bayesian, steinberg2025variational}.
%This method requires new $S$ samples each iteration of \gls{sgd}.
%
%
% For complex variational distributions, such as \glspl{lstm} and Transformers,
% on-policy estimators can be slow as they constantly need \emph{new samples each
% iteration of training}. For these situations we consider an `off-policy'
% gradient estimation alternative using importance weights,
%
%{\small
\begin{align}
    &\nabla_\qparam \aelbo{\qparam, \mparam, \uparam, \pparam} \nonumber \\
        &~~ = \expec{\qrobc{\qparam'}{\obs}{\prefr} \qrob{\gamma}{\prefr}}{
            w(\obs, \prefr) \cdot \!
            \left(\log \cpe{\lablf}{\mparam}{\obs, \prefr} + \log \cpe{\lablu}{\uparam}{\obs, \prefr}
            - \beta \log \frac{\qrobc{\qparam}{\obs}{\prefr}}{\probc{\obs}{\data_0}}
            \right) \nabla_\qparam \log \qrobc{\qparam}{\obs}{\prefr}}\!.
    \label{eq:off-pol-grads}
\end{align}
%}
Here $w(\obs, \prefr) = \qrobc{\qparam}{\obs}{\prefr} /
\qrobc{\qparam'}{\obs}{\prefr}$ are the importance weights
\citep{precup2000eligibility, burda2016importance}. Now we use $S$ samples from
$\obs^{(s)} \sim \qrobc{\qparam'}{\obs}{\prefr^{(s)}}$, to approximate the
expectation in \autoref{eq:off-pol-grads}. If we choose
$\qparam' = \qparam$ we recover on-policy gradients, however we typically
only update $\qparam'$ every 100 iterations of optimising A-\gls{elbo}, or if
the effective sample size drops below a predetermined threshold
($0.33S$). Whenever we update $\qparam'$ we also resample $S$ samples.
We use these estimated gradients with an appropriate \gls{sgd} algorithm, such
as Adam~\cite{kingma2014adam}, to optimize for $\qparam^*_t$.

\subsection{Generating Pareto set candidates for evaluation}

To recommend candidates for black-box evaluation in round $t$, we sample a set of $B$ designs from our
search distribution,
\begin{align}
    \{\obs_{bt}\}^B_{b=1} \sim \prod^B_{b=1} \qrobc{\qparam_t^*}{\obs}{\prefr_{bt}},
    \quad \textrm{where} \quad
    \qparam_t^* = \argmax_\qparam \aelbo{\qparam, \mparam^*_t, \uparam^*_t, \pparam^*_t}.
\end{align}
We are free to choose $\prefr_{bt}$ based on preferences
($\prefr_\star$); or if we do not have specific preferences to incorporate into
the query, we sample $\{\prefr_{bt}\}^B_{b=1} \sim \prod^B_{b=1}
\qrob{\pparam^*_t}{\prefr}$ for broad Pareto front exploration.
%We present the full online optimization algorithm in \autoref{alg:optloop}.
%
\begin{algorithm}[t]
    \small
    \caption{
        \gls{agps} optimization loop. See \autoref{fig:agps-loop} for a visual
        representation.
    }
    \label{alg:optloop}
    \begin{algorithmic}[1]
        \Require{
            Initial dataset $\data_N$,
            black-box $\bbfs$,
            prior $\probc{\obs}{\data_0}$,
            \gls{cpe}s $\cpe{\lablf}{\mparam}{\obs, \prefr}$ and $\cpe{\lablu}{\uparam}{\obs, \prefr}$,
            variational families $\qrob{\pparam}{\prefr}$ and $\qrobc{\qparam}{\obs}{\prefr}$,
            threshold function $f_\thresh$ and $\threshparam_1$,
            budget $T$ and $B$.
        }
        \Function{FitModels}{$\data_N, \thresh$}
            \Let{$\data^\lablf_N$}{$\{(\lablf_n, \obs_n, \prefr_n)\}^N_{n=1}$, ~
                where
                $\lablf_n = \indic{\obs_n \in \{\bigcup_k\pareto^t{_{,k}} : \forall k \leq \thresh\}}$,
                $\prefr_n = (\tars_n - \refp) / \| \tars_n - \refp \|$}
            \Let{$\data^\lablu_N$}{
                $\{(\lablu_n \!=\! 1, \obs_n, \prefr_n)\}^N_{n=1} \cup
                \bigcup_{i=1}^P \{(\lablu_n \! = \! 0, \obs_n, \prefr_{\rho_i(n)})\}^N_{n=1}$
            }
            \Let{$\pparam^*$}{$\argmin_\pparam \lpref{\pparam, \data^\lablf_N}$}
            \Let{$\mparam^*$}{$\argmin_\mparam \lcpef{\mparam, \data^\lablf_N}$}
            \Let{$\uparam^*$}{$\argmin_\uparam \lcpeu{\uparam, \data^\lablu_N}$}
            \Let{$\qparam^*$}{$\argmax_\qparam \aelbo{\qparam, \mparam^*, \uparam^*, \pparam^*}$}
            \State \Return{$\qparam^*, \mparam^*, \uparam^*, \pparam^*$}
        \EndFunction
        \For{round $t \in \{1, \ldots, T\}$}
            \Let{$\thresh_\round$}{$\threshfn{\{\tars: \tars \in \data_N\}, \threshparam_t}$}
            \Let{$\qparam_t^*, \mparam_t^*, \uparam^*_t, \pparam_t^*$}{\Call{FitModels}{$\data_N, \thresh_\round$}}
            \Let{$\{\prefr_{bt}\}^B_{b=1}$}{sample $\qrob{\pparam_t^*}{\prefr}$ or use $\prefr_\star$}
            \Let{$\{\obs_{bt}\}^B_{b=1}$}{sample $\qrobc{\qparam_t^*}{\obs}{\prefr_{bt}} ~~ \forall b \in \{1, \dots, B\}$}
            \Let{$\{\tars_{bt}\}^B_{b=1}$}{$\{\bbfs(\obs_{bt}) + \errs_{bt}\}^B_{b=1}$}
            \Let{$\data_{N}$}{$\data_N \cup \{(\obs_{bt}, \tars_{bt})\}^B_{b=1}$}
        \EndFor
        \State \Return{$\data_N, \qparam^*_T, \mparam^*_T, \uparam^*_T, \pparam^*_T$}
    \end{algorithmic}
\end{algorithm}
%
% E.g.~random scalarizations are a nice methods for this, \citep{paria2020flexible}, but I'm not sure how straightforward they would be to apply. Maybe encoding the preference as the direction (unit norm) to the part of the Pareto set we are interested in:
% \begin{itemize}
%     \item CPE: $\cpe{\mparam}{\obs, \prefr} \approx \probc{\lablf}{\obs, \prefr}$
%     \item Variational distribution: $\qrobc{\qparam}{\obs}{\prefr}$
%     \item Preference directions distribution: $\qrob{\gamma}{\prefr} \approx \prob{\prefr}$
% \end{itemize}
% Where $\prefr$ are preference directions, and the \gls{cpe} is trained with $\prefr_n = \mathbf{\tar}_n / ||\mathbf{\tar}_n||$, i.e.~the unit-vector of measured objectives. When we want to sample $\obs^{(s)} \sim \qrobc{\qparam}{\obs}{\prefr}$ we can manually set $\prefr$ exactly, or place a prior on it, e.g.~uniform over the sphere or empirical for no preference, or normalized Gaussian etc.
%
\section{Related Work}
\label{sec:related}

Our work sits at the intersection of online black-box optimization, generative
modeling, and user-guided multi-objective search. We organize existing methods
along three dimensions: whether they operate online or offline, whether
they directly optimize acquisition functions or learn conditional generative
models for optimization, if they use inference-time guidance or learning for generation. % or
%We also discuss our choice of preference direction vectors and contrast it to
% scalarization for incorporating preferences.

\paragraph{Online vs.\ offline.} Traditional \gls{mobo} methods, such as
hypervolume-based acquisition (\gls{ehvi}, \gls{nehvi}, and their variants),
entropy search \citep{yang2019efficient, daulton2020differentiable,
daulton2021parallel, hernandez2016predictive} and scalarization methods
\citep{knowles2006parego, zhang2007moea, paria2020flexible, de2022mbore},
operate online by sequentially querying the black-box using acquisition rules
that balance exploration and exploitation.  In contrast, offline \gls{mog}
approaches like ParetoFlow and guided diffusion frameworks
\citep{yuan2024paretoflow, yao2024proud, annadani2025preference} train
generative models from a fixed dataset of evaluated designs, without further
oracle queries. While these offline methods can leverage rich generative
priors, they have not been designed to adapt to new information.
%E.g.\
%\cite{yao2024proud} uses the gradients from \gls{mgd} as a guide that only
%considers exploitation of high fitness designs, in a guided diffusion
%framework.
% Only~\cite{annadani2025preference} explicitly handles preferences in
% the offline setting. Online generative frameworks such as
% \cite{gruver2023protein, lin2022pareto, jain2023multi} are most similar to
% \gls{vsd}-\gls{moo}.

% \paragraph{Preference Handling.} Many \gls{mobo} techniques require preferences
% to be specified a-priori via scalarization weights or utility functions
% \cite{zhang2007moea, yuan2024paretoflow}, which then remain fixed throughout
% optimization (ParetoFlow uses random crossing of these fixed weights).
% Alternatively, some random scalarization methods, such as
% \cite{knowles2006parego, paria2020flexible, de2022mbore}, require a prior
% (which can be uniform) to be specified over scalarization weights that are then
% randomly sampled during optimization. Others try to elicit preferences
% directly, and then use a preference model as a guide
% \citep{annadani2025preference}. A-posteriori methods, on the other hand, first
% estimate the entire Pareto front and allow users to impose preferences
% afterwards. This separation offers greater flexibility in downstream
% decision-making but often relies on costly hypervolume representation
% \citep{stanton2022accelerating, gruver2023protein} or explicit Pareto set
% models \citep{lin2022pareto}.
% \dan{TODO: make this about scalarization vs. EHVI vs. direct preferences and
% then what we do. In retrospect, I don't like the a-priori/posteriori narrative.}

\paragraph{Generative models vs.\ acquisition optimization.} Recent advances in
``active generation'' recast black-box optimization as fitting conditional
generative models to high-value regions, guided by predictors and/or
acquisition functions. Methods like \gls{vsd}, GFlowNets, and diffusion-based
solvers \citep{steinberg2025variational, jain2023multi, dhariwal2021diffusion,
gruver2023protein} show that generative search can match or exceed traditional
direct acquisition function optimization, particularly in large search spaces.
However, existing generative frameworks often need re-training to integrate subjective preferences.
Similarly, all direct acquisition optimization methods require additional
optimization runs to incorporate new preferences. An exception is Pareto set
learning \cite{lin2022pareto}, which learns a
neural-net, that maps from scalarization weights to
designs.% in the Pareto set.

\paragraph{Guidance vs.\ learning.} Or inference-time vs.\
re-training/fine-tuning based search. Guided generation methods, such as those
based on guided diffusion and flow matching \citep{gruver2023protein,
yao2024proud, yuan2024paretoflow, annadani2025preference} use a pre-trained
generative model, from which samples are then \emph{guided} at inference time
such that they are generated from a conditional generative model, leaving the
original generative model unchanged. It has been noted in
\cite{klarner2024context} that guided methods, though computationally
efficient, may be prone to co-variate shift preventing them being guided too
far from the support of the pretrained model. Conversely, learning-based
methods such as \cite{steinberg2025variational,wang2025finetuning} explicitly
re-train or fine-tune the generative model to condition it, thereby
circumventing these co-variate shift issues at the cost of more computation,
but allowing for less constrained exploration in online scenarios.

Our \gls{agps} approach unifies these dimensions: it learns an amortized
conditional generative model online, bypasses explicit acquisition optimization,
and uses sequential learning to avoid co-variate
shift. \autoref{tab:comparison} compares key features across representative
\gls{mog} methods.

% \paragraph{Scalarization vs.\ Direction Vectors.} We could have made an
% alternative design choice for incorporating preferences into
% \gls{vsd}-\gls{moo}, which is to explicitly use scalarization weights,
% $\pweights$, instead of direction vectors, $\prefr$. In particular, we could
% use MBORE \citep{de2022mbore}, which applies \gls{bore} \citep{tiao20:w21bore} to
% randomly scalarized objectives, for our \gls{cpe}/likelihood model in
% \gls{vsd}. We chose against this approach since a \gls{cpe} constructed from
% scalarized and thresholded objectives does not provide any benefit for
% \gls{mog} and still requires the construction of an alignment \gls{cpe}.
% Futhermore, this method of scalarization assumes that all objectives have a
% similar scale, or are scaled as such a-priori. This is hard to achieve in
% practice in our experience with synthetic biology tasks.

\begin{table}[tb]
    \caption{Comparison of recent \gls{mog} and related techniques. `\cmark'
    means the method has the feature, `\xmark' the method lacks the feature and
    `--' the method can be easily extended to incorporate the feature. `Modular'
    refers to the non-specific nature of the variational distribution used by
    \gls{cbas}, \gls{vsd} and \gls{agps}, i.e., it can be chosen based on the task.
    }
    \scriptsize
    \centering
    \begin{tabular}{r|c|c|c|c|c|c|c|c}
    \textbf{Method}
        & \rot{Designed for \gls{moo}}
        & \rot{Online \gls{bbo}}
        & \rot{Amortized preference conditioning}
        & \rot{Non-convex Pareto front}
        & \rot{Discrete/mixed $\obsspace$}
        & \rot{Generative pref.\ model, $\qrob{\pparam}{\prefr}$}
        & \rot{Generative obs.\ model, $\qrob{\qparam}{\obs}$}
        & Guide \\
    \hline
    LaMBO~\cite{stanton2022accelerating} & \cmark & \cmark & \xmark & \cmark & \cmark & \xmark &  Masked LM & \gls{nehvi} \\
    LaMBO-2~\cite{gruver2023protein} & \cmark & \cmark & \xmark & \cmark & \cmark & \xmark & Diffusion & \gls{nehvi} \\
    Pareto Set Learning (PSL)~\citep{lin2022pareto} & \cmark & \cmark & \cmark & \cmark & \xmark & \xmark & Deterministic MLP & Scalarization \\
    GFlowNets~\citep{jain2023multi} & \cmark & \cmark & \cmark & \cmark & \cmark & \xmark & GFlowNets & Scalarization \\
    ParetoFlow~\citep{yuan2024paretoflow} & \cmark & \xmark & \xmark & \xmark & \cmark & -- & Diffusion & Scalarization \\
    PROUD~\citep{yao2024proud} & \cmark & \xmark & \xmark & \cmark & -- & \xmark & Diffusion & Multiple grad.\ desc. \\
    Preference Guided Diffusion~\citep{annadani2025preference} & \cmark & \xmark & \xmark & \cmark & \xmark & \xmark & Diffusion & Preference \gls{cpe} \\
    \Gls{cbas}~\citep{brookes2019conditioning} & -- & -- & \xmark & \cmark & \cmark & \xmark & Modular & Dominance \gls{cpe} \\
    \gls{vsd}~\citep{steinberg2025variational} & -- & \cmark & \xmark & \cmark & \cmark & \xmark & Modular & Dominance \gls{cpe} \\
    \gls{agps} (ours) & \cmark & \cmark & \cmark & \cmark & \cmark & \cmark & Modular & Dominance \gls{cpe} \\
    \hline
    \end{tabular}
    \label{tab:comparison}
    \vspace{-1em}
\end{table}

\section{Experiments}
\label{sec:experiments}

We now evaluate \gls{agps} on a number of benchmarks and compare it to some
popular baselines. Firstly we apply \gls{agps} to a number of well known
\emph{continuous} synthetic \gls{moo} test functions as a proof-of-concept.
Then we apply it to three high-dimensional sequence design challenges ---
\gls{agps}'s intended application --- that emulate real protein engineering
tasks. Our primary measure of performance is Pareto front relative \gls{hvi}
\cite{zitzler1999multiobjective, fonseca2006improved} using the implementation
in \cite{balandat2020botorch}. For all experiments we set $\beta = 0.5$ as the
full \gls{kl} regularization in \autoref{eq:aelbo} can hamper exploitation in
later rounds on some tasks. We refer the reader to \autoref{app:experiments}
for full experimental details.

\subsection{Synthetic test functions}
\label{sub:exp_synth}

As a proof-of-concept, we demonstrate \gls{agps} on some classical continuous
paramterized \gls{moo} problems ($\obs \in \real^D$) commonly used for
\gls{mobo}~\cite{zhang2007moea, belakaria2019max, balandat2020botorch}. Even
though \gls{agps} has not been designed for purely continuous problems, they
none-the-less allow us to demonstrate some appealing properties of \gls{agps}.
We present three here; negative \textbf{Branin-Currin} ($D=2$, $L=2$),
\textbf{DTLZ7} ($D=7$, $L=6$), and \textbf{ZDT3} ($D=4$, $L=2$), see
\cite{deb2002scalable,zhang2007moea,belakaria2019max}. More detailed
descriptions of these functions, additional experimental detail and  on
additional test functions are presented in \autoref{app:syn_fn}.
We apply \gls{agps} to these continuous domains using a conditional Gaussian
generative model, $\qrobc{\qparam}{\obs}{\prefr} =
\normalc{\obs}{\boldsymbol{\mu}(\prefr), \boldsymbol{\sigma}^2(\prefr)}$, with
mean and variance parameterized by a \gls{nn}.
The top row of  \autoref{fig:syn_funcs} reports mean relative hypervolume
versus optimization round. The bands indicating $\pm 1$ std.\ from 10 runs with
random parameter initialisation. We compare to three \gls{gp}-based baselines,
qNEHVI \cite{daulton2021parallel}, qEHVI and qNParEGO
\cite{daulton2020differentiable}. All methods use 64 training points, and then
recommend $B=5$ candidates for $T=10$ rounds, and \gls{agps} has $\thresh_0$
set using the $p=0.25$ percentile of Pareto ranks. On Branin-Currin, \gls{agps}
(purple) rapidly outpaces the scalarization based qNParEGO, but is dominated by
the methods that explicitly estimate \gls{ehvi}. However, these methods do not
scale to the higher dimensional DTLZ7 problem, where A-GPS still outperforms
qNParEGO. While \gls{agps} can model the complex Pareto front of ZDT3, the
direct optimization methods outperform it, we suspect a stronger generative
backbone for continuous data would help here.
\begin{figure}[tb]
    \centering
    \includegraphics[width=0.329\linewidth]{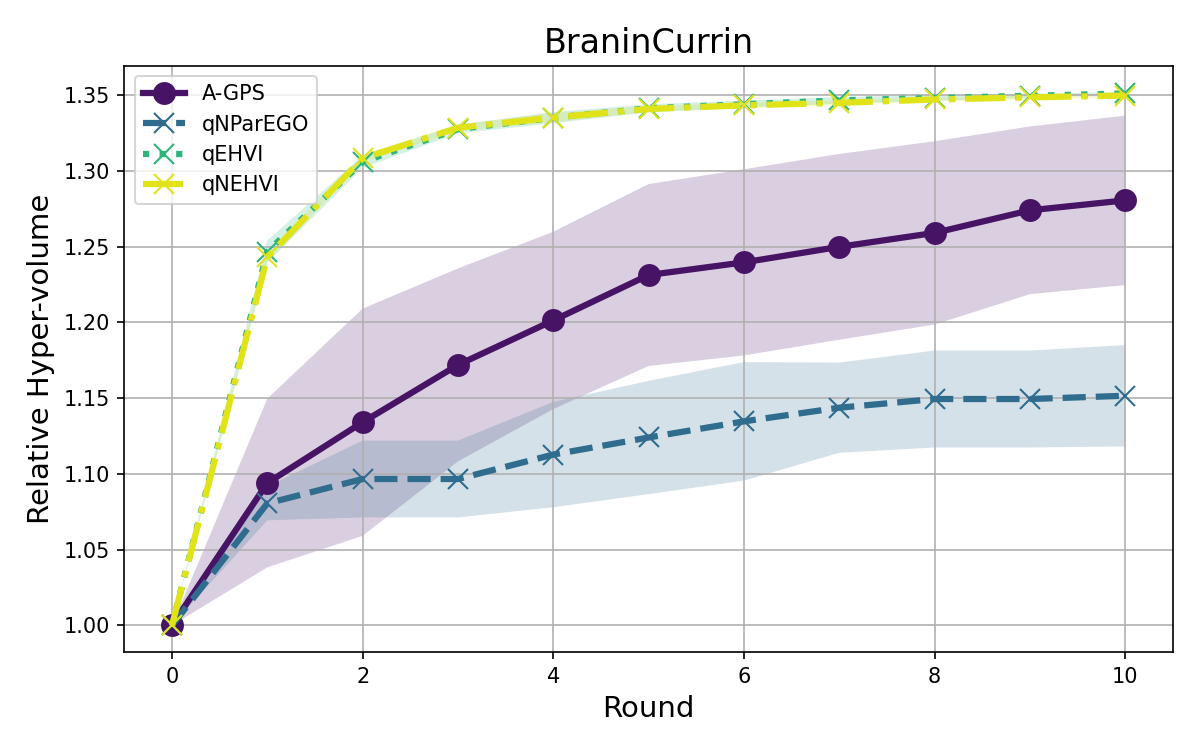}
    \includegraphics[width=0.329\linewidth]{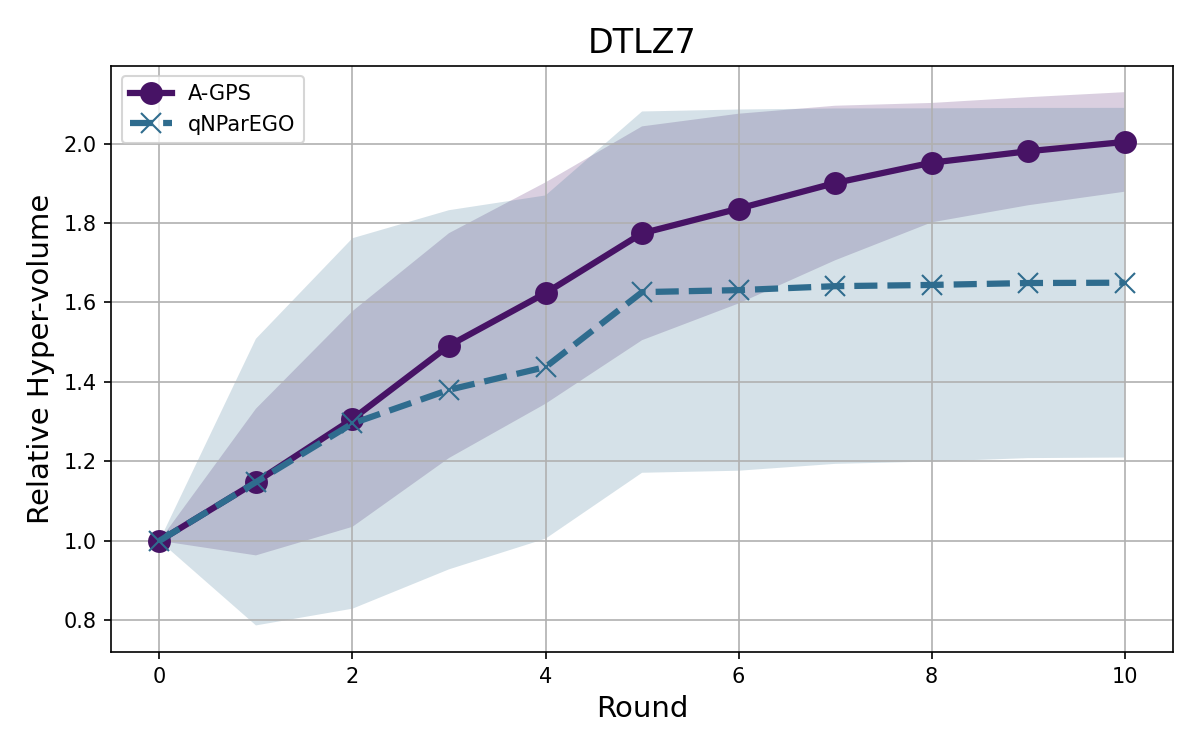}
    \includegraphics[width=0.329\linewidth]{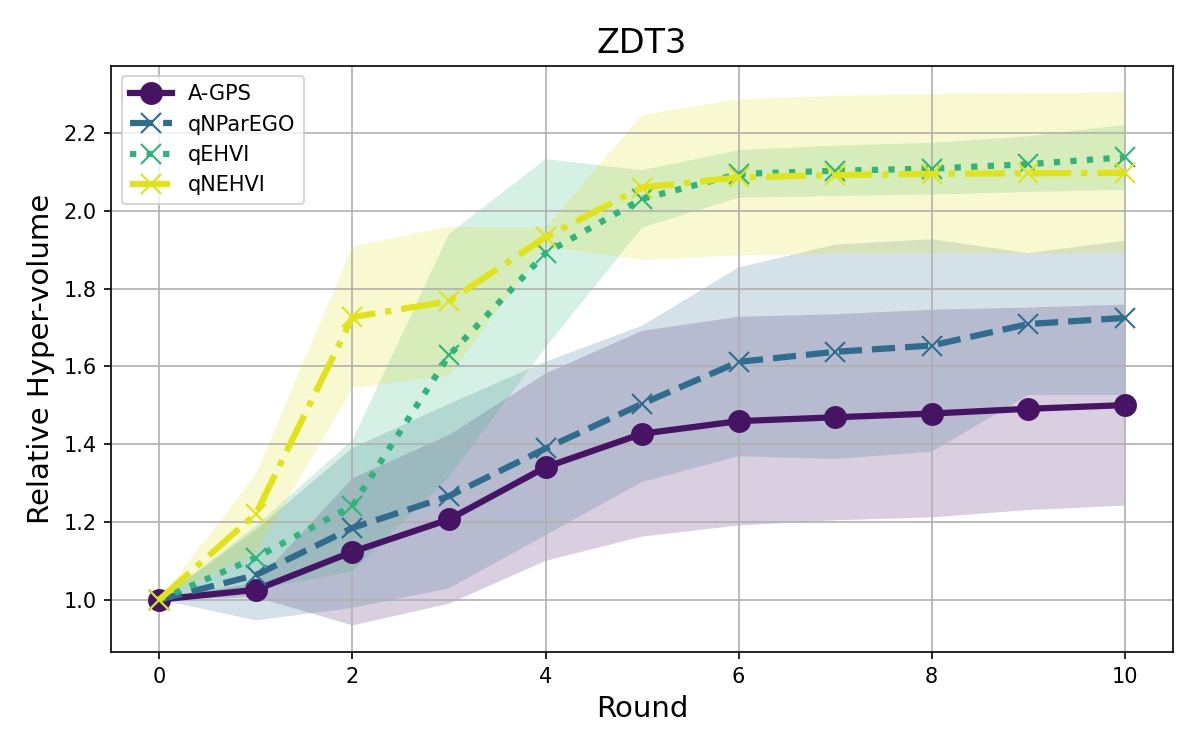} \\
    \subcaptionbox{Negative Branin-Currin \label{sfig:bc}}
    {\includegraphics[width=0.329\linewidth]{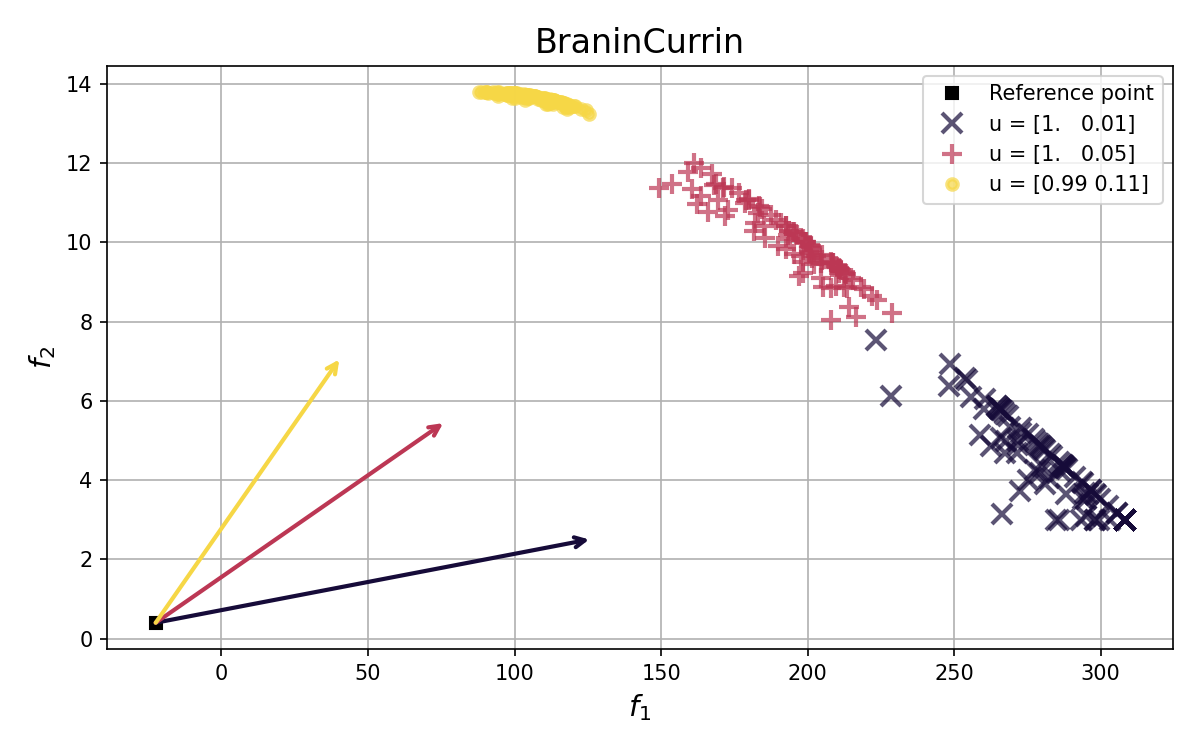}}
    \subcaptionbox{DTLZ7\label{sfig:dtlz7}}
    {\includegraphics[width=0.329\linewidth]{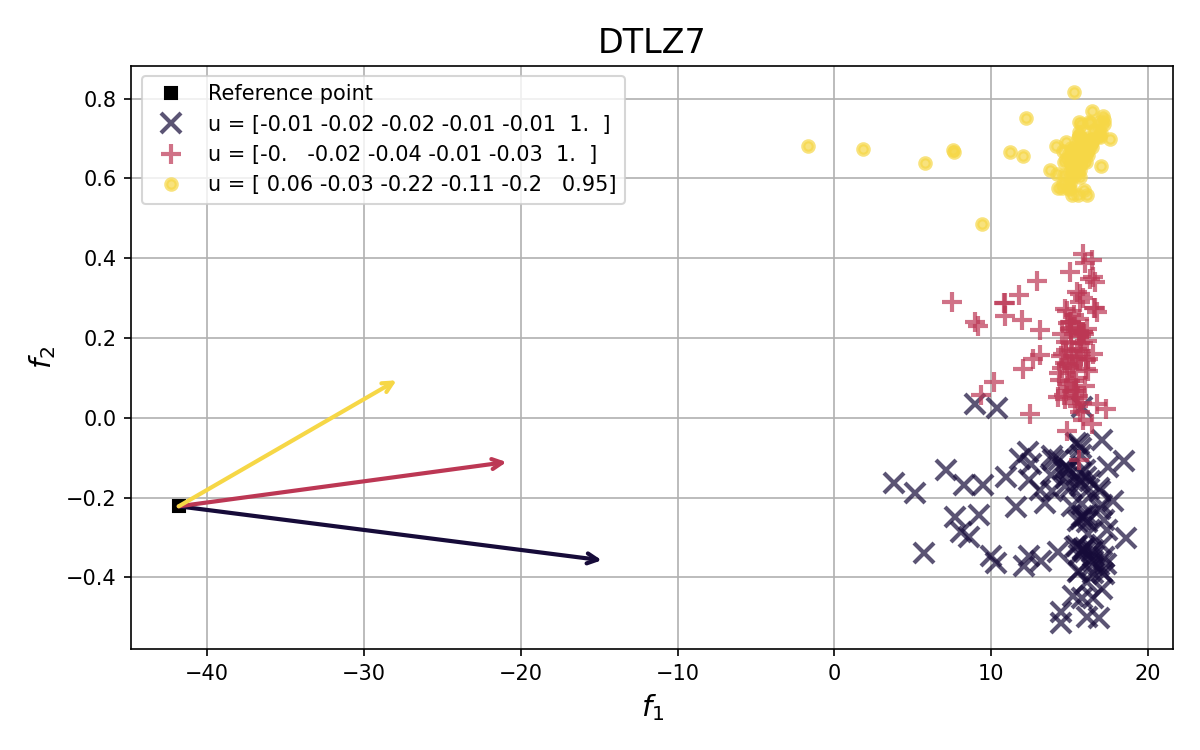}}
    \subcaptionbox{ZDT3\label{sfig:zdt3}}
    {\includegraphics[width=0.329\linewidth]{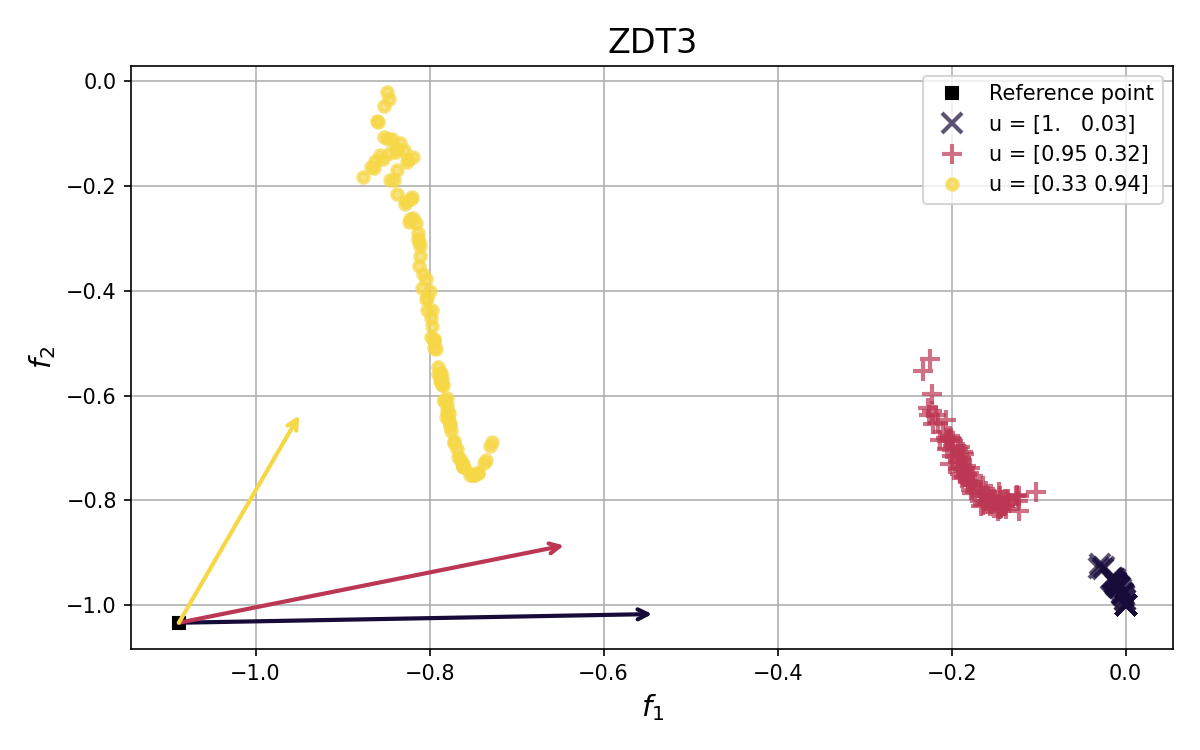}}
    \vspace{-0.3em}
    \caption{Experimental results on three test functions commonly used in the
        \gls{mobo} literature. The top row reports \gls{hvi} per round, the
        bottom row demonstrates amortized preference conditioning by
        generating Pareto front samples (DTLZ7 is a PCA projection of the
        front).%
    }
    \label{fig:syn_funcs}
    \vspace{-1em}
\end{figure}
The bottom row of  \autoref{fig:syn_funcs} illustrates preference conditioning:
each panel plots the sampled Pareto front (dots) from
$\qrobc{\qparam}{\obs}{\prefr_\star}$ colored by three representative preference
directions $\prefr_\star$. These $\prefr_\star$ were chosen by,
\begin{equation}
    \tars_\star \in \{
        [Q_{\fpareto^{t,1}}\!(0.9), Q_{\fpareto^{t,2}}\!(0.1)],
        [\text{Av}(\fpareto^{t,1}), \text{Av}(\fpareto^{t,2})],
        [Q_{\fpareto^{t,1}}\!(0.1), Q_{\fpareto^{t,2}}\!(0.9)]
    \},
\end{equation}%
where $Q$ is an empirical quantile function and $\text{Av}$ denotes the set
mean of each dimension, $l$, of the observed Pareto front, $\fpareto^{t,l}$. We
then use \autoref{eq:pref_train} to convert these into $\prefr_\star$ with an
automatically inferred reference point $\refp$. We project the higher
dimensional DTLZ7 outcomes into two dimensions using \gls{pca} for this
visualisation. Overall, these results show that \gls{agps} supports flexible,
a-posteriori preference conditioning across a variety of continuous landscapes.

\subsection{Ehrlich vs.\ naturalness}

We now evaluate \gls{agps} on a challenging two-objective synthetic `peptide' design task
that couples the Ehrlich synthetic landscape \cite{stanton2024closed} with a
ProtBert \cite{elnaggar2020prottrans} `naturalness' score. The Ehrlich function
has been designed to emulate key aspects of protein fitness; it maps each
discrete sequence to a scalar by embedding combinatorial motif interactions
in a highly rugged, multi-modal, but artificial terrain.
% Its epistatic peaks and valleys
% capture the statistical complexity of peptide design, yet it is entirely
% artificial, bearing no biochemical or evolutionary validity.
In stark contrast, ProtBert's loss reflects genuine amino-acid patterns learned
from 217 million real proteins. We use protein sequences $\obsspace =
\acidspace^M$ where $\card{\acidspace} = 20$ and $M \in \{15, 32, 64\}$. The
two objectives are,
% {\small
\begin{equation}
    \bbf^1(\obs) = \text{Ehrlich}(\obs), \qquad
    \bbf^2(\obs) = e^{-\mathcal{L}_\text{ProtBert}(\obs)}.
\end{equation}%
% }%
Here we convert ProtBert's log-loss, $\mathcal{L}_\text{ProtBert}(\obs)$ into a
likelihood ($\bbf^2 \in [0, 1]$), making it output in a comparable range to the
Ehrlich function ($\bbf^1 \in \{-1\} \cup [0, 1]$). We compare against a
baseline that randomly mutates the set of best candidates,
\gls{cbas} \cite{brookes2019conditioning} and \gls{vsd}
\cite{steinberg2025variational}, which use the same 2-layer CNN Pareto \gls{cpe} as
\gls{agps}, and against the guided diffusion based LaMBO-2
\cite{gruver2023protein}, which is formulated for discrete \gls{mobo} tasks
using \gls{ehvi} guidance. \Gls{agps} uses the same CNN architecture for its alignment \gls{cpe}.
\Gls{cbas} and \gls{vsd} use a causal transformer
architecture for their $\qrob{\qparam}{\obs}$. \gls{agps} uses the same
backbone transformer, but embeds $\prefr$ for its prefix token and uses FiLM
\cite{perez2018film} on the transformer embeddings for conditioning the
transformer on preferences, $\qrobc{\qparam}{\obs}{\prefr}$. The same
(unconditional) architectures are used as priors by \gls{cbas}, \gls{vsd} and
\gls{agps}, which are trained on the initial sequences using maximum
likelihood. Unlike in \cite{steinberg2025variational,brookes2019conditioning},
we allow \gls{cbas} to resample $\qrob{\qparam}{\obs}$ between rounds (once
every 100 iterations) --- this drastically improves its performance.
Results are reported in \autoref{fig:ehrlich_nat} for $T=40$ rounds from random
starting conditions (bands indicating $\pm 1$ std). All methods are given 128
training samples, and then recommend batches of size $B=32$ per round, and
$\thresh_0$ is set as $p=0.25$ percentile of Pareto ranks. We use the
\texttt{poli} and \texttt{poli-baselines} libraries for running the benchmarks
and LaMBO-2 baseline \cite{gonzalez-duque2024poli}. Additional experimental
details and timing results are in \autoref{app:ehrlich}. \Gls{agps} performs
similarly to \gls{vsd}, \gls{cbas} tends to be overly exploitative, and
LaMBO-2's guided masked-diffusion model tends to under-perform on this task.

\begin{figure}[tb]
    \centering
    \subcaptionbox{$M=15$\label{sfig:ehrlich-nat-15}}
    {\includegraphics[width=0.329\linewidth]{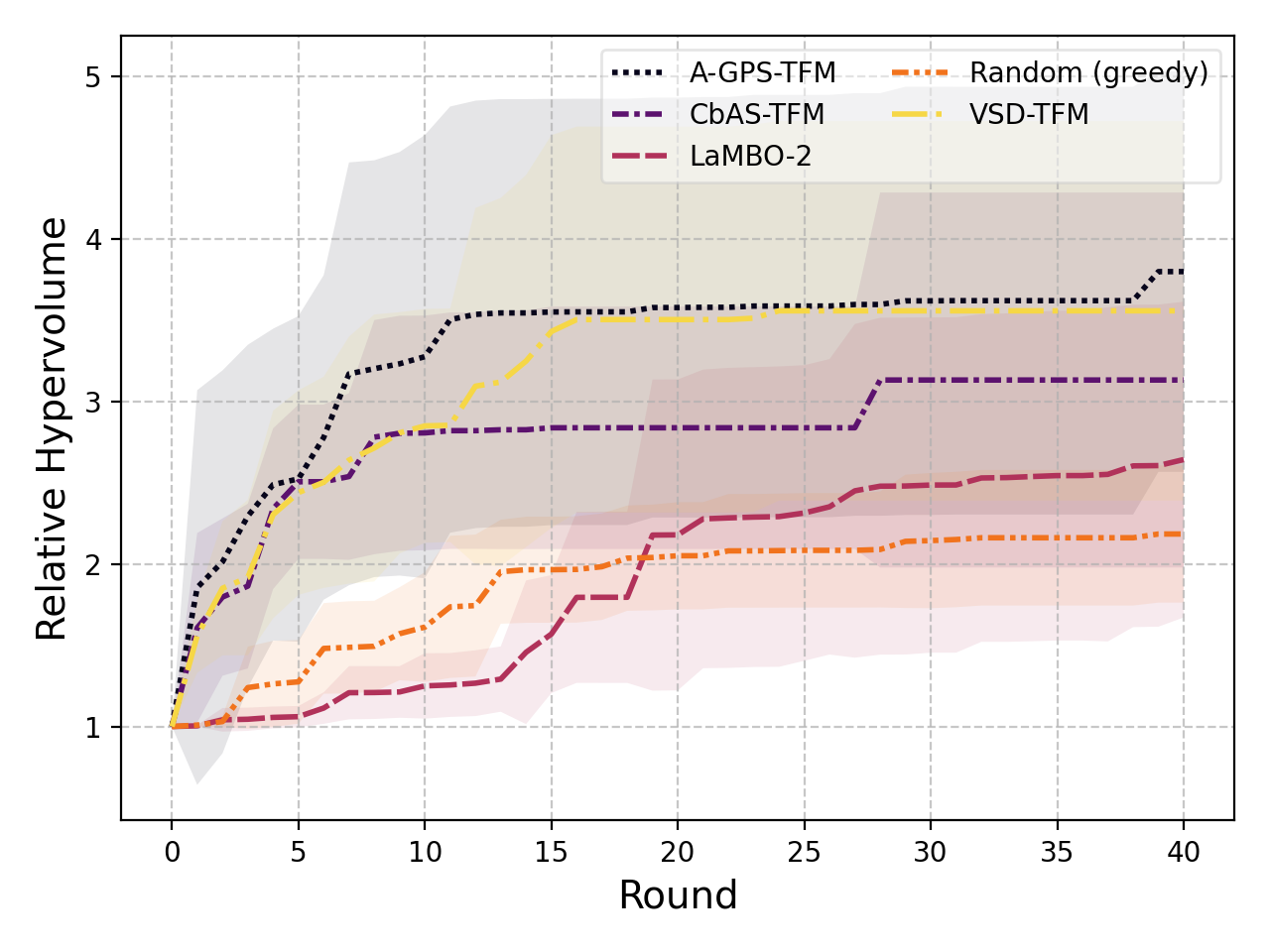}}
    \subcaptionbox{$M=32$\label{sfig:ehrlich-nat-32}}
    {\includegraphics[width=0.329\linewidth]{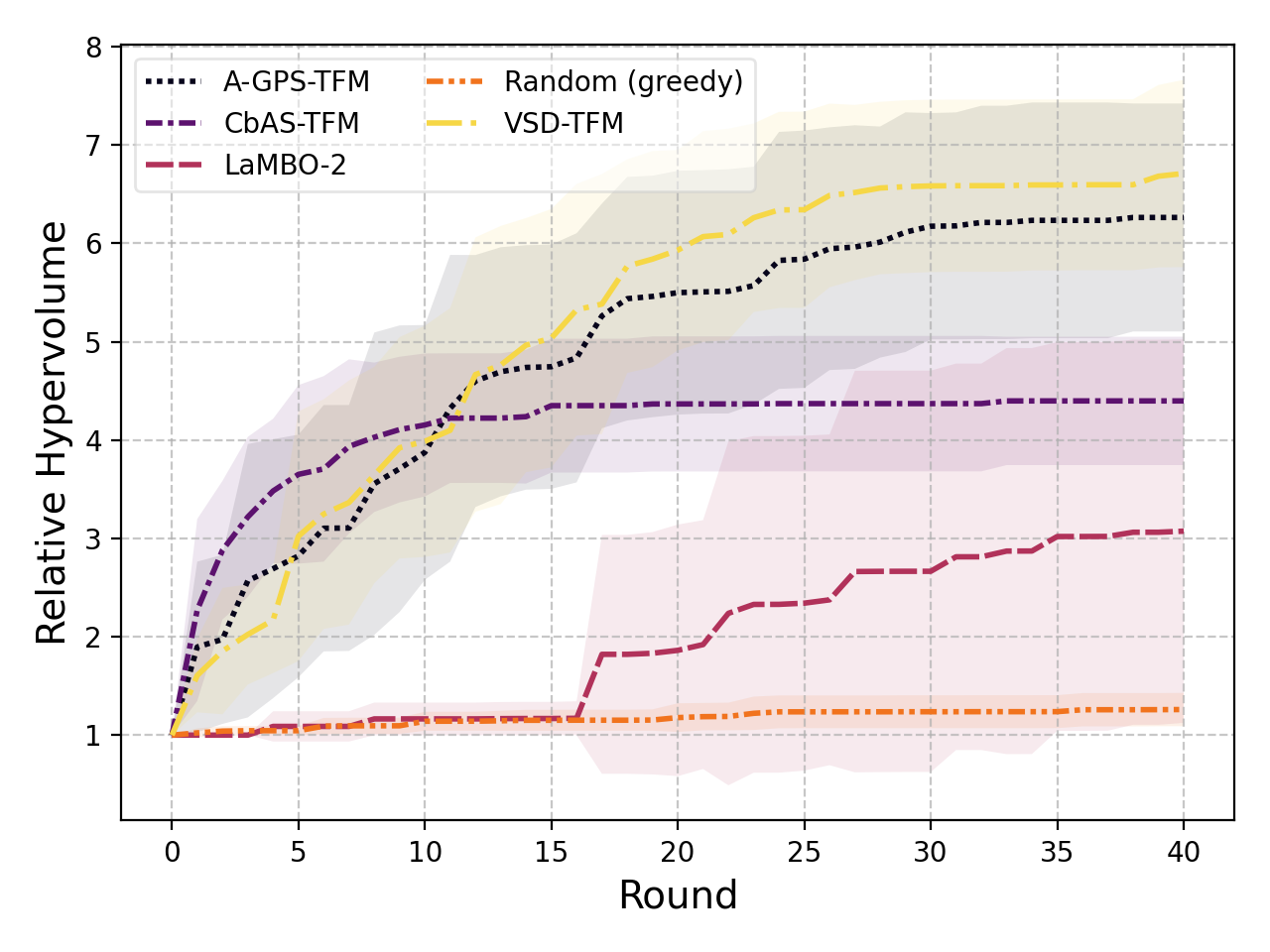}}
    \subcaptionbox{$M=64$\label{sfig:ehrlich-nat-64}}
    {\includegraphics[width=0.329\linewidth]{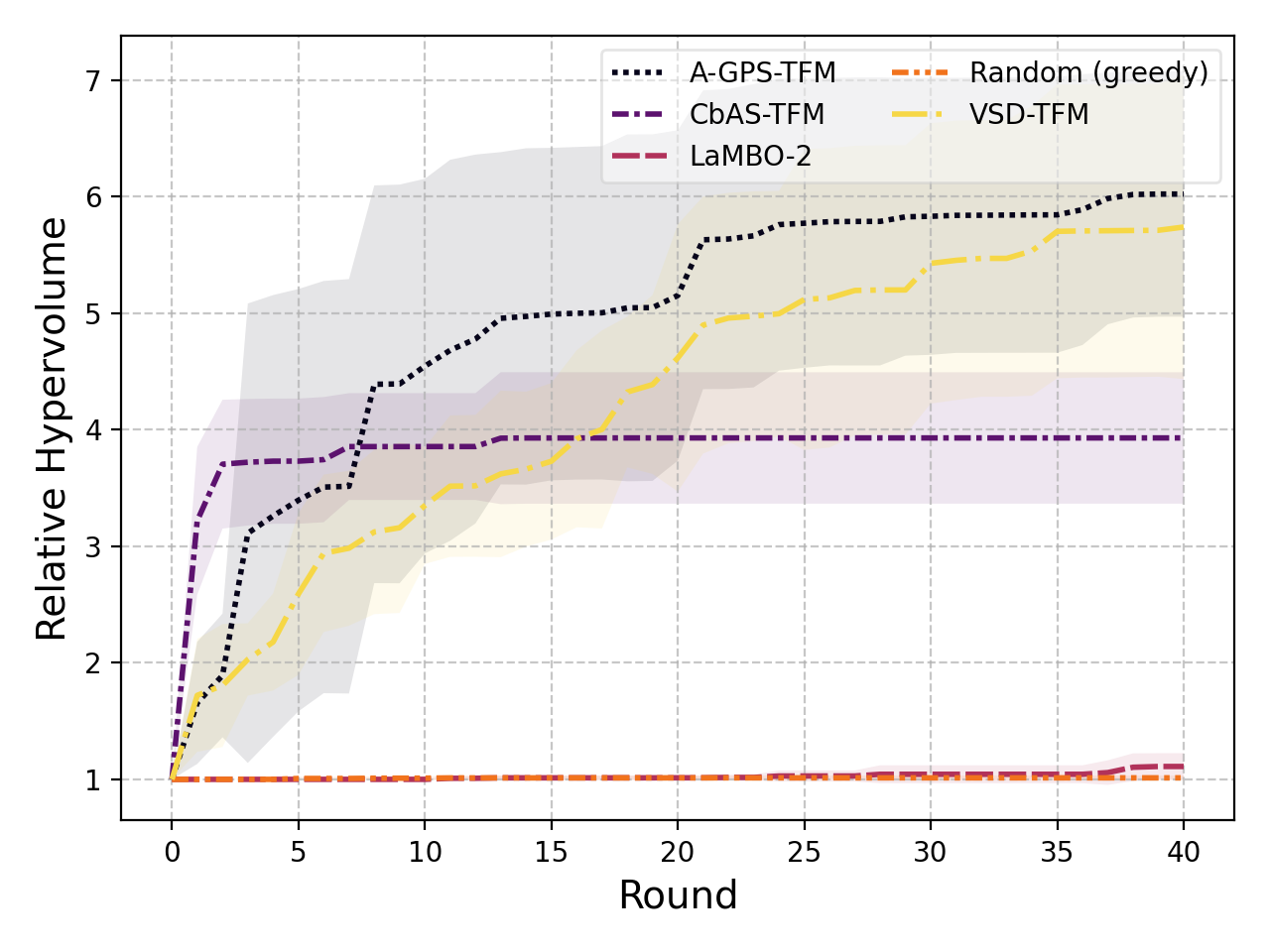}}
    \vspace{-0.3em}
    \caption{Ehrlich function vs.\ ProtBert naturalness score for different
    sequence lengths.}
    \label{fig:ehrlich_nat}
    \vspace{-1em}
\end{figure}

\subsection{Bi-grams}

For this experiment we use the bi-grams optimization task from
\cite{stanton2022accelerating} where the aim is to maximize the occurance of
three bi-grams (`AV', `VC' and `CA') in an $M=32$ length sequence. We start
with 512 random sequences that have no more than three of these bi-grams
present. We then have $T=64$ rounds of $B=16$ to optimize the bi-gram
occurances. The initial Pareto front is sparse so $\thresh_0$ is the $p=0.5$
percentile of Pareto ranks. The same models as the previous experiment are used
along with a masked-transformer model (mTFM) backbone for \gls{cbas}, \gls{vsd}
and \gls{agps}. This allows control over the number of mutations to apply to an
existing sequence, rather than generating a complete sequence from scratch ---
see \autoref{app:arch} for details. LaMBO-2's masked diffusion backbone also
has this ability, and following \cite{stanton2022accelerating} we use a
1-mutation budget for these models. Relative hypervolume improvement results
are presented in \autoref{fig:ngrams-foldx} (a) and (b), sequence diversity is
computed as the average pair-wise edit distance between all sequences. We can
see the causal transformer backbone \gls{agps} and \gls{vsd} models perform
best, followed by LaMBO-2. \gls{cbas} overfits early as we can see from the
diversity plot, and the mTFM models perform similarly to the random baseline and
take many rounds to show any improvement on this task.

\begin{figure}[tb]
    \centering
    \subcaptionbox{Bi-grams hypervolume\label{sfig:ngram-hv}}
    {\includegraphics[width=0.33\linewidth]{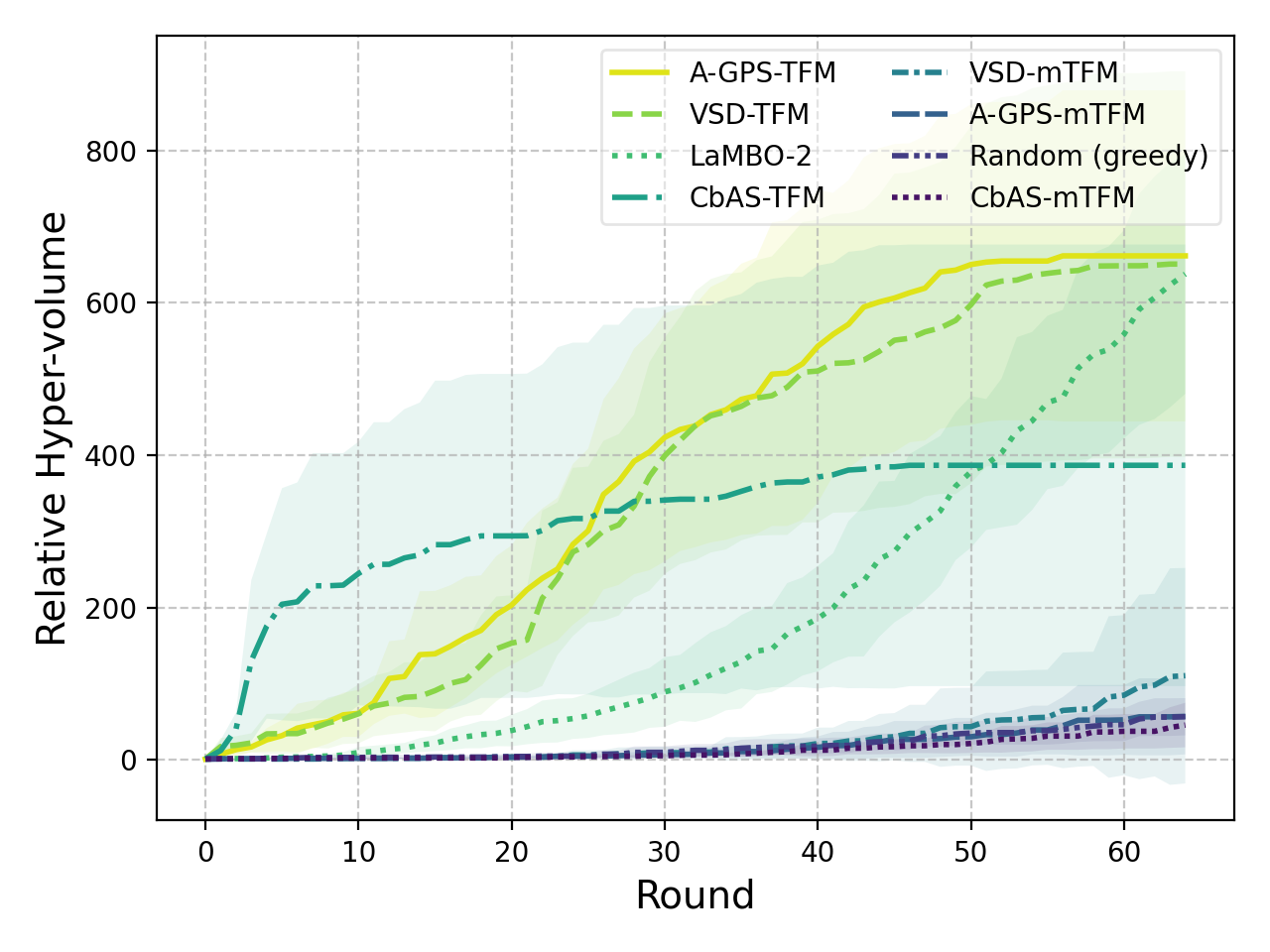}}
    \subcaptionbox{Bi-grams diversity\label{sfig:ngram-div}}
    {\includegraphics[width=0.33\linewidth]{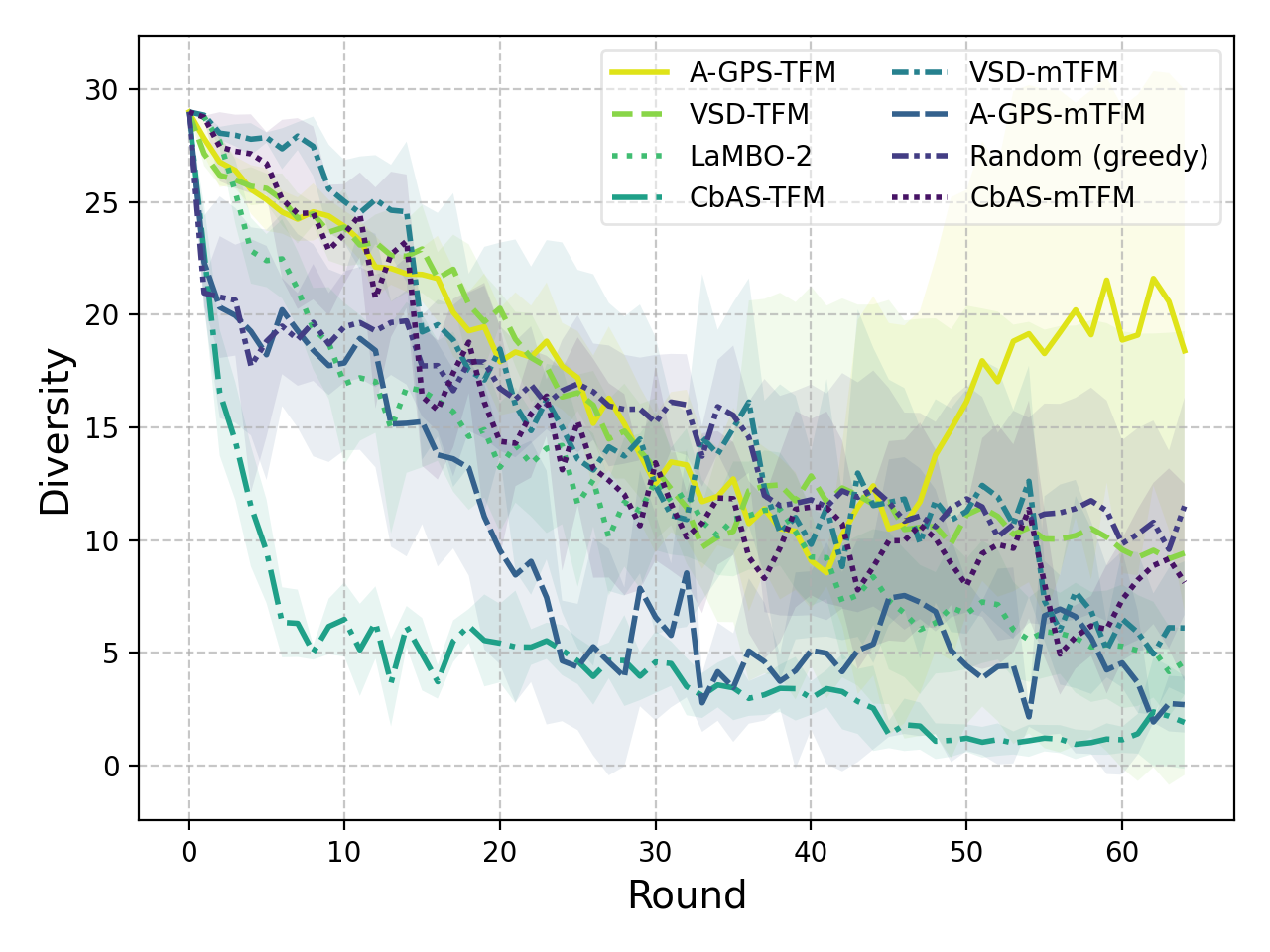}} \\
    \subcaptionbox{Stability vs.\ SASA hypervolume\label{sfig:foldx-hv}}
    {\includegraphics[width=0.33\linewidth]{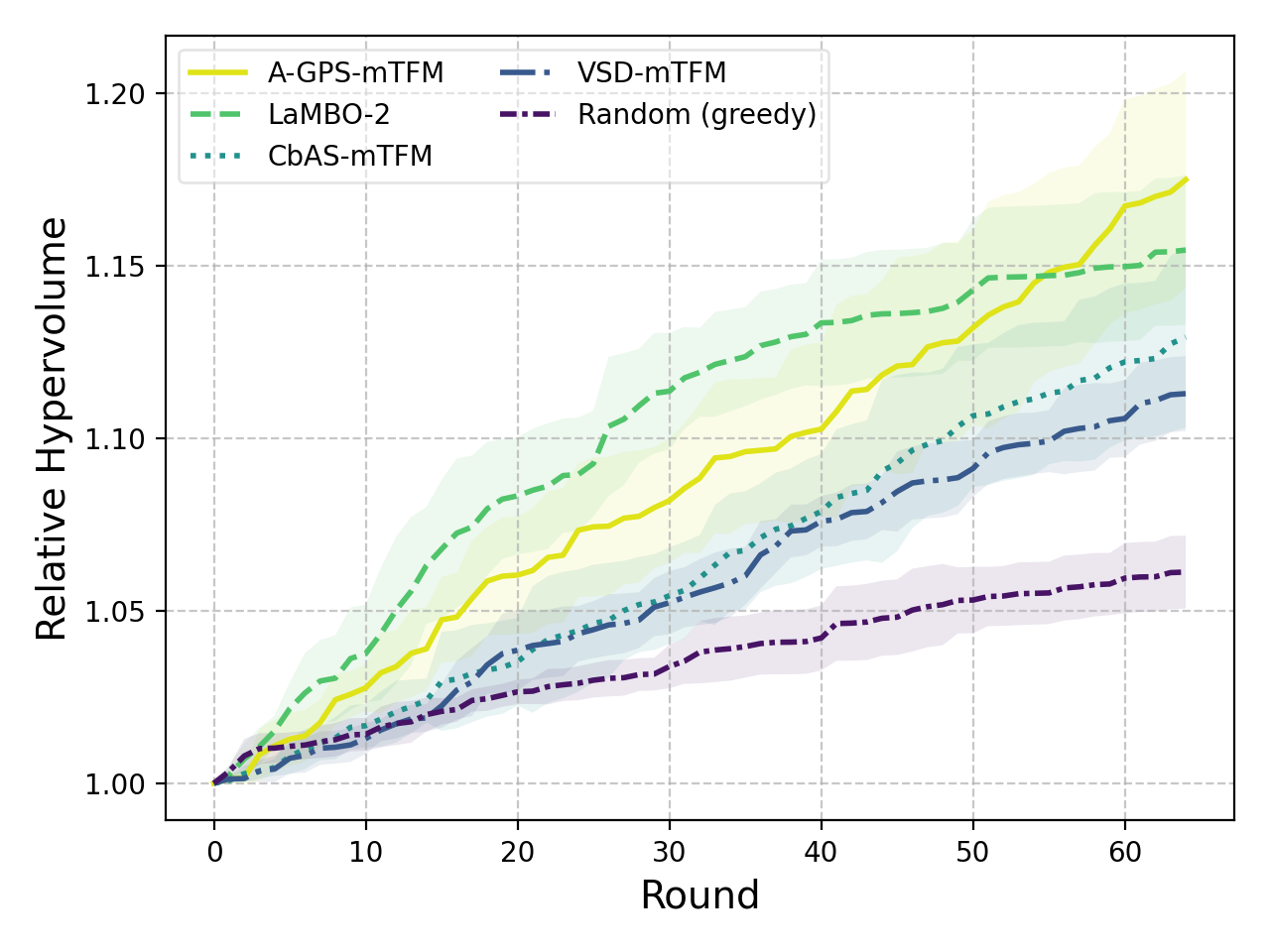}}
    \subcaptionbox{Stability vs.\ SASA diversity\label{sfig:foldx-div}}
    {\includegraphics[width=0.33\linewidth]{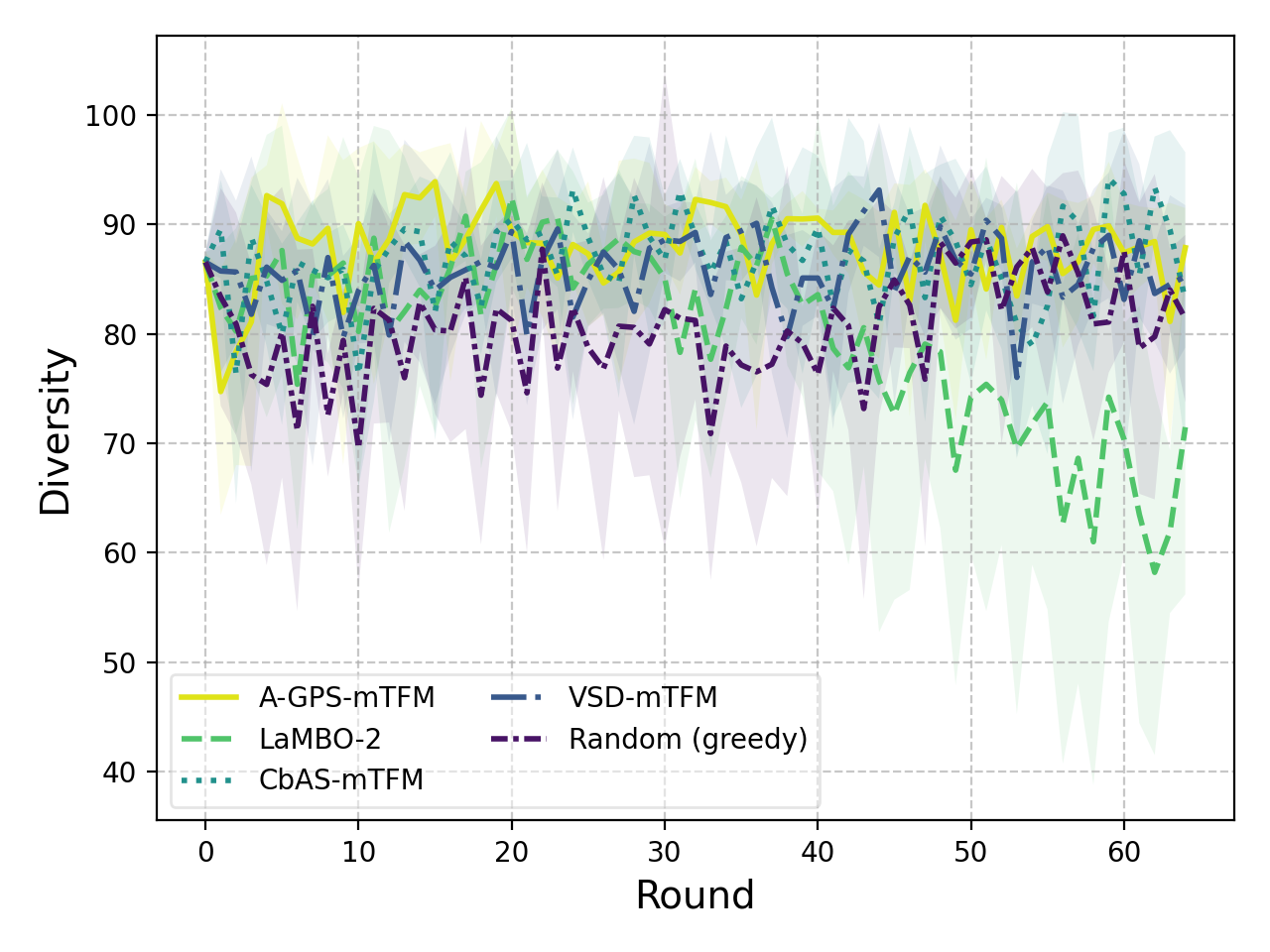}}
    \vspace{-0.3em}
    \caption{Bi-grams and Stability vs.\ SASA results for relative
    hypervolume (HV) improvement and diversity.}
    \label{fig:ngrams-foldx}
    \vspace{-1em}
\end{figure}

\subsection{Stability vs.\ SASA}

Our final experiment uses the simulation-based protein stability vs.\ solvent
accessible surface area (SASA) task from \cite{stanton2022accelerating}. The
aim is to optimize six base red fluorescent proteins with $M > 200$ for
stability ($-\Delta G$) and SASA. We use the FoldX black-box implementation in
\texttt{poli} \cite{gonzalez-duque2024poli}, with 512 training samples, $T=64$,
$B=16$, and a budget of one mutation per round. We have a rich starting Pareto
front so $\thresh_0$ is set from $p = 0.1$. We only use the mTFM backbone for
\gls{cbas}, \gls{vsd} and \gls{agps} as FoldX is best modelling only small
differences to the original sequences. The results are summarised in
\autoref{fig:ngrams-foldx} (c) and (d). LaMBO-2 initially performs well, but is
overcome by \gls{agps} as the sequence diversity diminishes.
See \autoref{app:foldx} for additional plots of the estimated Pareto front.

\section{Limitations and Discussion}
\label{sec:discussion}

We have considered the problem of active multi-objective generation, which
frames discrete black-box multi-objective optimization as an online sequential
generative learning task. Our proposed solution leverages recent advances in
generative models to estimate a distribution, $\qrobc{\qparam}{\obs}{\prefr}$,
of the Pareto set directly, conditioned on user preferences $\prefr$, which we
call \acrfull{agps}.

\paragraph{Limitations.} A limitation with \gls{agps}, and one that it shares
with many \gls{mobo} and \gls{mog} methods, is that it can be hard to specify
algorithm hyper-parameters a-priori --- before new data has been acquired ---
and the settings of these hyper-parameters can effect real-world performance.
We are mindful of this in our implementation and design of \gls{agps}, and as
such it comprises components that can be independently trained and validated
meaningfully on the initial training data at hand. In particular, we find
\gls{agps}, \gls{vsd} and \gls{cbas} are sensitive to the prior model used,
$\probc{\obs}{\data_0}$. To aid practitioners use these methods, we outline a
general procedure we find works well in \autoref{app:fit_prior} for prior
choice and initialization. Another limitation with \gls{agps} as presented in
this work is the choice of the generative model for the continuous synthetic
test functions. The MLP used does not appear to scale well to
higher-dimensional problems, and we expect using flow-matching
\cite{lipman2022flow} or diffusion \cite{ho2020denoising} backbones would
remedy this issue, and incorporating these models into the \gls{vsd} and
\gls{agps} frameworks is an active research direction.

Another avenue of future work would be to explore the rich literature on \glspl{moea}
for alternative multi-objective measures beyond scalarizations and hypervolume indicators that may be suitable for guided \gls{mog} \cite{afsar2023many}. Additionally, this literature gives a thorough treatment of stochastic experiments in the multi-objective setting \cite{rojasgonzalez2025bi} which may be useful for \gls{mog}.

In contrast to other approaches that are dependent on diffusion models, our
choice of generative model is flexible and modular. Our empirical experiments
demonstrate that our method performs well on high dimensional sequence design
tasks. We hope that our modular framework will result in many future extensions
to different architectures for generative models, resulting in further
practical algorithms for active generation in large search spaces.  For code
implementing \gls{agps}, \gls{vsd} and all of the experimental results, please see
\href{https://github.com/csiro-funml/variationalsearch}{\texttt{github.com/csiro-funml/variationalsearch}}.

\section*{Acknowledgements}
This work is funded by the CSIRO Science Digital and Advanced Engineering
Biology Future Science Platforms, and was supported by resources and expertise
provided by CSIRO IMT Scientific Computing. We would like to thank the
anonymous reviewers for the constructive feedback and advice, and Sebastian Rojas Gonzalez for pointing us to the literature on \glspl{moea} and for the discussion on estimating \gls{phvi} --- all of which greatly increased the quality of this work.

{%\small
\bibliography{vsd}
}

\newpage

\appendix

\section{Broader Impacts}
\label{app:impacts}

This work is motivated by applications that aim to improve societal
sustainability, for example, through the engineering of enzymes to help control
harmful waste. However, as with many technologies, it carries the risk of
misuse by malicious actors. We, the authors, explicitly disavow and do not
condone such uses.

\section{Additional Methodology Details}
\label{app:notes}

\begin{figure}[htb]
    \centering
    \includegraphics[width=\linewidth]{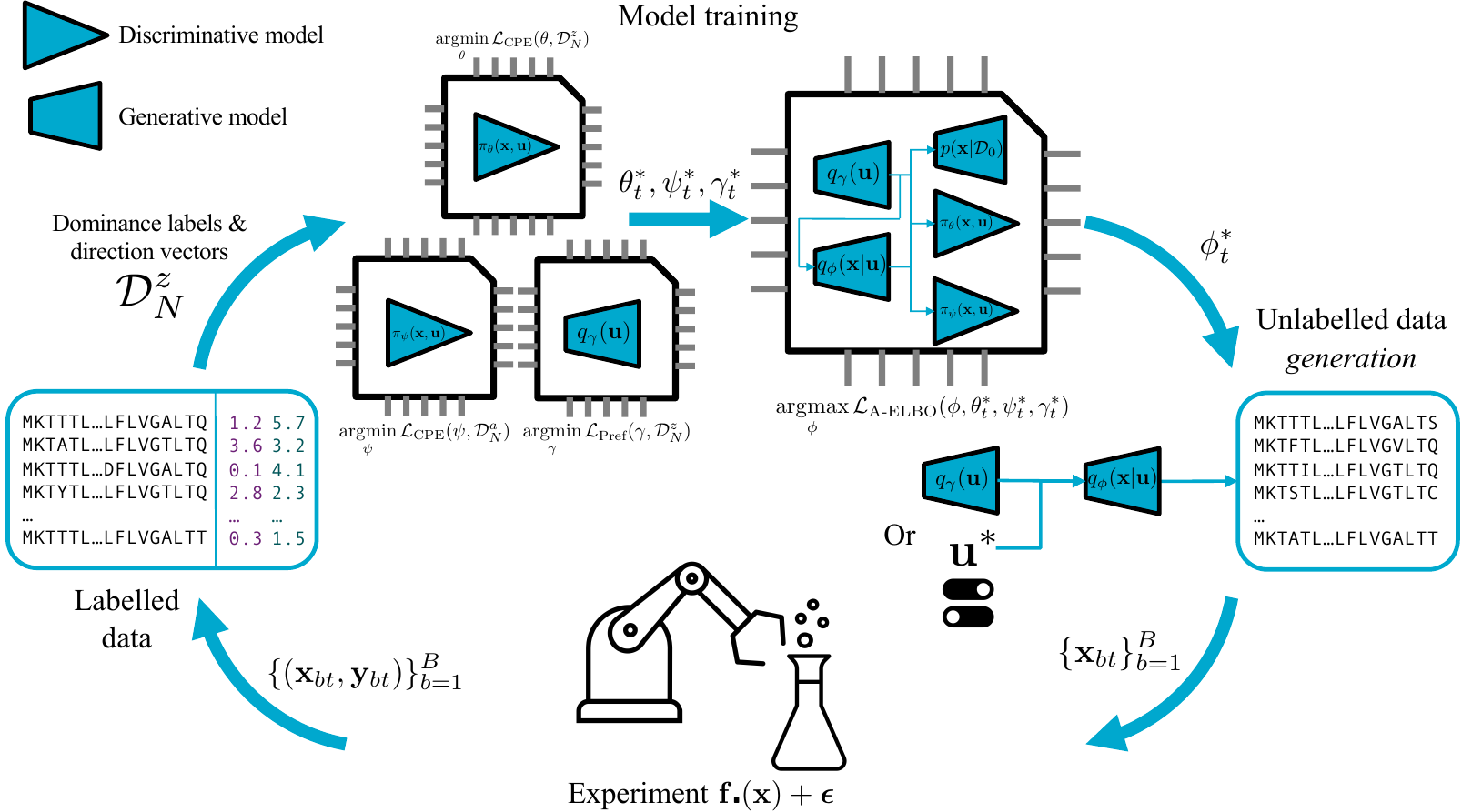}
    \caption{
        A visual depiction of \autoref{alg:optloop} --- \Gls{agps} learns all
        the distributions involved by optimizing different components of a
        reverse \gls{kl} loss. At time $t$, the optimized variational
        distribution $\qrobc{\qparam}{\obs}{\prefr}$ with parameters
        $\qparam^*_t$ is used to generate new designs that can incorporate new
        user's preferences $\prefr^*$. We iterate until a convergence/user
        criterion is satisfied.
    }
    \label{fig:agps-loop}
\end{figure}

\subsection{A Non-dominance CPE is estimating probability of hypervolume improvement}
\label{app:proof}

Using a \gls{cpe} trained on Pareto non-dominance labels, $\lablf$, is
equivalent to estimating the \gls{phvi} under some general conditions. The
proof is based on the current observed Pareto set and front, $\pareto^t$,
$\fpareto^t$, respectively, and relies on the box-decomposition definition of
hypervolume. We use $\tars$ as a shorthand for $\tars(\obs) = \bbfs(\obs) +
\errs$, and $\tars'$ for $\tars(\obs') = \bbfs(\obs') + \errs'$, etc., where the
context is clear. Firstly, our proof relies on the following assumptions,
\begin{assumption}[Measurement noise]\label{ass:noise}
    No objectives exhibit measurement noise, $\errs = \mathbf{0}$.
\end{assumption}%
This is so our dominance comparison, $\succ$, and hypervolume computation are exact with no ambiguity introduced from noisy measurements. We will discuss the consequences of relaxing this assumption later.

\begin{assumption}[Reference Point]\label{ass:ref}
    The reference point, $\refp \in \real^L$ is strictly dominated by every
    feasible objective vector,
    \begin{equation}
        \forall \tars \in \mathcal{Y} \subset \real^L, \qquad \tars \succ \refp.
    \end{equation}%
\end{assumption}%
This avoids negative-volume boxes when computing hypervolume, i.e.\ boxes must
have positive or no contribution to hypervolume.

% \begin{assumption}[No Ties]\label{ass:ties}
%     Given a new evaluation, $\tars$, for every pair $\tars', \tars'' \in
%     \left(\fpareto^t \cup \{\tars\}\right)$ % \times \left(\fpareto^t \cup \{\tars\}\right)$,
%     such that $\tars' \neq \tars''$,
%     we have $\tar_l' \neq \tar''_l$ for all $l \in \{1, \ldots, L\}$.
% \end{assumption}%
% In other words, no two distinct points in $\fpareto^t \cup \{\tars\}$ can share a
% coordinate, otherwise we have zero-width slabs with zero hypervolume when
% defining hypervolume with boxes. Unfortunately discrete objectives will violate
% this assumption with high probability. $\fpareto^t \cup \{\tars\}$ naturally
% also implies a corresponding set $\pareto^t \cup \{\obs\}$.

We define a box of hypervolume for any $\tars \in \mathcal{Y}$ under
\autoref{ass:noise} and \autoref{ass:ref} as,
\begin{equation}
    \hvbox{\tars} := [\refp, \tars] = [r_1, \tar_1] \times [r_2, \tar_2] \times \dots \times [r_L, \tar_L].
    \label{eq:box}
\end{equation}
Then let $\lambda_L$ denote an $L$-dimensional Lebesgue measure, so the
(dominated) hypervolume of a finite set, $\mathcal{A} \subset \real^L$ is,
\begin{equation}
    \hv{\mathcal{A}} := \lambda_L \Big(\!\bigcup_{\tars \in \mathcal{A}} \hvbox{\tars} \Big).
    \label{eq:hv}
\end{equation}
With this we can define \acrfull{hvi},
\begin{equation}
    \hvi{\obs} := \hv{\fpareto^t \cup \{\tars(\obs)\}} - \hv{\fpareto^t}.
    \label{eq:hvi}
\end{equation}
In addition, for any $\solnspace \subset \obsspace$, let $\paretofn{\solnspace}$ denote the Pareto subset of $\solnspace$, i.e.,
\begin{equation}
    \paretofn{\solnspace} := \{\obs \in \solnspace: \obs' \not\succ \obs, \forall \obs' \in \solnspace \backslash \obs \}\,.
\end{equation}
Note that $\paretofn{\solnspace \cup \{\obs\}} = \paretofn{\solnspace}$ for any $\obs \in \obsspace$ that is dominated by an element of $\solnspace$.
With these simple definitions and assumptions, and noting that:
\begin{equation}
    \lablf(\obs) = \indic{\obs \in \paretofn{\pareto^t \cup \{\obs\}}}\,, %TODO: It might be better to work with some \obsspace_t, denoting the set of observation points up to time t, though both definitions are equivalent.
\end{equation}
we have the following result.

\mainthm*%
\begin{proof}
    This is straightforward to see if we consider the dominated and
    non-dominated cases.

    \textbf{Case 1:} $\obs$ is dominated. Having $\lablf(\obs) = 0$ implies that the condition $\obs' \not\succ \obs, \forall \obs'\in\pareto^t$ is not satisfied and the Pareto set remains unchanged, i.e., $\paretofn{\pareto^t \cup \{\obs\}} = \pareto^t$, so that $\hv{\fpareto^t \cup \{\tars\}} = \hv{\fpareto^t}$. Hence, there is no hypervolume improvement,
    % $\tars' \succ \tars$ for some $\tars' \in \fpareto^t$, then
    % $\hvbox{\tars} \subseteq \hvbox{\tars'}$. Thus, under \autoref{ass:ref},
    % adding $\tars$ does not enlarge the union of boxes and so
    % $\hvi{\obs} = 0$, then,
    \begin{equation}
        \lablf(\obs) = 0 \quad \Rightarrow \quad \indic{\hvi{\obs} > 0} = 0.
    \end{equation}

    \textbf{Case 2:} $\obs$ is non-dominated. If
    $\obs' \not\succ \obs$ for all $\obs' \in \pareto^t$, there is \emph{no} $\tars' \in \fpareto^t$ such that $\hvbox{\tars} \subset \hvbox{\tars'}$.
    % then there is at least one $\tar'_l > \tar_l$ for $l \in \{1, \ldots, L\}$, so
    For $\hvi{\obs}$ to be positive, we need to show that $\hv{\fpareto^t} < \hv{\fpareto^t \cup \{\tars\}}$, i.e., whatever remains of the difference between $\hvbox{\tars}$ and the previous union $\bigcup_{\tars' \in \fpareto^t}\!\hvbox{\tars'}$ must be a set of positive Lebesgue measure. Indeed, letting $\gamma := \min_{\obs' \in \pareto^t} \max_{i\in\{1,\dots, L\}} \tar_i - \tar'_i$, which is positive due to dominance, and setting $\delta := \min\{\gamma/2, \min_{i \in \{1,\dots, L\}} (\tar_i - r_i)/2\} > 0$ (by \autoref{ass:ref}), we have that the $\delta$-box $\mathrm B_\delta := \prod_{i=1}^L [\tar_i - \delta, \tar_i] \subset \hvbox{\tars}$ is not covered by the union $\bigcup_{\tars' \in \fpareto^t}\!\hvbox{\tars'}$. It then follows that $\hvi{\obs} = \hv{\fpareto^t \cup \{\tars\}} - \hv{\fpareto^t} \geq \lambda_L(\mathrm B_\delta) > 0$, confirming that
    % By \autoref{ass:ties}, the uncovered part of $\hvbox{\tars'}$ has strictly
    % positive measure, so $\hvi{\obs'} > 0$. Therefore,
    \begin{equation}
        \lablf(\obs) = 1 \quad \Rightarrow \quad \indic{\hvi{\obs} > 0} = 1.
    \end{equation}
    Thus, for any choice $\obs \notin \pareto^t$ we have $\lablf(\obs) = \indic{\hvi{\obs'} > 0}$.
\end{proof}
\phvicor*

It is worth noting that we do not require \autoref{ass:noise} for these indicators to be equivalent so long as the same noisy samples, $\tars_i$, are being used when computing them. However, we can no longer claim that the indicators themselves are reliably computing hypervolume improvement or non-dominance since the addition of noise makes the underlying  comparisons ambiguous with respect to the true objectives, $\bbfs$.
It may still be possible to show convergence of A-GPS to the true Pareto set, $\pareto$, while using noisy observations. The authors of \gls{vsd} \cite{steinberg2025variational} were able to guarantee convergence using a \gls{cpe} trained with noisy \gls{pi} labels for single objective \gls{bbo} using results from the \gls{ntk} literature (see their Appendix F). Some of these results may be extended to Pareto set classification, though we leave this for future work, as all of our experiments are noise-free.

\subsection{Fitting the prior}
\label{app:fit_prior}

We found that \gls{agps}, \gls{vsd} and \gls{cbas}, which share generative
backbones, were all very sensitive to the choice of prior. For complex
problems, the best results were obtained in general when the prior is chosen to
be of the same form (or an unconditional variant) as the variational
distribution, and fit to the $T=0$ training data. However, for flexible
backbone models like transformers, it can be easy to overfit to this training
data. We found that adding dropout in conjunction with an early stopping
training procedure on the initial training dataset reliably led to good
performance for all methods. The exact procedure used is listed as follows,
\begin{enumerate}
    \item Set prior dropout $p$
    \item Train prior with a 10\% validation set, make note of number of iterations
        when validation loss begins to increase
    \item Train prior with all data for the number of iterations noted in the
        previous step
    \item If appropriate, copy weights to variational distribution (without
        dropout).
\end{enumerate}
When the prior and variational models were compatible, we initialized the
variational model with these learned prior weights. We run an ablation of this
procedure in \autoref{app:dropout}.

\section{Architectural Details}
\label{app:arch}

\subsection{Preference direction distributions}
\label{app:pref_dist}

In all the experiments we use a mixture of isotropic Normal distributions where
the samples have been constrained to the unit norm,
\begin{equation}
    \qrob{\pparam}{\prefr} =
        \sum^K_{k=1} w_k
        \mathcal{N}_{\|\prefr\|}(\prefr | \boldsymbol{\mu}_k, \boldsymbol{\sigma}^2_k),
    \label{eq:mixpref}
\end{equation}
and $\pparam = \{(\boldsymbol{\mu}_k, \boldsymbol{\sigma}_k)\}^K_{k=1}$.
Typically, we find $K=5$ is sufficient. We learn this via maximum likelihood as
per \autoref{eq:ml_pref}, but we add an extra regularisation term: $-
\frac{1}{K}\sum^K_{k=1}(\|\boldsymbol{\mu}_k\| - 1)^2$ so the magnitude of the
mixture means is controlled (and does not decrease to 0 or increase to
$\pm\infty$). We have compared this to von Mises distributions, and find it
more numerically stable, we also find no tangible benefit using more complex
spherical normalizing flow representations \cite{rezende2020normalizing}.
Furthermore, we find that the performance is similar to, if not slightly
superior to, the empirical approximation,
\begin{equation}
    \qrob{\pparam}{\prefr} =
    % \card{\pareto^t}^{-1}
    \frac{1}{\sum_{n=1}^N \lablf_n}
    \sum_{n=1}^N
    \lablf_n \, \indic{\prefr = \prefr_n},
    \label{eq:emppref}
\end{equation}
where $\pparam = \{\prefr_n : \lablf_n = 1\}^N_{n=1}$.
Though on occasions when only a few observations define the Pareto front, we
find that using this representation can lead to an overly exploitative strategy.

\subsection{Sequence variational distributions}
\label{app:var_dist}

In this section we summarize the main variational distribution architectures
considered for \gls{agps} \gls{vsd} and \gls{cbas}.

\paragraph{Causal Transformer.}
For some of the sequence experiments we implemented an auto-regressive (causal)
transformer of the form,
\begin{align}
    \qrob{\qparam}{\obs} &= \categc{\acid_1}{\softmax{\qparam_1}}
        \prod^M_{m=2} \qrobc{\qparam_{d}}{\acid_m}{\acid_{1:m-1}}, \quad \textrm{where}
        \nonumber \\
    \qrobc{\qparam_{1:m}}{\acid_m}{\acid_{1:m-1}} &=
        \categc{\acid_m}{\softmax{\mathrm{DTransformer}_{\qparam_d}\!(\acid_{1:m-1})}}.
        \label{eq:autoregressive-proposal}
\end{align}
For details on the decoder-transformer with a causal mask see
\cite[Algorithm 10 \& Algorithm 14]{phuong2022formal} for maximum likelihood
training and sampling implementation details respectively.

\paragraph{Masked Transformer.}
We also implemented a masked transformer model (mTFM) that learns to mutate an
initial sequence, $\obs'$, at a set of positions, $\mask = [\maskp_1, \ldots,
\maskp_M]$ where $\maskp_m \in \{0, 1\}$ and $\sum_{m=1}^M \maskp_m = O$ for a
mutation budget, $O \in \{1, \ldots, M\}$. The complete generative model is,
\begin{align}
    \qrobc{\qparam}{\obs, \mask}{\obs'} &=
        \qrobc{\qparam_{e}}{\obs}{\obs'[\mask \gets \acid_\text{mask}]}
        \qrobc{\qparam_{oe}}{\mask}{\obs'}.
\end{align}
Where the notation $\obs'[\mask \gets \acid_\text{mask}]$ means we apply a
masking token to the original sequence at the positions indicated by $\mask$.
The mask generation model is,
\begin{align}
    \qrobc{\qparam_{oe}}{\mask}{\obs'} &=
        \multinc{\mask}{\mathrm{NN}_{\qparam_o}\!(\mathrm{ETransformer}_{\qparam_e}\!(\obs'))},
\end{align}
Here $\mathrm{ETransformer}$ is an encoder-transformer, see \cite[Algorithm
9]{phuong2022formal} for details, and $\mathrm{NN}_{\qparam_o}$ is a \gls{nn}
decoder, with a convolutional residual layer for capturing additional local
structure. The token generation model is,
\begin{align}
    \qrobc{\qparam_{xe}}{\obs}{\obs'[\mask \gets \acid_\text{mask}]} &=
        \prod_{m=1}^M
        \begin{cases}
            \indic{\acid_m=\acid'_m} & \text{if} ~ \maskp_m = 0; \\
            \categc{\acid_m}{\mathbf{p}_m} & \text{if} ~ \maskp_m = 1,
        \end{cases} \quad \text{where} \nonumber \\
        [\mathbf{p}_1, \ldots, \mathbf{p}_M] &=
        \msoftmax{
            \mathrm{ETransformer}_{\qparam_e}\!(\obs'[\mask \gets \acid_\text{mask}])
        }. \nonumber
\end{align}
Here $\mathbf{softmax}: \real^{\card{\acidspace}\times M} \to
[0, 1]^{\card{\acidspace} \times M}$.
The same encoder-transformer is shared between the mask and token heads, and
only the masked positions are allowed to sample new tokens. For our experiments
we bias the categorical distribution so that the original token is not
resampled. The choice of the set of seed sequences, $\{\obs'_i\}^J_{j=1}$, can
drastically impact performance of \gls{agps}, \gls{vsd} and \gls{cbas}. A
simple heuristic that we find works well in practice is to uniformly sample
with replacement from the current active Pareto set, $\pareto^t$, so $\obs' \sim
\uniform{\pareto^t}$. This is similar to the strategy used in LaMBO-2, though
since they use \gls{ehvi}, they can maintain only the top-$B$ sequences per
round ranked by \gls{ehvi}, where $B$ is batch size.

Finally, if we use this model as a prior, we drop the conditional mask model,
$\qrobc{\qparam_{oe}}{\mask}{\obs'}$, and draw mask-positions independently
with a Bernoulli distribution, $\bernc{\maskp_m}{p_\text{mask}=0.15}$. Then we
use \cite[Algorithm 12]{phuong2022formal} for training.

We list the configurations of the transformer variational distributions in
\autoref{tab:qsettings}. We use additive positional encoding for all of these
models. When using these models for priors or initialization of variational
distributions, we find that over-fitting can be an issue. To circumvent this,
we use dropout and early stopping, see \autoref{app:fit_prior} for details.

\paragraph{Conditioning.}
As mentioned in the text, for the conditional generative models,
$\qrobc{\qparam}{\obs}{\prefr}$, for \gls{agps}, we use the same
architectures already discussed, but also learn a sequence prefix embedding from
$\prefr$, as well as a simple 1-hidden layer MLPs for implementing
FilM~\cite{perez2018film} adaptation of the sequence token embeddings,
\begin{equation}
    \mathbf{e}_m = \mathbf{e}'_{m} \circ (1 + f_\alpha(\prefr)) + f_\beta(\prefr),
    \quad \text{where} \quad
    \mathbf{e}_0 = f_\text{prefix}(\prefr).
\end{equation}
Where $\mathbf{e}'_m$ are the unmodified token embeddings (i.e.\ outputs of $\mathrm{DTransformer}$ 
or $\mathrm{ETransformer}$), and $\circ$ indicates element-wise product. $f_\alpha$ and $f_\beta$ are
the FiLM MLPs, and $f_\text{prefix}$ is the prefix embedding MLP input into the transformers (often we find just a linear projection is adequate). 
We initialize the transformer weights in these models from their
non-conditional counterparts when they are used as priors.

\begin{table}[htb]
    \small
    \centering
    \caption{Transformer network configuration for the sequence experiments.}
    \begin{tabular}{r| c c c | c | c}
        & \multicolumn{3}{c|}{\textbf{Ehrlich vs.\ Nat}} &
        \textbf{Bi-grams} & \textbf{Stability vs.\ SASA} \\
        $\downarrow$ \textbf{Property} / $M \rightarrow$ & 15 & 32 & 64 & 32 & > 200 (variable)\\
        \hline
        Layers & 2 & 2 & 2 & 2 & 2 \\
        Network Size & 128 & 128 & 128 & 128 & 256\\
        Attention heads & 4 & 4 & 4 & 4 & 4\\
        Embedding size & 64 & 64 & 64 & 64 & 64\\
        FiLM hidden size & 128 & 128 & 128 & 128 & 128 \\
        Prior dropout $p$ & 0.5 & 0.4 & 0.2 & 0.2 & 0.1 \\
        Mutation budget $O$ & -- & -- & -- & 1 & 1 \\
    \end{tabular}
    \label{tab:qsettings}
\end{table}

\subsection{Class probability estimator architectures}
\label{app:cpe}

For all of our experiments we share the same architecture for both
$\cpe{\lablf}{\mparam}{\obs, \prefr}$ and $\cpe{\lablu}{\uparam}{\obs, \prefr}$. On the
continuous synthetic test functions we use the MLP in \autoref{fig:cpe-arch}
(a), where we simply concatenate the inputs $\obs$ and $\prefr$. Here \texttt{Skip} is a skip connection which implements a residual layer, and \texttt{MaxAndAvg} means
a scalar weighted sum of max and average pooling.

For the sequence experiments we use the convolutional architecture given in
\autoref{fig:cpe-arch} (b). For \gls{vsd} and \gls{cbas} we simply add on
another \texttt{LayerNorm} and \texttt{LeakyReLU} and then an output linear
layer. For \gls{agps} we concatenate $\prefr$ to output of this CNN, and then
pass this concatenation into the MLP in \autoref{fig:cpe-arch} (c).
All architectural properties are listed in \autoref{tab:nnsettings}.

\begin{table}[htb]
    \caption{\Gls{cpe} configurations for \gls{agps}, \gls{vsd} and \gls{cbas}.}
    \small
    \centering
    \begin{tabular}{r|c c c | c | c | c}
        & \multicolumn{3}{c|}{\textbf{Ehrlich vs.\ Nat}} &
        \textbf{Bi-grams} & \textbf{Stability vs.\ SASA} &
        \textbf{Synthetic Fns.} \\
        $\downarrow$ \textbf{Property} / $M \rightarrow$ & 15 & 32 & 64 & 32
        & > 200 (variable) & -- \\
        \hline
        \texttt{E} & 16 & 16 & 16 & 16 & 16 & --\\
        \texttt{C} & 64 & 64 & 64 & 64 & 96 & --\\
        \texttt{Kc} & 5 & 5 & 5 & 5 & 7 & --\\
        \texttt{Kx} & 3 & 3 & 3 & 3 & 5 & --\\
        \texttt{Sx} & 2 & 2 & 2 & 2 & 4 & -- \\
        \texttt{H} & 128 & 128 & 128 & 128 & 192 & $\min\{16 D, 128\}$ \\
        \texttt{hidden\_layers} & -- & -- & -- & -- & -- & 2
    \end{tabular}
    \label{tab:nnsettings}
\end{table}

\begin{figure}[htb]
    \centering
\begin{minipage}{.45\textwidth}%
\begin{lstlisting}
Sequential(
    Linear(
        in_features=D + L,
        out_features=H
    ),
    LayerNorm(),
    LeakyReLU(),
    Dropout(p=0.1),
    *[Skip(
        Linear(
            in_features=H,
            out_features=H
        ),
        LayerNorm(),
        LeakyReLU(),
        Dropout(p=0.1),
    ) for _ in range(hidden_layers)],
    Linear(
        in_features=H,
        out_features=1
    ),
)
\end{lstlisting}
\centering
(a) Continuous MLP architecture
\begin{lstlisting}
Sequential(
    Linear(
        in_features=H + L,
        out_features=H
    ),
    LeakyReLU(),
    LayerNorm(),
    Skip(
        Linear(
            in_features=H,
            out_features=H
        ),
        LeakyReLU(),
    ),
    Linear(
        in_features=H,
        out_features=1
    ),
\end{lstlisting}
\centering
(c) Sequence-preference concatenation MLP architecture
\end{minipage}%
\hfill
\begin{minipage}{.45\textwidth}%
\begin{lstlisting}
Sequential(
    Embedding_And_Positional(
        num_embeddings=A,
        embedding_dim=E
    ),
    Dropout(p=0.2),
    Conv1d(
        in_channels=E,
        out_channels=C
        kernel_size=Kc
    ),
    GroupNorm(),
    LeakyReLU(),
    Dropout(p=0.1),
    MaxAndAvgPool1d(
        kernel_size=Kx,
        stride=Sx,
    ),
    Skip(
        Conv1d(
            in_channels=C,
            out_channels=C,
            kernel_size=Kc,
        ),
        GroupNorm(),
        LeakyReLU(),
    )
    AdaptiveMaxAndAvgPool1d(),
    Linear(
        out_features=H
    ),
)
\end{lstlisting}
\centering
(b)  Sequence CNN architecture
\vspace{3.35cm}
\end{minipage}%
    \caption{\Gls{cpe} architectures used for the experiments in PyTorch-like
    syntax. $\texttt{A} = \card{\acidspace}$, $\texttt{L} = L$ corresponding to
    $\tars \in \real^L$ and $\texttt{D} = D$ corresponding to $\obs \in
    \real^D$ for the continuous experiments. LaMBO-2 uses the same kernel size
    as our CNNs. See \autoref{tab:nnsettings} for specific property settings.}
    \label{fig:cpe-arch}
\end{figure}

\section{Experimental Details}
\label{app:experiments}

\subsection{Synthetic test functions}
\label{app:syn_fn}

We tested \gls{agps} on a number of popular \gls{moo} synthetic test functions
against strong \gls{gp}-based baselines. All the test functions are summarised
as follows, with additional results presented in \autoref{fig:syn_funcs_ext}.

\textbf{Branin-Currin} ($D=2$, $L=2$): We optimize the negative Branin-Currin
convex pair. We found the negative function has a more interesting
Pareto front while remaining a challenging \gls{mobo} task.

\textbf{DTLZ7} ($D=7$, $L=6$): A higher-dimensional constraint surface, made of
$L$ segments. The outcome dimension is too
high for methods that directly estimate the improvement to hypervolume.

\textbf{ZDT3} ($D=4$, $L=2$): A complex, non-convex front comprised
of several disconnected segments, which stresses an optimizer's capacity for
both exploration and front-segment coverage.

\textbf{DTLZ2 ($D=3, L=2$)}: A smooth, spherical front in the negative orthant,
which tests an algorithm's ability to approximate non-convex curved manifolds
in higher dimensions.

\textbf{DTLZ2 ($D=5, L=4$)}: A higher dimensional instantiation of the DTLZ2
function for testing performance with higher dimensional inputs and objectives.

\textbf{GMM ($D=2, L=2$)}: Here each objective is implemented as a Gaussian
mixture model, and so is highly multimodal \cite{daulton2022robust}. This is run
without observation noise.

\begin{figure}[htb]
    \centering
    \includegraphics[width=0.329\linewidth]{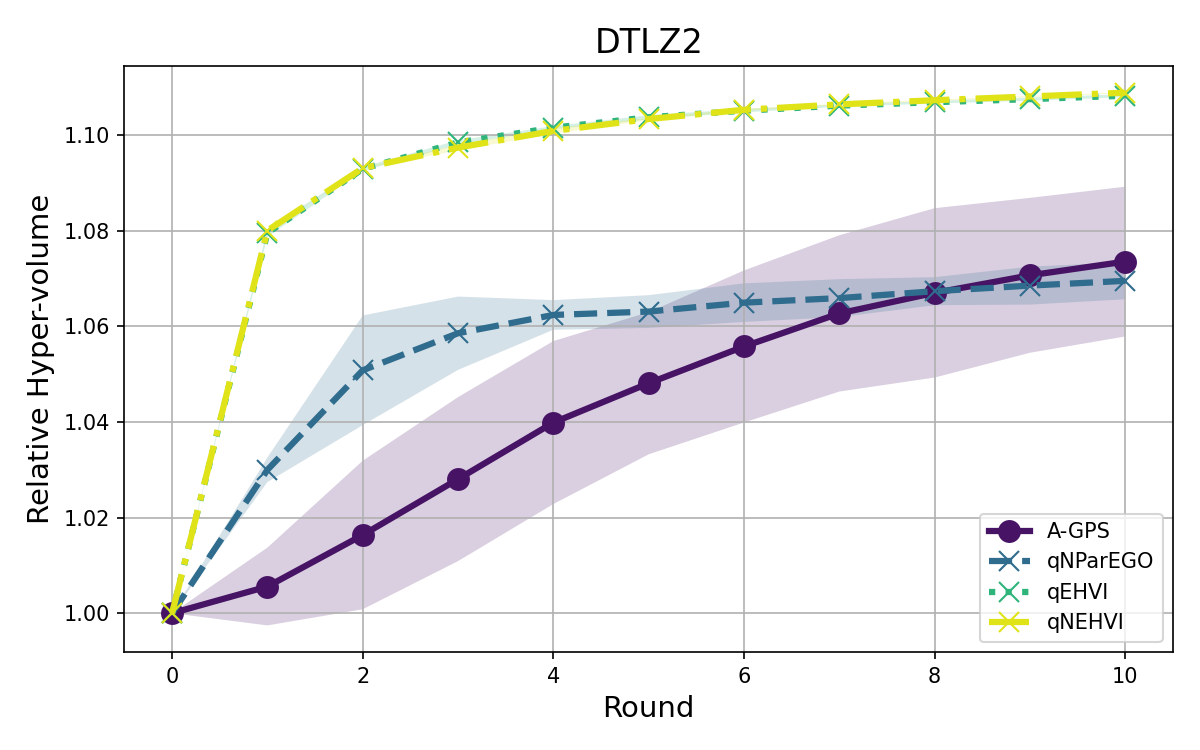}
    \includegraphics[width=0.329\linewidth]{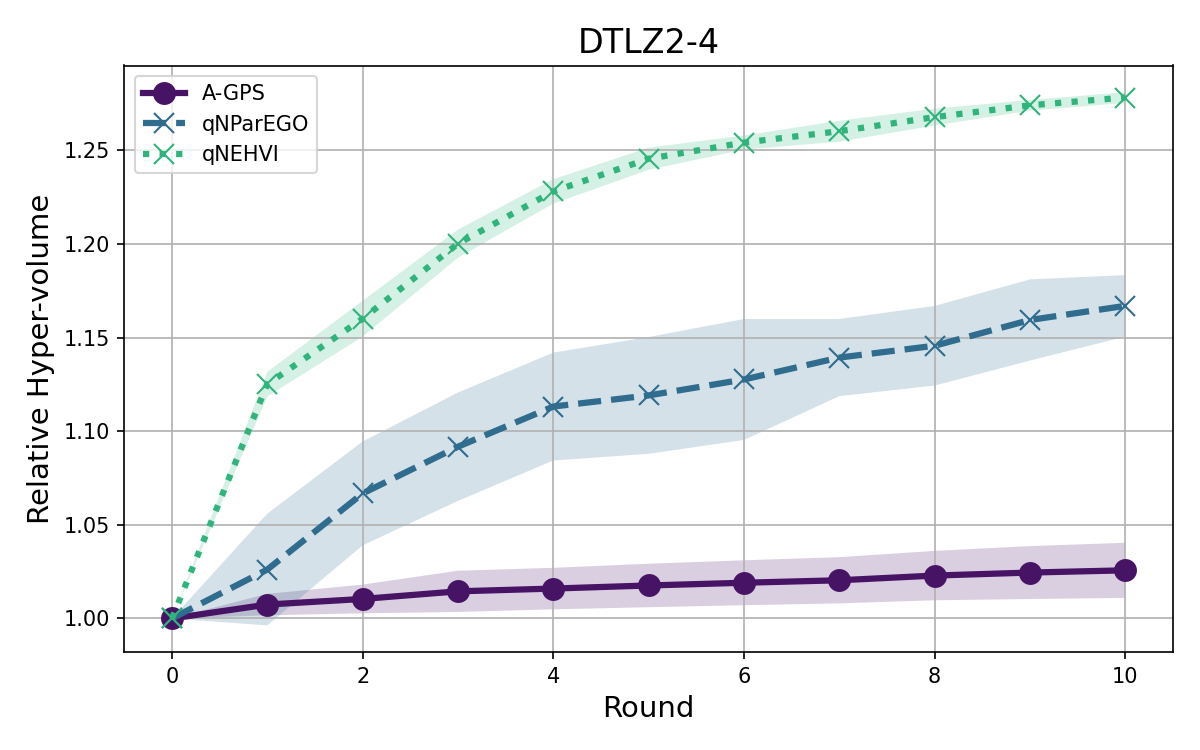}
    \includegraphics[width=0.329\linewidth]{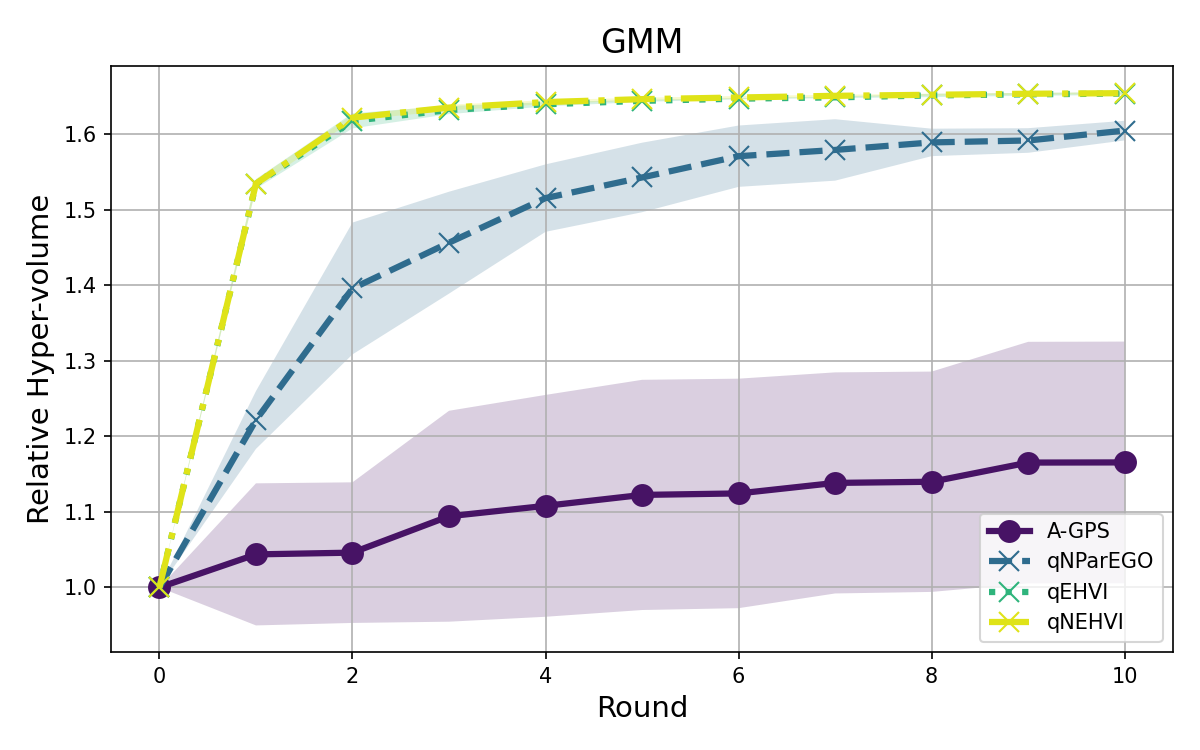} \\
    \subcaptionbox{DTLZ2 $L=2$ \label{sfig:dtlz2}}
    {\includegraphics[width=0.329\linewidth]{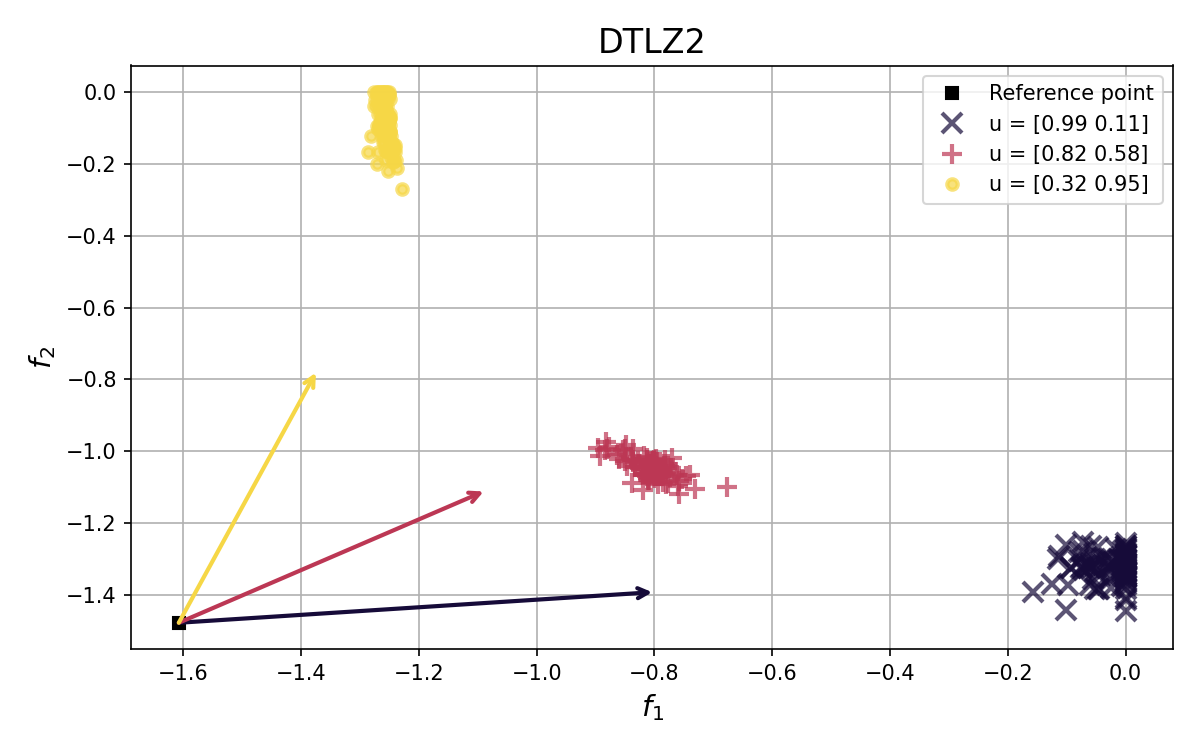}}
    \subcaptionbox{DTLZ2 $L=4$ \label{sfig:dtlz2-4}}
    {\includegraphics[width=0.329\linewidth]{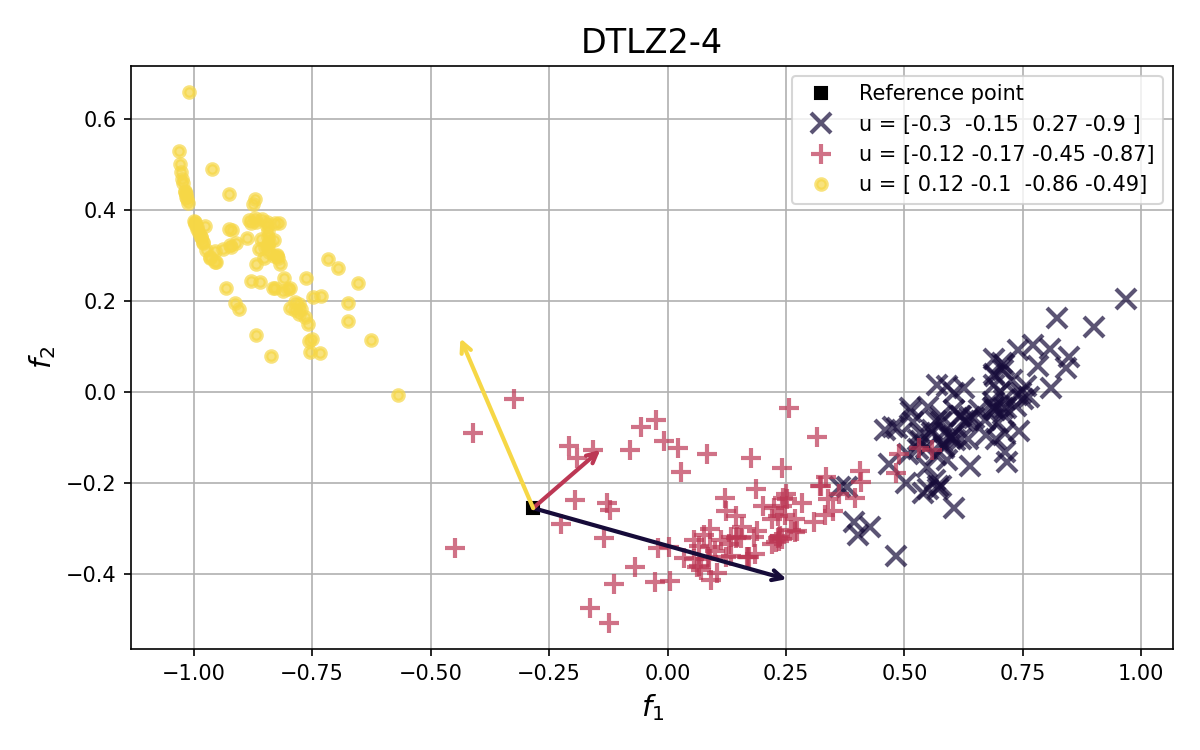}}
    \subcaptionbox{GMM\label{sfig:gmm}}
    {\includegraphics[width=0.329\linewidth]{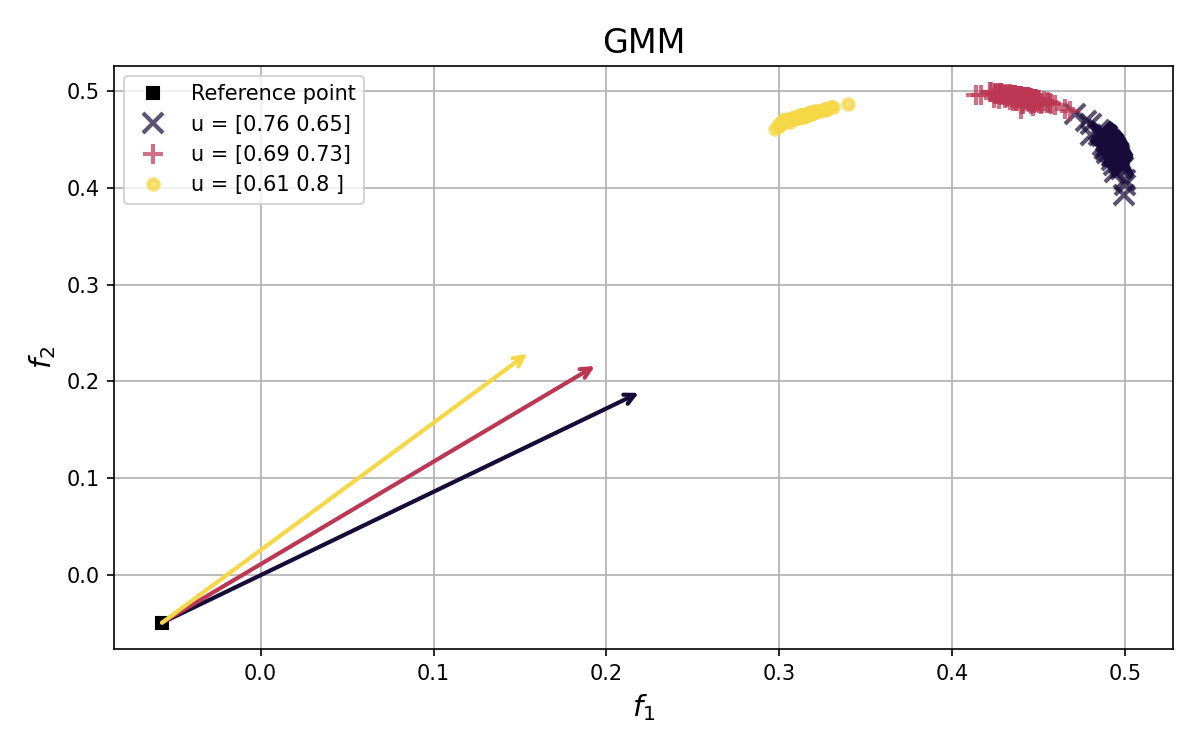}}
    \vspace{-0.3em}
    \caption{
        Experimental results on an additional three test functions commonly
        used in the \gls{mobo} literature. The top row reports \gls{hvi} per
        round, the bottom row demonstrates amortized preference conditioning by
        generating Pareto front samples (DTLZ2 with $L=4$ is a PCA projection
        of the front).%
    }
    \label{fig:syn_funcs_ext}
    \vspace{-1em}
\end{figure}

A detailed description of these functions and/or their Pareto-front geometries
can be found in \cite{deb2002scalable,zhang2007moea, belakaria2019max}.
We use \texttt{BoTorch} \cite{balandat2020botorch} for the implementations of
all the synthetic test functions and the baseline \gls{mobo} methods. All
experiments were initialized using Latin hyper-cube sampling. The original
design space is $\obsspace = [0, 1]^D$ for all problems, and we optimize
directly in this space with all methods. The \gls{gp} based methods use
an optimizer that respects these bounds directly, and we clamp \gls{agps}'s
generative model to these bounds. We found transforming the space (e.g. using
logit-sigmoid transforms) generally led to worse performance.

\begin{table}[htb]
    \caption{Synthetic test functions experimental settings.}
    \footnotesize
    \centering
    \small
    \begin{tabular}{r|c}
    \textbf{Setting} & \textbf{Value} \\
    \hline
    $N_{t=0}$ & 64 \\
    $T$ & 10 \\
    Replicates & 10 \\
    $B$ & 5 \\
    $S$ & 256 \\
    $\refp$ & inferred using BoTorch's \texttt{infer\_reference\_point()} \\
    Base BoTorch model & \texttt{SingleTaskGP} \\
    Optimizer &	Adam for \gls{agps} (lr = $10^{-5}$) and L-BFGS for the GPs \\
    Max optimization iter.\ for \gls{agps} & 3000 \\
    GP restarts & 10 \\
    Bounds & $[0, 1]^D$ \\
    GP kernel & Matern $\nu = 2.5$ ARD (standardized inputs) \\
    GP hyperparameters & LogNormal priors on $\sigma$ and $l$
    \end{tabular}
    \label{tab:synfnsettings}
\end{table}

For \gls{agps} we use the mixture model in \autoref{eq:mixpref} for the
preference direction distribution, and for the conditional generative model we
use a simple MLP,
\begin{equation}
    \qrobc{\qparam}{\obs}{\prefr} =
    \normalc{\obs}{\boldsymbol{\mu}(\prefr), \boldsymbol{\sigma}^2(\prefr)}.
\end{equation}
Here $\boldsymbol{\mu}(\prefr), \boldsymbol{\sigma}^2(\prefr)$ are MLPs with 2
or 4 hidden layers (if $D \geq 3$) of size of $\min(16 D, 256)$ with
skip-connections and layer normalization, making them residual networks. We
otherwise use the same experimental settings for the rest of the experiments,
which are given in \autoref{tab:synfnsettings}.

\subsection{Ehrlich vs.\ Naturalness}
\label{app:ehrlich}

The Ehrlich vs.\ naturalness score benchmark was implemented using the
\texttt{poli} benchmarking library \cite{gonzalez-duque2024poli}, where we
implemented our own ProtGPT2-based naturalness black box, and used the inbuilt
Ehrlich function \cite{stanton2024closed} (not the holo version). For
\gls{agps} we use the aforementioned generative and discriminative models,
otherwise the settings are given in \autoref{tab:ehrlichraspsettings}. We use a
modified version of the LaMBO-2 algorithm \cite{gruver2023protein} from
\texttt{poli-baselines} \cite{gonzalez-duque2024poli}.
We use the following Ehrlich function configurations:
\begin{description}
    \item[$M=15$:] motif length = 3, no.\ motifs = 2, quantization = 3
    \item[$M=32$:] motif length = 4, no.\ motifs = 3, quantization = 4
    \item[$M=64$:] motif length = 4, no.\ motifs = 4, quantization = 4
\end{description}
Additional experimental settings are given in \autoref{tab:ehrlichraspsettings},
and runtimes in \autoref{tab:ehrnat-times}.

\begin{table}[htb]
    \caption{Ehrlich vs.\ Naturalness times (mins).}
    \small
    \centering
    \begin{tabular}{r|c c c|c c c|c c c}
    & \multicolumn{3}{c|}{$M=15$} & \multicolumn{3}{c|}{$M=32$} &
    \multicolumn{3}{c}{$M=64$} \\
    \textbf{Method} & mean & min & max & mean & min & max & mean & min & max \\
    \hline
    A-GPS-TFM&18.10&17.79&18.38&19.40&18.94&19.78&21.43&20.99&21.75 \\
    VSD-TFM&12.43&12.26&12.61&13.62&13.10&14.34&15.85&15.24&17.26 \\
    CbAS-TFM&9.18&8.91&9.51&9.73&9.40&9.98&11.78&11.43&12.11 \\
    LaMBO-2&14.15&13.50&14.74&16.17&15.74&16.41&17.96&17.68&18.75 \\
    Random (greedy)&0.53&0.51&0.54&0.91&0.91&0.92&2.35&2.35&2.36
    \end{tabular}
    \label{tab:ehrnat-times}
\end{table}

\begin{table}[htb]
    \caption{Sequence experimental settings.}
    \small
    \centering
    \begin{tabular}{r|c c c|c|c}
    & \multicolumn{3}{c|}{\textbf{Ehrlich vs.\ Nat.}}
    & \textbf{Bi-grams} & \textbf{Stability vs.\ SASA} \\
    $\downarrow$ \textbf{Setting} / $M \rightarrow$ & 15 & 32 & 64 & 32
    & > 200 (variable) \\
    \hline
    $N_{t=0}$ & 128 & 128 & 128 & 512 & 512 \\
    $T$ & 40 & 40 & 40 & 64 & 64 \\
    Replicates & 5 & 5 & 5 & 5 & 5 \\
    $B$ & 32 & 32 & 32 & 16 & 16 \\
    $S$ & 256 & 256 & 256 & 256 & 256 \\
    $\refp$ & [-1, 0] & [-1, 0] & [-1, 0] & [0, 0, 0] & auto \\
    Threshold $\thresh_0$ percentile & 0.25 & 0.25 & 0.25 & 0.5 & 0.1
    \end{tabular}
    \label{tab:ehrlichraspsettings}
\end{table}

\subsection{Bi-grams}
\label{app:bigrams}

For the bi-grams experiment we implemented our own black-box for the
\texttt{poli} library based on the experiment in \cite{stanton2022accelerating}.
All architectural details are presented in \autoref{app:arch}, with additional
experimental settings in \autoref{tab:ehrlichraspsettings}. We also report
runtimes in \autoref{tab:bi-times}.

\begin{table}[htb]
    \caption{Bigrams times (mins).}
    \small
    \centering
    \begin{tabular}{r|c c c}
    \textbf{Method} & mean & min & max \\
    \hline
        A-GPS-TFM&30.99&30.04&31.89 \\
        A-GPS-mTFM&40.32&38.78&42.19 \\
        VSD-TFM&21.23&20.10&22.14 \\
        VSD-mTFM&28.13&27.93&28.44 \\
        CbAS-TFM&14.63&14.41&14.84 \\
        CbAS-mTFM&24.24&23.11&26.40 \\
        LaMBO-2&42.03&40.63&43.04 \\
        Random (greedy)&0.47&0.47&0.48
    \end{tabular}
    \label{tab:bi-times}
\end{table}

\subsection{Stability vs.\ SASA}
\label{app:foldx}

For the stability vs.\ SASA experiment we used bi-objective black-box from the
\texttt{poli} library, and the seed sequences and settings from the experiment
in \cite{stanton2022accelerating}. All architectural details are presented in
\autoref{app:arch}, with additional experimental settings in
\autoref{tab:ehrlichraspsettings}. We also report runtimes in
\autoref{tab:bi-times}, and show some samples from the \gls{agps} generative
model, as well as the empirical Pareto front in \autoref{fig:agps_foldx_soln}.

\begin{table}[htb]
    \caption{Stability vs.\ SASA times (mins).}
    \small
    \centering
    \begin{tabular}{r|c c c}
    \textbf{Method} & mean & min & max \\
    \hline
    A-GPS-mTFM&120.75&118.52&122.67 \\
    VSD-mTFM&108.93&107.57&109.92 \\
    CbAS-mTFM&97.48&96.11&99.16 \\
    LaMBO-2&59.33&57.55&61.18 \\
    Random (greedy)&11.38&10.87&12.20
    \end{tabular}
\end{table}

\begin{figure}[htb]
    \centering
    \subcaptionbox{
        Samples from the final \gls{agps} generative model
        conditioned on preference direction vectors.
    }
    {\includegraphics[width=.44\textwidth]{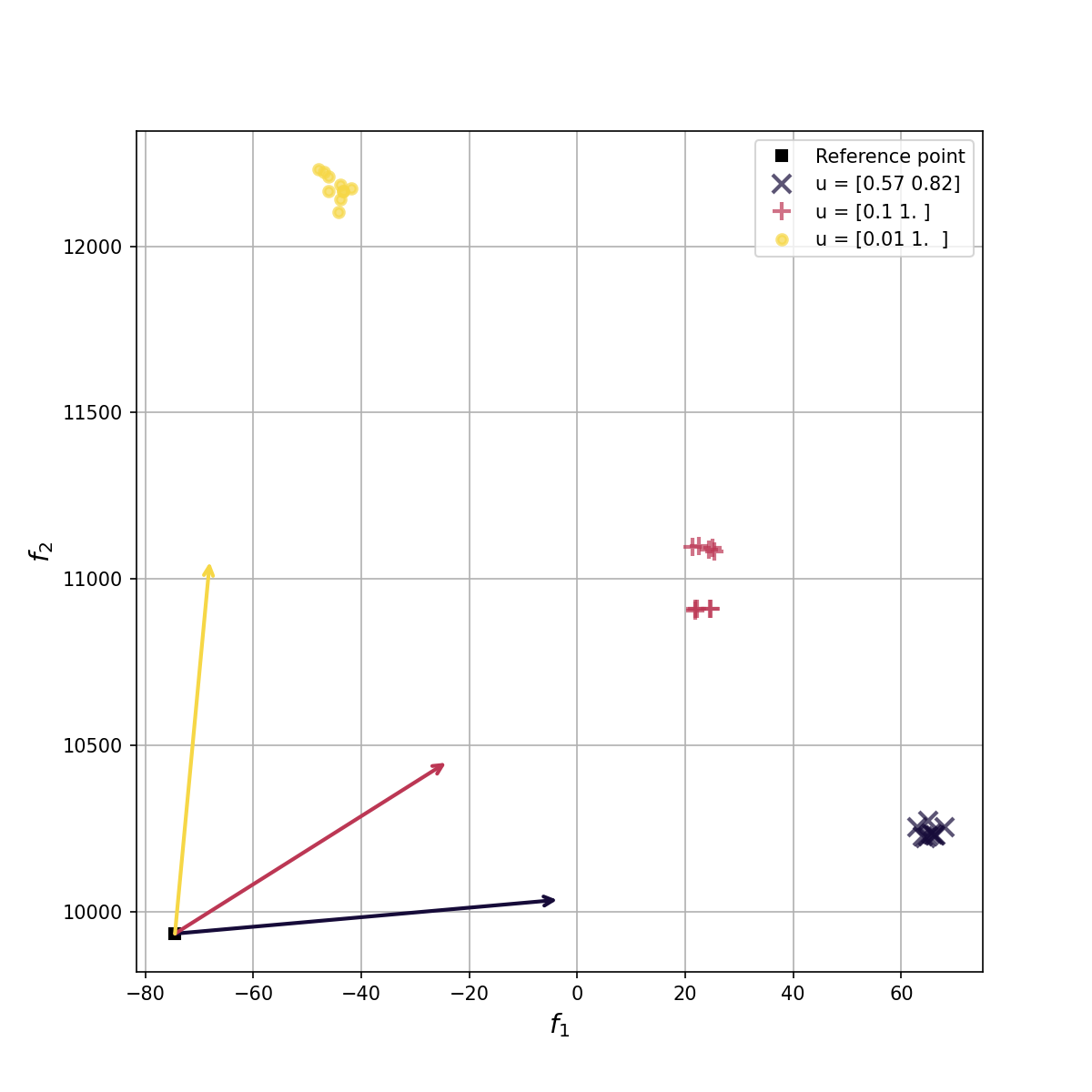}}
    \hfill
    \subcaptionbox{
        FoldX black-box evaluations of sequences generated by \gls{agps},
        colored by evaluation round.
    }
    {\includegraphics[width=.54\textwidth]{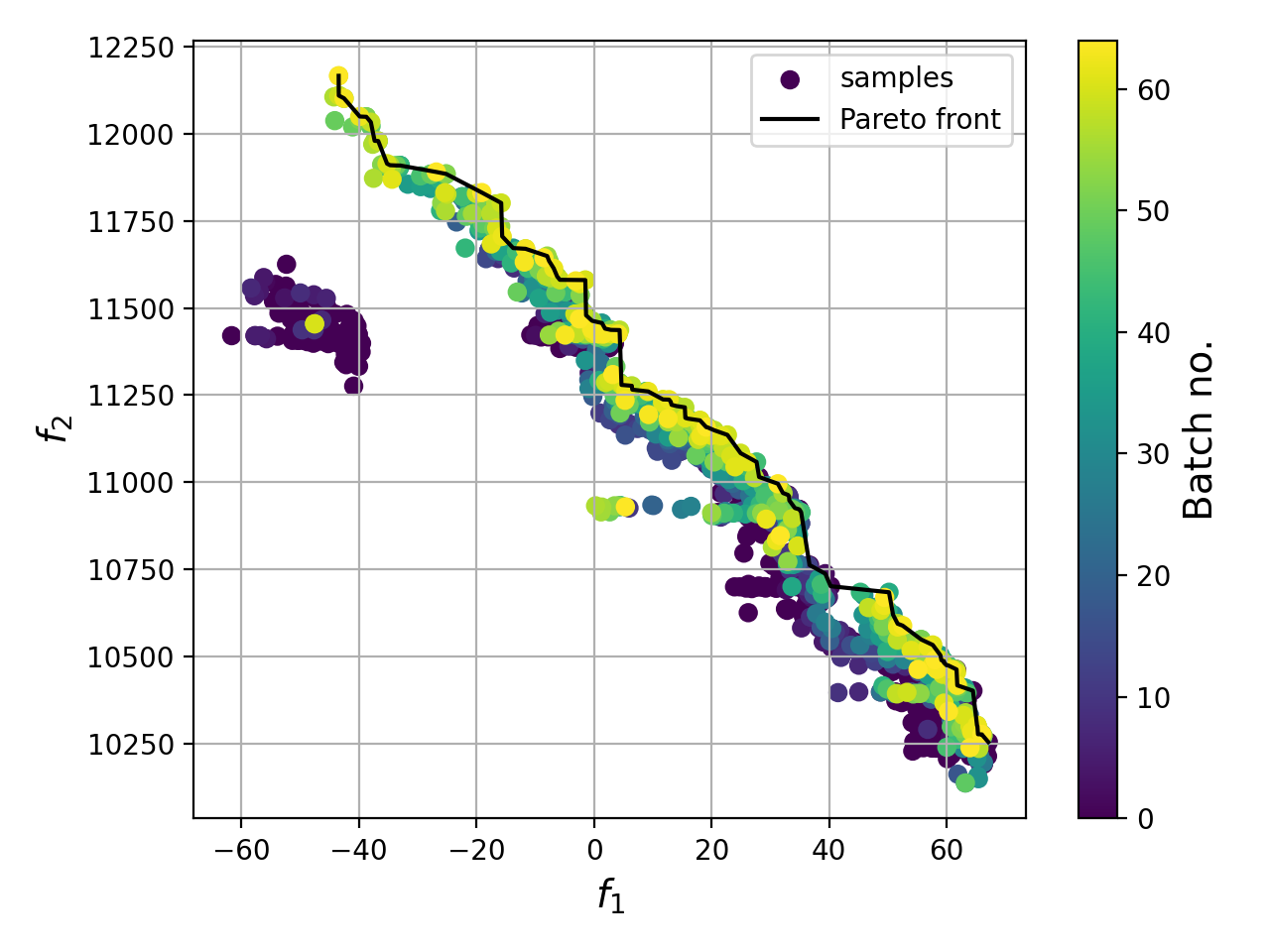}}
    \caption{Additional \gls{agps} sample visualisations from the stability vs.\
    SASA experiment in \autoref{sec:experiments}.}
    \label{fig:agps_foldx_soln}
\end{figure}

% \subsection{RaSP vs.\ Naturalness}
% \label{aps:rasp}

% Similarly to the Ehrlich experiment, we use \texttt{poli}
% \cite{gonzalez-duque2024poli} to implement this experiment, were we use the
% inbuilt \gls{rasp} black-box \cite{blaabjerg2023rasp} with additive mutation
% effects. All designs were based off the DsRED.M1 protein in
% \cite{stanton2022accelerating}, and the initial training data consisted of 128
% randomly mutated versions of this sequence (max 3 mutations), scored by the
% \gls{rasp} predictor. The generative and discriminative models are discussed in
% the preceding sections, and extra experimental configuration is in
% \autoref{tab:ehrlichraspsettings}.

\subsection{Computational resources}
\label{app:compute}

All experiments were run on a Dell PowerEdge XE9640 rack server cluster with
NVIDIA H100 GPUs and 4th generation Intel Xeon CPUs. All of our models could
easily fit on one GPU, and typically took less than 2 hours to complete the
experiments.

\section{Ablation Studies}
\label{app:ablation}

In this section we test some of the architectural decisions we have made when
designing \gls{agps}.

\subsection{On-policy vs.\ off-policy gradients}
\label{app:onvsoff}

We introduce a new gradient estimator in \autoref{eq:off-pol-grads} based on
off-policy importance weighting approximations  to the on-policy gradient
estimator used by \cite{steinberg2025variational}. To test its efficacy, we
re-run the Ehrlich vs.\ naturalness score experiments with this new estimator
and the original on-policy variant. We report performance and runtimes in
\autoref{tab:onvsoffruntime} for $M=32$.
There does not seem to be a consistent difference between the two gradient
estimators in terms of hypervolume performance, but runtime is significantly
lower for the off-policy estimator, being almost an order of magnitude less for
the off-policy variant.

\begin{table}[tbh]
    \caption{
        Ehrlich function vs.\ naturalness score ablation for $M=32$. Run time
        and performance comparison for the on-policy gradient estimators
        (`-reinf.') vs. the importance weighted off-policy estimators for the
        \gls{agps} and \gls{vsd} methods. All times are in minutes.
    }
    \centering
    \small
    \begin{tabular}{l|c|c|c|c}
        & \multicolumn{3}{c|}{\textbf{Time (min)}} & \\
        \textbf{Method} & mean & min & max & \textbf{$T=40$ relative HV improvement} \\
        \hline
        A-GPS-TFM & 19.40 & 18.94 & 19.78 & 6.264 (1.159) \\
        VSD-TFM & 13.62 & 13.10 & 14.34 & 6.711 (0.952) \\
        CbAS-TFM & 9.73 & 9.40 & 9.98 & 4.398 (0.652) \\
        LaMBO-2 & 16.17 & 15.74 & 16.41 & 3.074 (1.950) \\
        Random (greedy) & 0.91 & 0.91 & 0.92 & 1.260 (0.169) \\
        \hline
        A-GPS-TFM-reinf. & 105.09 & 100.20 & 109.52 & 6.656 (0.920)\\
        VSD-TFM-reinf. & 97.95 & 90.74 & 115.81 & 6.257 (1.338)\\
    \end{tabular}
    \label{tab:onvsoffruntime}
\end{table}

\subsection{Off-policy gradient estimator samples}

We test the effect on performance of the number of samples, $S$, used for
estimating the gradients of the off-policy estimator,
\autoref{eq:off-pol-grads}, used for \gls{agps} and \gls{vsd} in
\autoref{tab:offpolsamples}. We also do the same for the \gls{cbas} estimator.
We generally find that more samples lead to more performance for all methods,
however this effect plateaus for \gls{agps} and \gls{vsd} starting at $S=256$,
whereas \gls{cbas} still sees improvement beyond $S=512$. This is also something
noted by the original authors in \cite{brookes2019conditioning}.

\begin{table}[htb]
    \caption{
    Round $T=40$ relative hypervolume improvement results for the varying the
    number of samples used to estimate the gradients with the off-policy
    gradient estimator, \autoref{eq:off-pol-grads}, and the \gls{cbas} gradient
    estimator.
    }
    \centering
    \small
    \begin{tabular}{l l|c|c|c|c}
        $M$ & \textbf{Method} & $S=64$ & $S=128$ & $S=256$ & $S = 512$ \\
        \hline
        32 & A-GPS-TFM & 5.663 (1.357) & 6.212 (0.905) & 6.264 (1.159) & 6.245 (1.398) \\
        & VSD-TFM & 5.665 (0.918) & 6.218 (1.041) & 6.711 (0.952) & 6.964 (0.899) \\
        & CbAS-TFM & 4.048 (1.249) & 4.167 (0.707) & 4.398 (0.652) & 5.531 (1.347) \\
        \hline
        64 & A-GPS-TFM & 5.655 (1.646) & 5.522 (1.188) & 6.021 (1.052) & 6.269 (1.808) \\
        & VSD-TFM & 5.435 (0.844) & 5.198 (1.240) & 5.738 (1.305) & 4.633 (0.817) \\
        & CbAS-TFM & 4.168 (1.352) & 3.803 (0.741) & 3.929 (0.564) & 5.092 (0.972)
    \end{tabular}
    \label{tab:offpolsamples}
\end{table}

\subsection{Prior regularization}
\label{app:dropout}

We found that \gls{agps}, \gls{vsd} and \gls{cbas}, which share generative
backbones, were all very sensitive to the choice of prior. For complex
problems best results were obtained in general when the prior is fit to the
$T=0$ training data, however for flexible backbone models like transformers, it
can be easy to overfit. Adding dropout only to the prior generative model is an
effective means of controlling this overfitting when used with a validation set
to infer the number of learning iterations, as in the process outlined in
\autoref{app:fit_prior}.

To demonstrate the effect of over- and under-fitting on performance, we change
the transformer dropout probability while holding all else constant, e.g.\
fitting iterations, on the Ehrlich vs.\ naturalness experiment for $M=32$ in
\autoref{tab:dropout}. The training iterations were tuned while setting
$p=0.4$, and so we would expect other values of dropout to be suboptimal, and
we can indeed see this is so. All methods severely under-perform with and
overfit prior $p \in \{0.2, 0.3\}$. \Gls{agps} and \gls{vsd} perform slightly
worse with an under-fit prior, whereas \gls{cbas}'s performance actually
improves. This is a trend we have noticed in all experiments; prior overfitting
leads to poor performance, and under-fitting is less consequential. \Gls{cbas}
tends to favour exploitation, which explains why a broader prior, encouraging
exploration, helps the method.

\begin{table}[htb]
    \caption{
    Round $T=40$ relative hypervolume improvement results for the varying the
    prior dropout probability of the transformer backbone. We only present
    results for $M=32$, and the training regime was originally optimized for
    $p=0.4$.
    }
    \centering
    \small
    \begin{tabular}{l|c|c|c|c}
        \textbf{Method} & $p=0.2$ & $p=0.3$ & $p=0.4$ & $p = 0.5$ \\
        \hline
        A-GPS-TFM & 3.786 (0.925) & 5.525 (1.224) & 6.264 (1.159) & 5.743 (0.849) \\
        VSD-TFM & 3.585 (0.894) & 4.455 (0.433) & 6.711 (0.952) &  6.496 (1.246) \\
        CbAS-TFM & 2.543 (0.589) & 4.099 (0.678) & 4.398 (0.652) & 5.838 (1.090)
    \end{tabular}
    \label{tab:dropout}
\end{table}

\subsection{Empirical vs.\ parameterized preference direction distribution}
\label{app:prefdistable}

For all of our experiments we use the constrained mixture of Normal
distributions ($K=5$), \autoref{eq:mixpref}, as our parameterized preference
directions distribution, $\qrob{\pparam}{\prefr}$. We now wish to validate this
choice by comparing it to two simpler alternatives: a single constrained Normal
distribution ($K=1$), and the empirical distribution in \autoref{eq:emppref}.

We make these comparisons on the synthetic test functions, as their Pareto
fronts are well sampled and diverse in their shapes. As can be seen in the
results in \autoref{fig:syn_funcs_abl}, there is not a consistent leader among
the different preference distribution parameterizations across all of these
functions. Perhaps the empirical preference distribution strikes the optimal
blend of simplicity and performance --- however we do occasionally find that it
under-performs compared to the parameterized distributions in cases where the
Pareto front is sparsely sampled, which happens in some sequence experiments.

\begin{figure}[htb]
    \centering
    \subcaptionbox{Negative Branin-Currin \label{sfig:bc_a}}
    {\includegraphics[width=0.329\linewidth]{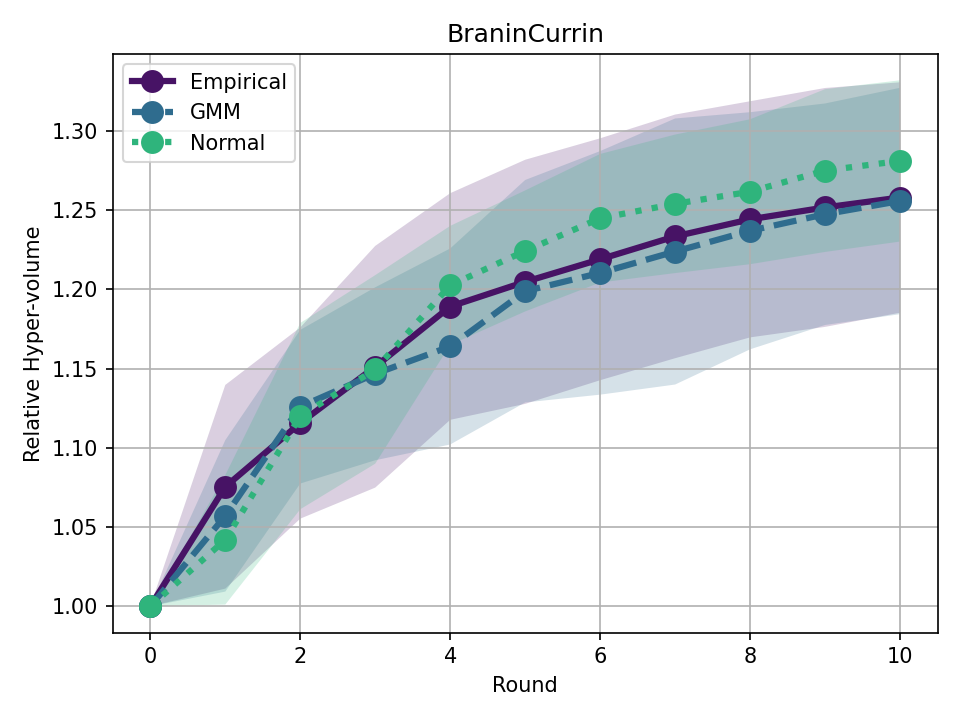}}
    \subcaptionbox{DTLZ7\label{sfig:dtlz7_a}}
    {\includegraphics[width=0.329\linewidth]{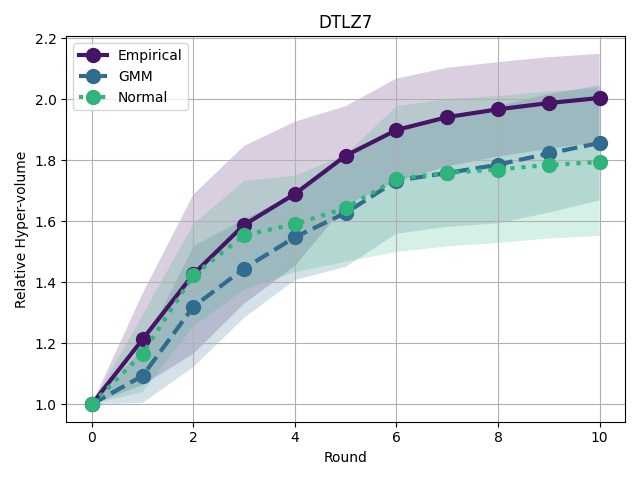}}
    \subcaptionbox{ZDT3\label{sfig:zdt3_a}}
    {\includegraphics[width=0.329\linewidth]{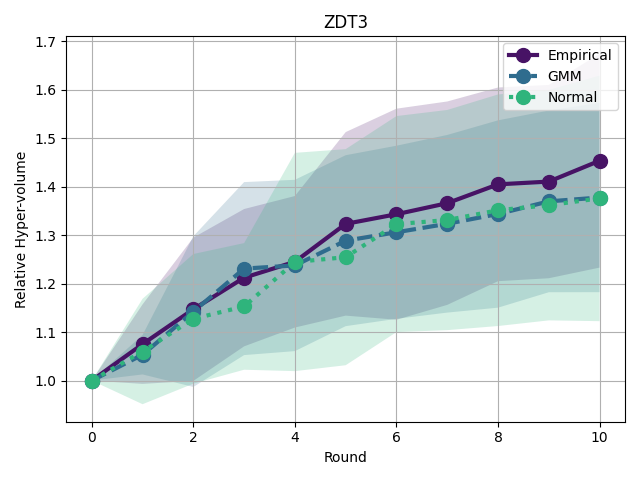}}
    \\
    \subcaptionbox{DTLZ2\label{sfig:dtlz2_a}}
    {\includegraphics[width=0.329\linewidth]{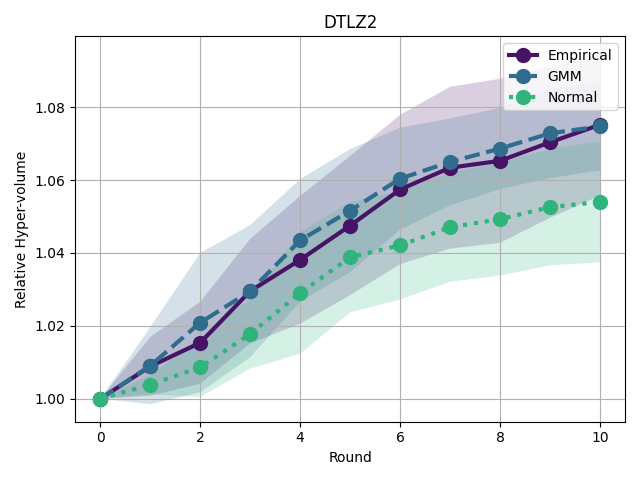}}
    \subcaptionbox{GMM\label{sfig:gmm_a}}
    {\includegraphics[width=0.329\linewidth]{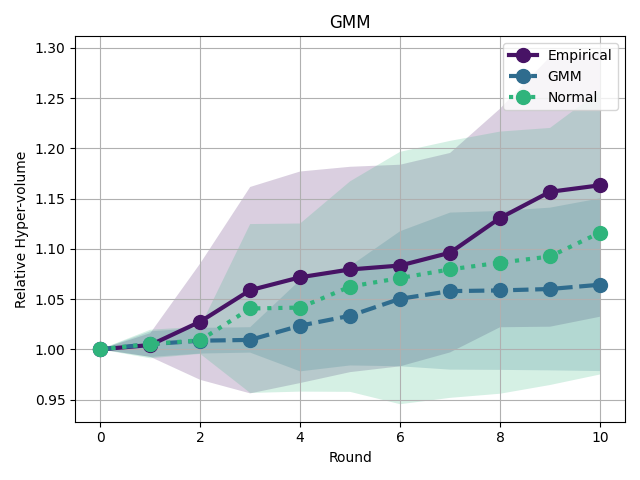}}
    \caption{\gls{agps} preference distribution ablation experimental results on
    the synthetic test functions used for \autoref{sub:exp_synth}.}
    \label{fig:syn_funcs_abl}
\end{figure}

\newpage
\section*{NeurIPS Paper Checklist}

\begin{enumerate}

\item {\bf Claims}
    \item[] Question: Do the main claims made in the abstract and introduction accurately reflect the paper's contributions and scope?
    \item[] Answer: \answerYes{} % Replace by \answerYes{}, \answerNo{}, or \answerNA{}.
    \item[] Justification: The title describes our proposed method, and the abstract and introduction identifies the problem that our method solves and places it in the context of other machine learning problems.
    % \item[] Guidelines:
    % \begin{itemize}
    %     \item The answer NA means that the abstract and introduction do not include the claims made in the paper.
    %     \item The abstract and/or introduction should clearly state the claims made, including the contributions made in the paper and important assumptions and limitations. A No or NA answer to this question will not be perceived well by the reviewers.
    %     \item The claims made should match theoretical and experimental results, and reflect how much the results can be expected to generalize to other settings.
    %     \item It is fine to include aspirational goals as motivation as long as it is clear that these goals are not attained by the paper.
    % \end{itemize}

\item {\bf Limitations}
    \item[] Question: Does the paper discuss the limitations of the work performed by the authors?
    \item[] Answer: \answerYes{} % Replace by \answerYes{}, \answerNo{}, or \answerNA{}.
    \item[] Justification: The limitations are discussed in the ``Related Work'' and ``Discussion'' sections, and illustrated in our empirical results.

\item {\bf Theory assumptions and proofs}
    \item[] Question: For each theoretical result, does the paper provide the full set of assumptions and a complete (and correct) proof?
    \item[] Answer: \answerYes{} % Replace by \answerYes{}, \answerNo{}, or \answerNA{}.
    \item[] Justification: We have provided careful derivations of our proposed model, and also a proof of equivalence of indicators in the main text. We provide all assumptions and derivations of this proof in the appendix.
    % \item[] Guidelines:
    % \begin{itemize}
    %     \item The answer NA means that the paper does not include theoretical results.
    %     \item All the theorems, formulas, and proofs in the paper should be numbered and cross-referenced.
    %     \item All assumptions should be clearly stated or referenced in the statement of any theorems.
    %     \item The proofs can either appear in the main paper or the supplemental material, but if they appear in the supplemental material, the authors are encouraged to provide a short proof sketch to provide intuition.
    %     \item Inversely, any informal proof provided in the core of the paper should be complemented by formal proofs provided in appendix or supplemental material.
    %     \item Theorems and Lemmas that the proof relies upon should be properly referenced.
    % \end{itemize}

    \item {\bf Experimental result reproducibility}
    \item[] Question: Does the paper fully disclose all the information needed to reproduce the main experimental results of the paper to the extent that it affects the main claims and/or conclusions of the paper (regardless of whether the code and data are provided or not)?
    \item[] Answer: \answerYes{} % Replace by \answerYes{}, \answerNo{}, or \answerNA{}.
    \item[] Justification: The proposed method is presented in detail in Algorithm~\ref{alg:optloop}, and how it is derived from first principles is described in Section~\ref{sec:vsd_moo}. Details of experimental settings are provided in the main text and supplement, along with supporting additional experiments.

\item {\bf Open access to data and code}
    \item[] Question: Does the paper provide open access to the data and code, with sufficient instructions to faithfully reproduce the main experimental results, as described in supplemental material?
    \item[] Answer: \answerYes{} % Replace by \answerYes{}, \answerNo{}, or \answerNA{}.
    \item[] Justification: Data reported in this paper are from existing publicly available benchmarks, described in Section~\ref{sec:experiments}. Our open source code is provided at a linked github repository.
    % \item[] Guidelines:
    % \begin{itemize}
    %     \item The answer NA means that paper does not include experiments requiring code.
    %     \item Please see the NeurIPS code and data submission guidelines (\url{https://nips.cc/public/guides/CodeSubmissionPolicy}) for more details.
    %     \item While we encourage the release of code and data, we understand that this might not be possible, so “No” is an acceptable answer. Papers cannot be rejected simply for not including code, unless this is central to the contribution (e.g., for a new open-source benchmark).
    %     \item The instructions should contain the exact command and environment needed to run to reproduce the results. See the NeurIPS code and data submission guidelines (\url{https://nips.cc/public/guides/CodeSubmissionPolicy}) for more details.
    %     \item The authors should provide instructions on data access and preparation, including how to access the raw data, preprocessed data, intermediate data, and generated data, etc.
    %     \item The authors should provide scripts to reproduce all experimental results for the new proposed method and baselines. If only a subset of experiments are reproducible, they should state which ones are omitted from the script and why.
    %     \item At submission time, to preserve anonymity, the authors should release anonymized versions (if applicable).
    %     \item Providing as much information as possible in supplemental material (appended to the paper) is recommended, but including URLs to data and code is permitted.
    % \end{itemize}

\item {\bf Experimental setting/details}
    \item[] Question: Does the paper specify all the training and test details (e.g., data splits, hyperparameters, how they were chosen, type of optimizer, etc.) necessary to understand the results?
    \item[] Answer: \answerYes{} % Replace by \answerYes{}, \answerNo{}, or \answerNA{}.
    \item[] Justification: The details of the algorithm, including the estimation of gradients and the optimizer are described in Section~\ref{sec:vsd_moo}. Specific settings corresponding to the three experimental sub-sections in Section~\ref{sec:experiments} are provided in Appendix~\ref{app:experiments}
    % \item[] Guidelines:
    % \begin{itemize}
    %     \item The answer NA means that the paper does not include experiments.
    %     \item The experimental setting should be presented in the core of the paper to a level of detail that is necessary to appreciate the results and make sense of them.
    %     \item The full details can be provided either with the code, in appendix, or as supplemental material.
    % \end{itemize}

\item {\bf Experiment statistical significance}
    \item[] Question: Does the paper report error bars suitably and correctly defined or other appropriate information about the statistical significance of the experiments?
    \item[] Answer: \answerYes{} % Replace by \answerYes{}, \answerNo{}, or \answerNA{}.
    \item[] Justification: All our experimental results in Section~\ref{sec:experiments} are with appropriate uncertainty intervals.
    % \item[] Guidelines:
    % \begin{itemize}
    %     \item The answer NA means that the paper does not include experiments.
    %     \item The authors should answer "Yes" if the results are accompanied by error bars, confidence intervals, or statistical significance tests, at least for the experiments that support the main claims of the paper.
    %     \item The factors of variability that the error bars are capturing should be clearly stated (for example, train/test split, initialization, random drawing of some parameter, or overall run with given experimental conditions).
    %     \item The method for calculating the error bars should be explained (closed form formula, call to a library function, bootstrap, etc.)
    %     \item The assumptions made should be given (e.g., Normally distributed errors).
    %     \item It should be clear whether the error bar is the standard deviation or the standard error of the mean.
    %     \item It is OK to report 1-sigma error bars, but one should state it. The authors should preferably report a 2-sigma error bar than state that they have a 96\% CI, if the hypothesis of Normality of errors is not verified.
    %     \item For asymmetric distributions, the authors should be careful not to show in tables or figures symmetric error bars that would yield results that are out of range (e.g. negative error rates).
    %     \item If error bars are reported in tables or plots, The authors should explain in the text how they were calculated and reference the corresponding figures or tables in the text.
    % \end{itemize}

\item {\bf Experiments compute resources}
    \item[] Question: For each experiment, does the paper provide sufficient information on the computer resources (type of compute workers, memory, time of execution) needed to reproduce the experiments?
    \item[] Answer: \answerYes{} % Replace by \answerYes{}, \answerNo{}, or \answerNA{}.
    \item[] Justification: Details of our computing resource, as well as required compute time is provided in Appendix~\ref{app:compute}, in the Supplement.
    % \item[] Guidelines:
    % \begin{itemize}
    %     \item The answer NA means that the paper does not include experiments.
    %     \item The paper should indicate the type of compute workers CPU or GPU, internal cluster, or cloud provider, including relevant memory and storage.
    %     \item The paper should provide the amount of compute required for each of the individual experimental runs as well as estimate the total compute.
    %     \item The paper should disclose whether the full research project required more compute than the experiments reported in the paper (e.g., preliminary or failed experiments that didn't make it into the paper).
    % \end{itemize}

\item {\bf Code of ethics}
    \item[] Question: Does the research conducted in the paper conform, in every respect, with the NeurIPS Code of Ethics \url{https://neurips.cc/public/EthicsGuidelines}?
    \item[] Answer: \answerYes{} % Replace by \answerYes{}, \answerNo{}, or \answerNA{}.
    \item[] Justification: This is an algorithmic contribution to multi-objective optimization (which has been studied for decades), and the application is to existing benchmark datasets, so no new ethical questions arise.
    % \item[] Guidelines:
    % \begin{itemize}
    %     \item The answer NA means that the authors have not reviewed the NeurIPS Code of Ethics.
    %     \item If the authors answer No, they should explain the special circumstances that require a deviation from the Code of Ethics.
    %     \item The authors should make sure to preserve anonymity (e.g., if there is a special consideration due to laws or regulations in their jurisdiction).
    % \end{itemize}

\item {\bf Broader impacts}
    \item[] Question: Does the paper discuss both potential positive societal impacts and negative societal impacts of the work performed?
    \item[] Answer: \answerYes{} % Replace by \answerYes{}, \answerNo{}, or \answerNA{}.
    \item[] Justification: This paper proposes an algorithmic advance, hence there are no new direct impacts. However, a further discussion of societal impacts of multi-objective generation is discussed in Appendix~\ref{app:impacts} in the Supplement.

\item {\bf Safeguards}
    \item[] Question: Does the paper describe safeguards that have been put in place for responsible release of data or models that have a high risk for misuse (e.g., pretrained language models, image generators, or scraped datasets)?
    \item[] Answer: \answerNA{} % Replace by \answerYes{}, \answerNo{}, or \answerNA{}.
    \item[] Justification: This paper does not release any large models, nor does it reveal any new datasets.
    % \item[] Guidelines:
    % \begin{itemize}
    %     \item The answer NA means that the paper poses no such risks.
    %     \item Released models that have a high risk for misuse or dual-use should be released with necessary safeguards to allow for controlled use of the model, for example by requiring that users adhere to usage guidelines or restrictions to access the model or implementing safety filters.
    %     \item Datasets that have been scraped from the Internet could pose safety risks. The authors should describe how they avoided releasing unsafe images.
    %     \item We recognize that providing effective safeguards is challenging, and many papers do not require this, but we encourage authors to take this into account and make a best faith effort.
    % \end{itemize}

\item {\bf Licenses for existing assets}
    \item[] Question: Are the creators or original owners of assets (e.g., code, data, models), used in the paper, properly credited and are the license and terms of use explicitly mentioned and properly respected?
    \item[] Answer: \answerYes{} % Replace by \answerYes{}, \answerNo{}, or \answerNA{}.
    \item[] Justification: Data for benchmarks are from previous publications (appropriately cited). Our code builds on Apache 2.0 open source licensed code.
    % \item[] Guidelines:
    % \begin{itemize}
    %     \item The answer NA means that the paper does not use existing assets.
    %     \item The authors should cite the original paper that produced the code package or dataset.
    %     \item The authors should state which version of the asset is used and, if possible, include a URL.
    %     \item The name of the license (e.g., CC-BY 4.0) should be included for each asset.
    %     \item For scraped data from a particular source (e.g., website), the copyright and terms of service of that source should be provided.
    %     \item If assets are released, the license, copyright information, and terms of use in the package should be provided. For popular datasets, \url{paperswithcode.com/datasets} has curated licenses for some datasets. Their licensing guide can help determine the license of a dataset.
    %     \item For existing datasets that are re-packaged, both the original license and the license of the derived asset (if it has changed) should be provided.
    %     \item If this information is not available online, the authors are encouraged to reach out to the asset's creators.
    % \end{itemize}

\item {\bf New assets}
    \item[] Question: Are new assets introduced in the paper well documented and is the documentation provided alongside the assets?
    \item[] Answer: \answerYes{} % Replace by \answerYes{}, \answerNo{}, or \answerNA{}.
    \item[] Justification: The software implementing the algorithmic implementation is well documented, and the documentation is provided in the software package in the linked repository.
    % \item[] Guidelines:
    % \begin{itemize}
    %     \item The answer NA means that the paper does not release new assets.
    %     \item Researchers should communicate the details of the dataset/code/model as part of their submissions via structured templates. This includes details about training, license, limitations, etc.
    %     \item The paper should discuss whether and how consent was obtained from people whose asset is used.
    %     \item At submission time, remember to anonymize your assets (if applicable). You can either create an anonymized URL or include an anonymized zip file.
    % \end{itemize}

\item {\bf Crowdsourcing and research with human subjects}
    \item[] Question: For crowdsourcing experiments and research with human subjects, does the paper include the full text of instructions given to participants and screenshots, if applicable, as well as details about compensation (if any)?
    \item[] Answer: \answerNA{} % Replace by \answerYes{}, \answerNo{}, or \answerNA{}.
    \item[] Justification: No crowdsourcing nor any research with human subjects.
    % \item[] Guidelines:
    % \begin{itemize}
    %     \item The answer NA means that the paper does not involve crowdsourcing nor research with human subjects.
    %     \item Including this information in the supplemental material is fine, but if the main contribution of the paper involves human subjects, then as much detail as possible should be included in the main paper.
    %     \item According to the NeurIPS Code of Ethics, workers involved in data collection, curation, or other labor should be paid at least the minimum wage in the country of the data collector.
    % \end{itemize}

\item {\bf Institutional review board (IRB) approvals or equivalent for research with human subjects}
    \item[] Question: Does the paper describe potential risks incurred by study participants, whether such risks were disclosed to the subjects, and whether Institutional Review Board (IRB) approvals (or an equivalent approval/review based on the requirements of your country or institution) were obtained?
    \item[] Answer: \answerNA{} % Replace by \answerYes{}, \answerNo{}, or \answerNA{}.
    \item[] Justification: No crowdsourcing nor any research with human subjects.
    % \item[] Guidelines:
    % \begin{itemize}
    %     \item The answer NA means that the paper does not involve crowdsourcing nor research with human subjects.
    %     \item Depending on the country in which research is conducted, IRB approval (or equivalent) may be required for any human subjects research. If you obtained IRB approval, you should clearly state this in the paper.
    %     \item We recognize that the procedures for this may vary significantly between institutions and locations, and we expect authors to adhere to the NeurIPS Code of Ethics and the guidelines for their institution.
    %     \item For initial submissions, do not include any information that would break anonymity (if applicable), such as the institution conducting the review.
    % \end{itemize}

\item {\bf Declaration of LLM usage}
    \item[] Question: Does the paper describe the usage of LLMs if it is an important, original, or non-standard component of the core methods in this research? Note that if the LLM is used only for writing, editing, or formatting purposes and does not impact the core methodology, scientific rigorousness, or originality of the research, declaration is not required.
    %this research?
    \item[] Answer: \answerYes{} % Replace by \answerYes{}, \answerNo{}, or \answerNA{}.
    \item[] Justification: The core method development in this research does not involve large language models.
    % \item[] Guidelines:
    % \begin{itemize}
    %     \item The answer NA means that the core method development in this research does not involve LLMs as any important, original, or non-standard components.
    %     \item Please refer to our LLM policy (\url{https://neurips.cc/Conferences/2025/LLM}) for what should or should not be described.
    % \end{itemize}

\end{enumerate}

%%%%%%%%%%%%%%%%%%%%%%%%%%%%%%%%%%%%%%%%%%%%%%%%%%%%%%%%%%%%

\end{document}